\documentclass[default,iicol,sn-basic,pdflatex]{sn-jnl}

\usepackage{cleveref,multicol,ulem}

\jyear{2022}%

\newcommand\pscal[2]{\left\langle#1,#2\right\rangle}

\def\3PS{{\tt 3P-SPIDER}}

\def\nset{{\mathbb{N}}}
\def\rset{\mathbb R}

\def\Zset{\mathsf{Z}}
\def\Zsigma{\mathcal{Z}}

\def\rmd{\mathrm{d}}
\def\argmin{\operatorname{argmin}}

\def\max{\mathrm{max}}

\def\PP{\mathbb{P}} 
\def\PE{\mathbb{E}} 

\newcommand{\F}{\mathcal{F}} 

\def\kernel{P} 
\def\linit{\lambda} 

\def\Id{\mathrm{I}}
\newcommand{\eqdef}{\ensuremath{\stackrel{\mathrm{def}}{=}}}

\newcommand{\kout}{k^\mathrm{out}}
\newcommand{\kin}[1]{k^{\mathrm{in}}_{{#1}}}

\newcommand\init{\mathrm{init}}

\newcommand{\precond}{\mathsf{h}}
\newcommand{\barprecond}{\bar{\mathsf{h}}}

\newcommand{\hatS}{\widehat{S}}
\newcommand{\Smem}{\mathsf{S}}
\newcommand{\Pmatrix}{\mathcal{P}}
\def\B{\mathsf{B}}

\newcommand{\Sset}{\mathcal{S}}

\newcommand{\error}{\eta}
\newcommand{\nbrmc}{m}
\newcommand{\nbrmcv}{M}
\newcommand{\prox}{\mathrm{prox}}

\newcommand{\param}{\theta}
\newcommand{\Param}{\Theta}
\newcommand{\map}{\mathsf{T}}
\newcommand{\lyap}{\operatorname{W}}
\newcommand{\batch}{\mathcal{B}}
\newcommand{\lbatch}{\mathsf{b}}
\newcommand{\pas}{\gamma}
\def\pa{\mathsf{a}}

\def\indic{\chi}

\def\MC{\mathsf{MC}}

\newcommand\sequence[3] {\ifthenelse{\equal{#3}{}}{\ensuremath{\{
#1_{#2}\}}}{\ensuremath{\{ #1^{#2}, \eqsp #2 \in #3 \}}}}
\newcommand\sequencedown[3] {\ifthenelse{\equal{#3}{}}{\ensuremath{\{
#1_{#2}\}}}{\ensuremath{\{ #1_{#2}, \eqsp #2 \in #3 \}}}}

\def\eqsp{\;}
\newcommand{\ie}{i.e.}

\newcommand{\coint}[1]{\left[#1\right)}
\newcommand{\ocint}[1]{\left(#1\right]}
\newcommand{\ooint}[1]{\left(#1\right)}
\newcommand{\ccint}[1]{\left[#1\right]}

\newtheorem{assumption}{{\sf A}\hspace{-3pt}}

\theoremstyle{thmstyleone}%
\newtheorem{theorem}{Theorem}[section]
\newtheorem{proposition}[theorem]{Proposition}
\newtheorem{lemma}[theorem]{Lemma}
\newtheorem{corollary}[theorem]{Corollary}

\begin{document}

\title[3P-SPIDER]{Stochastic Variable Metric
  Proximal Gradient with variance reduction for non-convex composite optimization }

\author*[1]{\fnm{Gersende} \sur{Fort}}\email{gersende.fort@math.univ-toulouse.fr}
\author[2]{\fnm{Eric} \sur{Moulines}}\email{eric.moulines@polytechnique.edu}

\affil*[1]{\orgdiv{Institut de Math\'ematiques de Toulouse}, \orgname{CNRS \& Universit\'e de Toulouse}, \orgaddress{\street{118 route de Narbonne}, \city{Toulouse}, \postcode{31400},  \country{France}}}

\affil[2]{\orgdiv{CMAP}, \orgname{Ecole Polytechnique}, \orgaddress{\street{Route de Saclay}, \city{Palaiseau}, \postcode{91128 Cedex}, \country{France}}}

\abstract{ This paper introduces a novel algorithm, the Perturbed
  Proximal Preconditioned SPIDER algorithm (\3PS), designed to solve
  finite sum non-convex composite optimization. It is a stochastic
  Variable Metric Forward-Backward algorithm, which allows approximate
  preconditioned forward operator and uses a variable metric proximity
  operator as the backward operator; it also proposes a mini-batch
  strategy with variance reduction to address the finite sum setting.
  We show that \3PS extends some Stochastic preconditioned Gradient
  Descent-based algorithms and some Incremental Expectation
  Maximization algorithms to composite optimization and to the case
  the forward operator can not be computed in closed form. We also
  provide an explicit control of convergence in expectation of \3PS,
  and study its complexity in order to satisfy the approximate
  epsilon-stationary condition. Our results are the first to combine
  the non-convex composite optimization setting, a variance
  reduction technique to tackle the finite sum setting by using a
  minibatch strategy and, to allow deterministic or random
  approximations of the preconditioned forward operator. Finally,
  through an application to inference in a logistic regression model
  with random effects, we numerically compare \3PS to other stochastic
  forward-backward algorithms and discuss the role of some design
  parameters of \3PS.}





\maketitle

\section{Introduction}\label{sec:introduction}
Efficient learning from large data sets require new optimization
algorithms designed to be robust to big data and complex models
era. In Statistics and Machine Learning, we are often faced with solving problems of the form
\[
\mathrm{\argmin}_{s\in \rset^q}  \left( \frac{1}{n} \sum_{i=1}^n W_i(s) + g(s) \right) \eqsp,
\]
where $n$ is the number of examples in the training data set, $s$ is
an unknown quantity to be learnt from the examples, $W_i$ is a loss
function associated to the example $\# i$ and $g$ is a regularization
term encoding a priori knowledge and constraints on $s$; $g$ may also
prevent from overfitting.  Quite often, the regularization term $g:
\rset^q \to \ocint{0, + \infty}$ is not differentiable, and the data fidelity
term $n^{-1} \sum_{i=1}^n W_i$ is smooth on the domain of $g$.  

This paper is concerned with stochastic
optimization of a non-convex finite sum composite function; more
precisely, it addresses the differential inclusion problem
\begin{equation}\label{eq:problem-1}
0 \in  \frac{1}{n} \sum_{i=1}^n G_i(s)+   \partial g(s), \qquad s \in \rset^q,
\end{equation}
where $g: \rset^q \to \ocint{-\infty,+\infty}$ is  lower
semi-continuous convex with non-empty domain $\Sset$ and for all $i$,
$G_i: \Sset \to \rset^q$ is globally Lipschitz on $\Sset$.

The first goal of this paper is to provide a novel
algorithm. Motivated by applications in Statistics and Machine
learning, we require this algorithm to satisfy the following three
constraints.  {\sf (c1)} The algorithm uses possibly preconditioned
operators $\precond_i$ instead of the forward operator  $G_i$:
\begin{equation}\label{eq:problem-2}
\forall s \in \Sset, \qquad  \precond_i(s,B) \eqdef - B^{-1} \, G_i(s),
\end{equation}
where $B$ is a $q \times q$ positive definite matrix. Such a condition
encompasses preconditioned gradient methods for example, which also
includes gradient methods with adaptive step sizes. It also
encompasses Expectation Maximization (EM)
algorithms~(\cite{dempster:1977}) designed for large scale
learning. {\sf (c2)} The algorithm may only have access to
approximations of $\precond_i(s,B)$. Such a condition addresses the
situations when $\precond_i(s,B)$ is not explicit, for example when it
is defined by an intractable integral. This occurs at each E-step of
EM, when the conditional expectations under the a posteriori
distributions can not be computed exactly. {\sf (c3)} The algorithm
addresses the finite sum challenge while keeping the caused
variability induced by the algorithmic solution as small as possible. For example, when the solution
relies on a random selection of a mini-batch of examples, the
algorithm has to propose a variance reduction scheme.

A first class of problems of the form \eqref{eq:problem-1} are
minimizations of regularized loss functions through gradient-based
algorithms. In that case, $G_i \eqdef \nabla W_i$.
\eqref{eq:problem-2} allows preconditioned gradients; such a variable
metric is known to accelerate the convergence.  Variable Metric Forward-Backward (VMFB) algorithms were
introduced to solve
\eqref{eq:problem-1}-\eqref{eq:problem-2} in the case $G$ is a gradient. Nevertheless, as discussed
in \Cref{sec:VMFB}, to our best knowledge none of the variants of VMFB
address the three constraints {\sf c1}, {\sf c2} and {\sf c3}
simultaneously.

A second application of \eqref{eq:problem-1}-\eqref{eq:problem-2} is
the EM algorithm, an algorithm originally designed to compute the
Maximum-Likelihood estimator of an unknown parameter $\param$ in
latent variable models. When the complete data model is from the
curved exponential family, EM is equivalent to an algorithm in the
so-called {\it statistic space} (see
e.g.~\cite{lavielle:delyon:moulines:1999}). This remark is the
cornerstone of stochastic EM algorithms including incremental EM ones
designed for incremental processing of large data sets. EM only
supplies preconditioned forward operators $ -\B(s)^{-1} G_i(s)$.
Therefore, stochastic EM algorithms are naturally in the setting
\eqref{eq:problem-1}-\eqref{eq:problem-2} (see
\cite{fort:etal:2020}). Here again, as discussed in \Cref{sec:EM}, none
of the EM variants in the literature address the constraints {\sf
  (c2)} and {\sf (c3)} simultaneously.

\bigskip 

Our first contribution is the design of a novel iterative algorithm,
named {\tt Perturbed Proximal Preconditioned SPIDER} (\3PS), which
combines \textit{(i)} a preconditioned forward operator associated to
the smooth part $n^{-1} \sum_{i=1}^n G_i(s)$, \textit{(ii)} a variable
metric proximity operator with respect to the non-smooth part $g$,
\textit{(iii)} a stochastic approach to address the finite sum setting
induced by $n^{-1} \sum_{i=1}^n G_i(s)$ combined with a variance
reduction technique based on the {\tt SPIDER} algorithm
(\cite{nguyen:etal:2017,fang:etal:2018}); it also allows \textit{(iv)}
numerical approximations of the preconditioned forward operator when
it has no closed form. The algorithm is introduced in
\Cref{sec:algo:3P}, together with discussions on implementation
questions. We also design a stochastic VMFB algorithm which answers the
constraints {\sf c1} and {\sf c2} but does not contain a variance
reduction step as required by {\sf c3}.

\bigskip 

The second contribution is to provide a non-asymptotic convergence
analysis in expectation of \3PS in the case the variable metric $B$ at
iteration $\# t$ depends on the current value of the iterate, and the
$G_i$'s are gradient operators; see \Cref{sec:theory}. The proof
relies on a Lyapunov inequality with an original construction, which
is a consequence of the non-convex optimization setting, and the fact
the algorithm uses preconditioned forward operators and variable
metric proximity operators (see \Cref{sec:proof:Lyapunov}).
\\ \Cref{theo:main:general} provides a control of convergence in
expectation for \3PS, which explicitly identifies the impact of
non-exact preconditioned forward operators, and the impact of
initialization strategies. First, we prove that the learning rate of
\3PS can be chosen constant over iterations when the preconditioned
forward operator is exact or replaced with an unbiased random oracle;
and is decreasing along iterations when it is replaced with a biased
oracle (deterministic or random). Second, we provide the first
convergence result for a stochastic VMFB algorithm addressing {\sf
  c1}, {\sf c2} and {\sf c3} for non-convex finite sum composite
optimization. For example, it is the first result for incremental EM
with a non-smooth penalty term ($g \neq 0$) and possibly biased Monte
Carlo approximations of the E-step.  \\ When the forward operator
$\precond_i(s,\B(s))$ is exact, we study the complexity of \3PS (see
\Cref{coro:complexity}): in order to satisfy the approximate
$\epsilon$-stationary condition, the number of calls
$\mathcal{K}_{\barprecond}$ to one of the operator
$\precond_i(s,\B(s))$ is $O(\sqrt{n}/ \epsilon)$, the number of calls
$\mathcal{K}_{\prox}$ to the backward operator is $O(1/\epsilon)$ and,
the learning rate can be chosen independent of the accuracy
$\epsilon$. Applied to the Gradient method and applied to the EM
method when there are no constraints $(g=0)$, these explicit controls
of convergence retrieve previous results in the literature (see
e.g.~\cite{wang:etal:2019,fort:etal:2020}) which are known to be at
the state-of-the-art. \\ Finally we show that this complexity analysis
remains valid when the forward operators are approximated. In the
difficult case when the approximations are biased random oracles based
on Monte Carlo sums, we show that $\mathcal{K}_{\barprecond}$ and
$\mathcal{K}_{\prox}$ are not impacted by the approximation and are
the same as with exact operators $\precond_i(\cdot, \B(\cdot))$, by choosing an adequate number of
terms in the Monte Carlo sums. The price to be paid is a Monte Carlo
complexity of order $O(\sqrt{n}/\epsilon^2)$.

\bigskip 

In \Cref{sec:application}, \3PS is applied to inference in a logistic
regression model with random effects.  We show how the problem is of
the form \eqref{eq:problem-1}-\eqref{eq:problem-2}. In this example,
the preconditioned forward operators are approximated by a Monte Carlo
sum computed from a Markov chain Monte Carlo sampler. Through
numerical analyses in the case \3PS is a stochastic Expectation
  Maximization algorithm in the statistic space, we discuss the
choice of design parameters. We also show how the {\tt
  SPIDER} variance reduction effect can be increased by exploiting the Monte Carlo
approximations of the forward operator.

\bigskip

The proofs are given in \Cref{sec:proof1} and \Cref{sec:proof2};
technical details are also provided in Appendix.

\bigskip 

{\bf Notations.}  We denote by $\pscal{\cdot}{\cdot}$ the dot product
on $\rset^q$, and by $\| \cdot \|$ the associated norm.  For a $q
\times q$ positive definite matrix $B$, we set $\pscal{\cdot}{\cdot}_B
\eqdef \pscal{B \cdot}{\cdot}$ and $\| \cdot \|_B$ the associated
norm.  $\Id_q$ denotes the $q \times q$ identity matrix. For a matrix
$B$, $B^\top$ is its transpose. $\Pmatrix_+^q$ denotes the set of the
$q \times q$ positive definite matrices.

$\nset$ (resp. $\nset^\star$) is the set of non negative
(resp. positive) integers. For $n \in \nset^\star$, we set $[n] \eqdef
\{0, \cdots, n\}$ and $[n]^\star \eqdef \{1, \cdots, n\}$.  $\rset_+$
is the set of the positive real numbers.  For $q \in \rset$, $\lceil q
\rceil$ is the upper integer part.

 $\Id$  is the identity function. For a proper function $g: \rset^q \to
\ocint{-\infty, + \infty}$, $\partial g(s)$ denotes the
sub-differential at $s$. For a continuously differentiable function
$W$ at $s \in \rset^q$, $\nabla W(s)$ is the gradient of $W$ at $s$.
 
All the random variables (r.v.) are defined on $(\Omega, \mathcal{A},
\PP)$; for a r.v. $U$, $\sigma(U)$ is the sigma-field generated by
$U$.

\section{Motivating examples}
In this section, we show that the Variable Metric Proximal-Gradient
algorithm and the Expectation Maximization algorithm are examples of
the general framework described by \eqref{eq:problem-1} and
\eqref{eq:problem-2}. In the first case, the preconditioning matrices
$B$ are chosen by the user, while in the second case, they are
supplied by the algorithm. 

\subsection{Variable Metric Proximal-Gradient algorithms}
\label{sec:VMFB}
Consider the  non-convex composite problem with finite sum structure
\begin{equation}\label{eq:problem:gdt}
\text{find $s \in \rset^q$:} \quad 0 \in \frac{1}{n} \sum_{i=1}^n
\nabla W_i(s) + \partial g(s) \eqsp,
\end{equation}
where $g$ is a proper lower semicontinuous convex function with domain
$\Sset$; and for all $i \in [n]^\star$, $W_i$ is continuously
differentiable on $\Sset$. It is of the form \eqref{eq:problem-1} with
$G_i \eqdef \nabla W_i$ being a gradient.  \eqref{eq:problem:gdt} is
an example of the more general problem: finding a zero on $\rset^q$ of
the sum of two (set-valued) operators $0 \in \mathsf{A} s + \mathsf{C}
s$. Here, ${\mathsf A} \eqdef n^{-1} \sum_{i=1}^n\nabla W_i$ and
$\mathsf{C} \eqdef \partial g$ is a maximally monotone operator (see
e.g.~\cite[Theorem
  20.25]{bauschke:combettes:2011}). \eqref{eq:problem:gdt} can be
solved by Forward-Backward splitting algorithms (see
e.g.~\cite{combettes:wajs:2005,beck:2017}): the forward step uses the
gradient of some if not all the functions $W_i$'s at each iteration;
the backward step uses a proximity operator associated to $g$. This
yields Proximal-Gradient based algorithms.

In the case $g=0$, which includes unconstrained optimization problems,
stochastic gradient methods with variance reduction were proposed in
the situation
\begin{equation} \label{eq:randomoracles}
n^{-1} \sum_{i=1}^n G_i(s) =
\PE\left[\mathcal{G}(Z,s)\right]
\end{equation} and random oracles
$\mathcal{G}(Z,s)$ are available; in the non-convex setting, let us
cite
e.g. \cite{ghadimi:lan:2013,reddi:etal:2016,allenzhu:hazan:2016,nguyen:etal:2017,allenzhu:2018,fang:etal:2018,dongruo:etal:2020}. These
algorithms address the
problem~\eqref{eq:problem-1}-\eqref{eq:problem-2} by choosing $B$
equal to the identity matrix $\Id_q$ and they use unbiased random
oracles $\mathcal{G}(Z,s)$ for the approximation of the forward
operator.

For non-convex composite optimization ($g \neq 0$), let us cite
\cite{ghadimi:etal:2016} and \cite{karimi:etal:2020} for stochastic
Proximal-Gradient algorithms using unbiased oracles $\mathcal{G}(Z,s)$
(see \eqref{eq:randomoracles}).  \cite{li:li:2018},
\cite{wang:etal:2019}, \cite{zhang:xiao:2019}, \cite{nhan:etal:2020}
and \cite{metel:takeda:2021} propose stochastic Proximal-Gradient
methods with unbiased random oracles and including variance reduction
schemes. In \cite{metel:takeda:2021}, $g$ may be non-convex but admits
an efficiently computable proximity operator.
\cite{atchade:fort:moulines:2015} allow for deterministic or random
approximations of the forward operator $n^{-1} \sum_{i=1}^n G_i(s)$;
when the perturbation is stochastic, the convergence analysis covers
both biased and unbiased oracles, includes Nesterov acceleration
schemes, but is restricted to convex optimization.  Here again, all
these algorithms address the
problem~\eqref{eq:problem-1}-\eqref{eq:problem-2} by choosing
$B=\Id_q$.

Forward-Backward suffers from slow convergence, and Variable Metric
Forward-Backward ({\tt VMFB}) methods were proposed by
\cite{chen:rockafeller:1997} in order to accelerate the convergence
(see also refs. 11 to 16 in
\cite{chouzenoux:pesquet:repetti:2014}). {\tt VMFB} changes the metric
at each iteration by using symmetric positive definite scaling
matrices multiplying the forward operator. It is an alternative to
inertial methods such as Heavy Ball or Nesterov acceleration which use
informations from the previous iterates.  When solving the inclusion
\eqref{eq:problem:gdt}, {\tt VMFB} uses preconditioned gradients with
an iteration-dependent preconditioning matrix $B_t^{-1}$ for the
forward step at iteration $\# t$, and scales the proximal step
consequently. Examples showing that {\tt VMFB} is more efficient than
Forward-Backward and Forward-Backward with inertial schemes, are
provided in \cite{chouzenoux:pesquet:repetti:2014} and
\cite{repetti:chouzenoux:pesquet:2014}.  Different strategies exist
for the definition of the variable metric $B_t$; for example, it may
be a diagonal matrix depending on the past history of the algorithm
(see e.g. \cite{park:etal:2019} and references therein for variable
scalar metrics; see also \cite{chen:liu:etal:2018} in the case $g=0$),
or inherited from Newton-type methods (see
e.g. \cite{becker:fadili:2012,lee:etal:2014,becker:etal:2019} for the composite convex case; see also \cite{kolte:etal:2015,moritz:etal:2016,gower:etal:2016} for the smooth convex case with finite sum structure; finally, see \cite{zhang:etal:2022} for the  smooth non-convex case with finite sum structure), or defined
through a Majorize-Minimize strategy to make the backward operator
explicit (see e.g. \cite{chouzenoux:pesquet:repetti:2014} and
\cite{repetti:wiaux:2021}).  Convergence results for {\tt VMFB} exist
in the convex case (see e.g. \cite{combettes:vu:2014} and
\cite{bonettini:etal:2021}; and \cite{park:etal:2019} for the strongly
convex case) and in the non-convex case (see e.g.
\cite{chouzenoux:pesquet:repetti:2014} and
\cite{repetti:wiaux:2021}). In \cite{yun:lozano:yang:2021}, a
stochastic {\tt VMFB} is studied in the non-convex case; the exact
gradient is approximated by a linear combination of random oracles,
with exponential forgetting, and the oracles are assumed unbiased and
bounded.

\3PS addresses non-convex composite optimization with a finite sum structure by using
Proximal-Gradient algorithms accelerated via the Variable Metric
cunning. It is a stochastic {\tt VMFB}, which contains a variance
reduction technique in order to overcome the finite sum setting; it
also allows oracles for the preconditioned forward operators, oracles
which can be biased or unbiased when random (see
e.g. \cite{atchade:fort:moulines:2015} and
\cite{fort:risser:etal:2018} for examples motivating biased random
approximations of the gradient).  The combination of these two sources
of perturbations is an original setting which, to our best knowledge,
is not addressed in the literature.  \\ \3PS uses preconditioning
matrices, which may depend on the current value of the iterate and
therefore may be random.  The non-asymptotic convergence analysis
derived in \Cref{sec:theory} will rely on weaker minorization
assumptions on the spectrum of the scaling matrices (see \Cref{hyp:W})
than in \cite{yun:lozano:yang:2021}; it will not require ordering
assumptions on the sequence of scaling matrices as in
\cite{combettes:vu:2014} and \cite{bonettini:etal:2021}, and will not
assume a Kurdyka-{\L}ojasiewicz condition on the objective function as
in \cite{chouzenoux:pesquet:repetti:2014}.  As a consequence, the
construction of the Lyapunov inequality for the convergence analysis
of \3PS differs from these previous works.  \\ \3PS requires the
backward operator to be explicit, which may be a strong assumption
especially for variable metric proximity operators (see
\Cref{sec:comput:proximity}); extensions of the convergence analysis
to the case of inexact proximity operators is out of the scope of this
paper.

\subsection{Expectation Maximization for curved exponential families}
\label{sec:EM}
Consider the parametric statistical model: the observations are
independent with density
\[
y \mapsto \int_\Zset p(y,z;\theta) \mu_{lv}(\rmd z) \eqsp,
\]
with respect to (w.r.t.) a $\sigma$-finite positive measure $\mu_o$ on
$\rset^{d_y}$. In this model, $z$ acts as a {\it latent} variable
taking values in the measurable set $(\Zset, \Zsigma)$ endowed with a
$\sigma$-finite positive measure $\mu_{lv}$ (see
e.g. \cite{everitt:1984} for examples of latent variable models).  The
goal is to learn the parameter $\theta \in \Theta \subseteq \rset^d$
from $n$ observations $Y_1, \cdots, Y_n$, by minimizing the 
negative normalized log-likelihood
\begin{equation}\label{eq:EM:LogLike}
 F(\param) \eqdef - \frac{1}{n} \sum_{i=1}^n \log \int_\Zset
 p(Y_i,z;\theta) \mu_{lv}(\rmd z)
\end{equation}
on $\Theta$. Unfortunately, this is a non-convex problem and most
often, an optimization algorithm for the minimization of
\eqref{eq:EM:LogLike} can not do better than converging to a critical
point of the objective function (see \cite{Wu:1983}).

A popular model is the case when the {\it complete data likelihood}
$(y,z) \mapsto p(y,z; \theta)$ is of the form
\[
p(y,z;\theta) \eqdef H(y,z) \, \exp\left( \pscal{S(y,z)}{\phi(\param)} - \psi(\param) \right)
\]
where $H: \rset^{d_y} \times \Zset \to \rset_+$, $S: \rset^{d_y}
\times \Zset \to \rset^q$, $\phi: \Param \to \rset^q$, $\psi: \Param
\to \rset$; it corresponds to the so-called {\it curved exponential
  family} assumption. It is satisfied by the mixture models as soon as
the components of the mixture are from the curved exponential
family. See e.g. \cite{brown:1986} for an introduction to curved
exponential family of distributions; and
\cite{mclachlan:krishnan:2008} for examples of such latent variable
models.

EM for curved exponential families defines iteratively a
$\Theta$-valued sequence $\{\param_t, t \geq 0\}$ through the
mechanism: given $\param_t$,
\begin{itemize}
\item {\bf (E-step)} Compute $\bar{s}(\param_t)$, the expectation of
  the {\em sufficient statistics} w.r.t. the {\it a posteriori} distributions
\begin{align*}
  &  \bar{s}(\param_t) \eqdef \frac{1}{n} \! \sum_{i=1}^n \bar{s}_i(\param_t) \eqsp, \\
  & \bar{s}_i(\param_t) \eqdef \int_\Zset \! S(Y_i,z) 
\frac{p(Y_i,z; \param_t) \, \mu_{lv}(\rmd z)}{\int_\Zset p(Y_i, u;
  \param_t ) \mu_{lv}(\rmd u)} \eqsp.
\end{align*}
\item {\bf (M-step)} Update the parameter
  \[
\param_{t+1} \eqdef \argmin_{\param \in \Param} \left( \psi(\param) -
\pscal{\bar{s}(\param_t)}{\phi(\param)} \right) \eqsp.
  \]
\end{itemize}
The algorithm alternates between a step in the parameter space
$\Param$ (when computing $\param_{t+1} \in \rset^d$), and a step in
the {\it statistic space} when computing $\bar{s}(\param_t) \in
\rset^q$.  \Cref{prop:EM:statisticspace} states that the limiting
points of EM run in the parameter space $\Theta$ are the fixed points
of an operator onto $\Param$; finding such a fixed point is equivalent
to find a fixed point of an operator onto the statistic space
$\bar{s}(\Param) \subseteq \rset^q$.
\begin{proposition} \label{prop:EM:statisticspace}
  Assume that for any $s \in \Sset \supseteq \bar{s}(\Param)$,
  \[
\map(s) \eqdef \argmin_{\param \in \Param} \left( \psi(\param) -
\pscal{s}{\phi(\param)} \right)
\]
exists and is unique. Set $ \mathcal{L}_\param \eqdef \{\param \in
\Param: \map(\bar{s}(\param)) = \param \}$ and $ \mathcal{L}_s \eqdef
\{s \in \bar{s}(\Param): \bar{s}(\map(s)) = s \}$.

$\mathcal{L}_\param$ is the set of the limiting points of EM.  If
$\param_\star \in \mathcal{L}_\param$, then $s_\star \eqdef
\bar{s}(\param_\star)$ is in $\mathcal{L}_s$. Conversely, if $s_\star
\in \mathcal{L}_s$ then $\param_\star \eqdef \map(s_\star)$ is in
$\mathcal{L}_\param$.
\end{proposition}
See e.g. \cite{lavielle:delyon:moulines:1999} for the proof. An
algorithmic corollary of \Cref{prop:EM:statisticspace}, is that EM is
equivalent to any algorithm run in the statistic space and designed to
find the roots of the function
\[
s \mapsto \frac{1}{n} \sum_{i=1}^n \barprecond_i(s) \eqsp,
\ \ \text{where} \ \ \barprecond_i(s) \eqdef \bar{s}_i(\map(s)) -s \eqsp,
\]
on the subset $\bar{s}(\Param)$ of $\rset^q$.  Under regularity
conditions on the statistical model, it is proved in
\cite{lavielle:delyon:moulines:1999} (see also a statement in
\cite[Proposition 1]{fort:etal:2020}) that there exists a $q
\times q$ positive definite matrix $\B(s)$ such that
 \begin{equation}\label{eq:EM:preconfGdt}
\nabla \left(F \circ \map\right)(s) = - \B(s)  \left( \frac{1}{n} \sum_{i=1}^n
\barprecond_i(s) \right) \eqsp,
\end{equation}
where $F$ is the negative normalized log-likelihood (see
\eqref{eq:EM:LogLike}). Therefore, the roots of $n^{-1} \sum_{i=1}^n
\barprecond_i(s)$ on $\bar{s}(\Param)$ are the roots of $n^{-1}
\sum_{i=1}^n G_i(s)$ on $\bar{s}(\Param)$, where $G_i(s) \eqdef -
 \B(s) \, \barprecond_i(s)$. It also means that the roots of $n^{-1}
\sum_{i=1}^n \barprecond_i(s)$ are the roots of the gradient of $F \circ
\map$, the objective function transferred on the statistic space
through the map $\map: \Param \to \rset^q$.

As a conclusion, EM in the statistic space is an example of problem
\eqref{eq:problem-1}-\eqref{eq:problem-2}, where the function $g$
collects the constraint on $s$ such as $s \in \Sset \supseteq
\bar{s}(\Param)$: \textit{(i)} it is designed to find a root of
$n^{-1} \sum_{i=1}^n G_i(s) = \nabla(F \circ \map)(s)$ under the
constraint that $s \in \Sset$; \textit{(ii)} it uses the quantities
$\barprecond_i(s)$ which are preconditioned forward operator since
there exists $B(s)$ such that $\barprecond_i(s) = - \B(s)^{-1} \,
G_i(s)$ (see \eqref{eq:EM:preconfGdt}); \textit{(iii)} this
preconditioned forward operator is intractable for at least two
reasons: first, due to the inner integrations on the set $\Zset$ when
computing $\bar{s}_i$, since $p(y,z; \param)$ and $\Zset$ are often
too complex to make the integrals explicit; second, due to the outer
integration on the $n$ examples when computing $\bar{s}$, which has a
prohibitive computational cost in large scale learning.  However, a
Monte Carlo approximation of $\barprecond_i(s)$ is always possible,
whatever $i$ and $s$.  This remark is the cornerstone to understand
the stochastic versions of EM (see
e.g. \cite{celeux:diebolt:1985,Wei:tanner:1990,lavielle:delyon:moulines:1999,fort:moulines:2003}
which address the inner sum intractability; and
\cite{neal:hinton:1998,Ng:mclachlan:2003,cappe:moulines:2009,chen:etal:2018,karimi:etal:2019,fort:etal:2020,fort:etal:2021}
for the outer sum intractability).  They consist in running a
Stochastic Approximation (SA) algorithm with mean field $ n^{-1}
\sum_{i=1}^n \barprecond_i(s)$ (for an introduction to SA, see
e.g. \cite{benveniste:etal:1990} or \cite{borkar:2008}); this yields
{\it SA within EM} procedures. They differ through the construction of
the random field used for the approximation of the mean field (see
\cite[Section 2.2.]{fort:etal:2021} for a description of some SA
within EM algorithms).

 \3PS is among the SA within EM algorithms. Compared to previous
 stochastic EM methods, it encompasses the two random approximations
 (of the sum in $i$ and of the integrals on $\Zset$) and a variance
 reduction step, and it also allows a more general penalty term $g$
 than the $\{0, + \infty \}$-valued indicator function of a set.
 
 \section{The \3PS\ algorithm}
 \label{sec:algo:3P}
 We introduce a novel algorithm named {\it Perturbed Proximal
   Preconditioned SPIDER} (\3PS), solving \eqref{eq:problem-1} and
 satisfying {\sf (c1)}, {\sf (c2)} and {\sf (c3)}. It requires $g$ to
 satisfy the following assumption
\begin{assumption}
 \label{hyp:g} $g: \rset^q \to \ocint{-\infty,+\infty}$ is proper, lower
    semicontinuous and convex. Denote by $\Sset$ its domain $\Sset
    \eqdef \{s \in \rset^q: g(s) < +\infty \}$.
\end{assumption}
Under this assumption, we define a {\it variable metric} proximity
operator.  For any $\pas >0$ and $B \in \Pmatrix_+^q$, the proximity
operator of the proper lower semicontinuous convex function $\pas g:
\rset^q \to \ocint{-\infty, + \infty}$ relative to the metric induced
by $B$ is defined by (see e.g. \cite[Section
  XV.4]{hiriarturruty:lemarechal:1996})
\begin{equation}
  \label{eq:varmatric:prox}
\prox_{\pas g}^B(s) \eqdef \!\! \argmin_{\rset^q} \left(\pas g(\cdot) + \frac{1}{2} \| \cdot-s\|^2_B \right).
\end{equation}
When $B = \Id_q$, we simply write $\prox_{\pas g}(s)$, which is the
proximity operator originally defined by \cite{moreau:1965}.
\Cref{lem:ExistenceUniqueness:prox} shows that under \Cref{hyp:g},
$\prox_{\pas g}^B(s)$ exists and is unique for all $s \in \rset^q$,
$\pas >0$ and $B \in \Pmatrix_+^q$.  It also provides
characterizations of this point. Its proof is in
\Cref{sec:proof:lem:Prox}.
\begin{lemma} 
\label{lem:ExistenceUniqueness:prox} Assume \Cref{hyp:g}.
\begin{enumerate}
\item For any $\pas >0$, $B \in \Pmatrix_+^q$ and $s \in \rset^q$, the
  optimization problem \eqref{eq:varmatric:prox} has a unique
  solution, characterized as the unique point ${\sf p} \in \Sset$
  satisfying
  \[
  - \, \pas^{-1} \, B \, ({\sf p} -s) \in \partial g({\sf p}) \eqsp.
  \]
\item \label{lem:ExistenceUniqueness:prox:item2} For any $\pas >0$, $B
  \in \Pmatrix_+^q$, $s \in \Sset$ and $h \in \rset^q$,
\begin{equation} \label{eq:fixedpoint}
s = \prox_{\pas g}^{B}(s + \pas h) \quad \text{iff} \quad  B h \in  \partial g(s). 
\end{equation}
  \end{enumerate}
\end{lemma}
For $s \in \rset^q$ and $B \in \Pmatrix_+^q$, set
\begin{align} \label{eq:def:meanfieldh}
  & \precond_i(s,B) \eqdef - B^{-1} \,  G_i(s) \eqsp, \nonumber  \\
& \precond(s,B) \eqdef n^{-1} \sum_{i=1}^n \precond_i(s,B) \eqsp.
\end{align}
By
\Cref{lem:ExistenceUniqueness:prox}-\cref{lem:ExistenceUniqueness:prox:item2},
it holds for any $B \in \Pmatrix_+^q$ and $\pas >0$: $s = \prox_{\pas
  \, g}^{B}\left(s+ \pas \, \precond(s, B) \right)$ iff $-B
\precond(s,B) \in \partial g(s)$.  By \eqref{eq:def:meanfieldh}, this
yields for any $B \in \Pmatrix_+^q$, and $\pas >0$:
\begin{multline}
  \label{eq:problem-3}
 \quad s_\star = \prox_{\pas \, g}^{B}\left(s_\star + \pas
\ \precond(s_\star,B) \right) \\ \text{iff} \ \ s_\star \ \ \text{solves
  \eqref{eq:problem-1}.}
\end{multline}

\subsection{Variable Metric Proximal and Preconditioned Gradient}
\eqref{eq:problem-3} shows that when solving the composite
optimization problem \eqref{eq:problem-1}, as soon as a preconditioned
version of the operator $s \mapsto n^{-1} \sum_{i=1}^n G_i(s)$ is
used -- with preconditioning matrix $B^{-1}$, a proximity operator of
$g$ relative to a metric induced by the matrix $B$ has to be used.

Based on the characterization \eqref{eq:problem-3}, a natural
splitting algorithm to solve \eqref{eq:problem-1} under the condition
{\sf (c1)} is: given $\sigma_0 \in \Sset$, a positive stepsize
sequence $\{\pas_{k+1}, k \geq 0 \}$ and a $\Pmatrix_+^q$-valued
sequence $\{B_{k+1}, k \geq 0\}$, repeat
\begin{equation} \label{eq:VMFB}
\sigma_{k+1} = \prox_{\pas_{k+1} \, g}^{B_{k+1}}\left( \sigma_k + \pas_{k+1}  \precond(\sigma_k,B_{k+1}) \right)\eqsp.
\end{equation}
  It corresponds to the Variable Metric Forward-Backward algorithm
  (see e.g. \cite{chen:rockafeller:1997,combettes:vu:2014}).

  In the large scale learning setting, the full sum over the $n$
  functions $\precond_i$ (see \eqref{eq:def:meanfieldh}) can not be
  computed at each iteration of \eqref{eq:VMFB}.  In addition, it may
  happen that $\precond_i(s)$ is not explicit (see e.g. the case of
  the incremental EM algorithms, \Cref{sec:EM}). Therefore, a natural
  idea is to propose the inexact version of \eqref{eq:VMFB} defined by
  \Cref{algo:inexactVMFB}: the proximal step is unchanged (see
  \cref{algo1:prox}); the SA step in \cref{algo1:SA} uses a random
  approximation $\Smem_{k+1}$ of the exact mean field $n^{-1}
  \sum_{i=1}^n \precond_i(\hatS_k)$; this approximation, defined by
  \cref{algo1:meanfield}, combines a mini-batch approximation of a
  full sum (see \cref{algo1:minibatch}) and possibly approximated
  terms $\delta_{k+1,i}$ (see \cref{algo1:delta}).

\begin{algorithm}
\caption{A stochastic Variable Metric Forward-Backward \label{algo:inexactVMFB}}
\begin{algorithmic}[1]
  \Require $\kout \in \nset^\star$, $\pas_{k} > 0$ for $k \in
           [\kout]^\star$, $\lbatch \in \nset^\star$,
           $\hatS_{\mathrm{init}} \in \Sset$.  \Ensure The sequence
           $\{\hatS_{k}, k \in [\kout]\}$.  \State $\hatS_{0} =
           \hatS_{\mathrm{init}}$ \For{$k=0,\cdots, \kout-1$}
           \State \label{algo1:minibatch} Sample a batch $\batch_{k+1}$
           of size $\lbatch$ in $[n]^\star$ \State Choose $B_{k+1} \in \Pmatrix_+^q$ \State \label{algo1:delta}
           For $i \in \batch_{k+1}$, compute an approximation
           $\delta_{k+1,i}$ of $\precond_i(\hatS_{k}, B_{k+1})$.
           \State \label{algo1:meanfield} $\Smem_{k+1} = \lbatch^{-1}
           \sum_{i \in \batch_{k+1}} \delta_{k+1,i}$
           \State \label{algo1:SA} $\hatS_{k+1/2} = \hatS_{k} +
           \pas_{k+1} \ \Smem_{k+1}$ \State \label{algo1:prox}
           $\hatS_{k+1} = \prox_{\pas_{k+1} \,
             g}^{B_{k+1}}(\hatS_{k+1/2})$.  \EndFor
\end{algorithmic}
\end{algorithm}

\subsection{The {\tt SPIDER} variance reduction technique}
\label{sec:spider}
\3PS leverages on \Cref{algo:inexactVMFB} and on the variance
reduction technique {\tt SPIDER} for the definition of the field
$\Smem_{k+1}$ that approximates $n^{-1} \sum_{i=1}^n
\precond_i(\hatS_k, B_{k+1})$.  {\tt SPIDER} stands for {\it
  Stochastic Path-Integrated Differential EstimatoR}, and was
originally introduced in the stochastic gradient descent literature by
\cite{fang:etal:2018} (see also
\cite{nguyen:etal:2017,wang:etal:2019}).  We give the intuition of
     {\tt SPIDER} in the SA setting which encompasses the stochastic
     gradient one.
     
     SA scheme solves a root finding problem $\xi(s) = 0$ on $\rset^q$
     by: given an initial value $s_0 \in \rset^q$ and a stepsize
     sequence $\{\pas_{k+1}, k \geq 0\}$, repeat $ s_{k+1} = s_k +
     \pas_{k+1} \, \Xi_{k+1}$, where at each iteration $\# (k+1)$,
     $\Xi_{k+1}$ is a random approximation of $\xi(s_k)$. Usually, it
     is required that conditionally to the past of the algorithm, the
     expectation of $\Xi_{k+1}$ is $\xi(s_k)$; in that case,
     $\Xi_{k+1}$ can be replaced with $\Smem_{k+1} \eqdef \Xi_{k+1} +
     V_{k+1}$, where conditionally to the past, $V_{k+1}$ is centered.
     {\tt SPIDER} leverages on this remark and on the {\it control
       variate} technique: it proposes a clever construction of a
     random variable $V_{k+1}$ approximating zero and correlated to
     $\Xi_{k+1}$.

The recipe is as follows: consider that at iteration $\# k$,
$\Smem_{k}$ is a random approximation of $\precond(\hatS_{k-1},B_{k})$. Then
define $\Smem_{k+1}$ by $ \Smem_{k+1} \eqdef H_{k+1} + V_{k+1}$ where
\begin{align*}
H_{k+1} & \eqdef \lbatch^{-1} \sum_{i
  \in \batch_{k+1}} \precond_i(\hatS_k, B_{k+1}) \eqsp, \\
V_{k+1} & \eqdef \Smem_k - \lbatch^{-1} \sum_{i \in \batch_{k+1}}
\precond_i(\hatS_{k-1},B_k) \eqsp,
\end{align*}
and $\batch_{k+1}$ is sampled at random in $[n]^\star$.  The
r.v. $V_{k+1}$ approximates zero since both $\lbatch^{-1} \sum_{i \in
  \batch_{k+1}} \precond_i(\hatS_{k-1},B_k)$ and $\Smem_{k}$ approximate
$n^{-1} \sum_{i=1}^n \precond_i(\hatS_{k-1},B_k)$; $V_{k+1}$ and $H_{k+1}$ are
correlated via $\batch_{k+1}$.

Unfortunately, the r.v. $\Smem_{k+1}$ is not an unbiased approximation
of $n^{-1} \sum_{i=1}^n \precond_i(\hatS_k,B_{k+1})$ (see
\Cref{prop:bias} in the case $B_{k+1}$ is of the form $\B(\hatS_k)$).  In order to remove the bias, {\tt SPIDER} restarts
the control variate mechanism regularly: every $k^{\mathrm{in}}$
iterations, compute a full sum over the $n$ terms and set
$\Smem_{k^{\mathrm{in}}+1} = n^{-1} \sum_{i=1}^n
\precond_i(\hatS_{k^{\mathrm{in}}},B_{k^{\mathrm{in}}+1})$.

\subsection{\3PS}
\label{sec:3PS}
\begin{algorithm*}
\caption{The Perturbed Proximal Preconditioned SPIDER algorithm (\3PS) \label{algo:3PS}}
\begin{algorithmic}[1]
  \Require $\kout \in \nset^\star$, $\kin{t} \in \nset^\star$ for $t
  \in [\kout]^\star$, $\pas_{t,k+1} > 0$ for $t \in [\kout]^\star, k
  \in [\kin{t}] $, $\lbatch \in \nset^\star$, $\lbatch'_t \in
  \nset^\star$ for $t \in [\kout]^\star$, $\hatS_{\mathrm{init}} \in
  \Sset$ and $B_{\mathrm{init}} \in \Pmatrix_+^q$  \Ensure The sequence $\{\hatS_{t,k}, t \in [\kout]^\star, k
  \in [\kin{t}]^\star\}$.  \State $\hatS_{0,\kin{0}} =
  \hatS_{\mathrm{init}}$, $B_{0,\kin{0}} =
  B_{\mathrm{init}}$ \For{$t=1,\cdots, \kout$}
  \State \label{algo2:init1} $\hatS_{t,0} = \hatS_{t-1,\kin{t-1}}$,
  \ \ $\hatS_{t,-1} = \hatS_{t-1,\kin{t-1}}$, \ \ $B_{t,0} = B_{t-1,\kin{t-1}}$  \State Sample a batch
  $\batch_{t,0}$ of size $\lbatch'_t$ in $[n]^\star$, with or without
  replacement.  \State For all $i \in \batch_{t,0}$, compute
  $\delta_{t,0,i}$ equal to or approximating $\precond_i(\hatS_{t,0},B_{t,0})$.  \State \label{algo2:init2} $\Smem_{t,0}
  = (\lbatch'_t)^{-1} \sum_{i \in \batch_{t,0}} \delta_{t,0,i}$
  \For{$k = 0, \cdots, \kin{t}-1$} \State \label{algo2:batch} Sample a
  mini batch $\batch_{t,k+1}$ of size $\lbatch$ in $[n]^\star$, with
  or without replacement. \State \label{algo2:matrix} Choose $B_{t,k+1} \in \Pmatrix_+^q$.
  \State \label{algo2:delta} For all $i \in
  \batch_{t,k+1}$, compute $\delta_{t,k+1,i}$ equal to or approximating $\precond_i(\hatS_{t,k},B_{t,k+1}) - \precond_i(\hatS_{t,k-1},B_{t,k})$.
  \State \label{algo2:meanfield} $\Smem_{t,k+1} = \Smem_{t,k} +
  \lbatch^{-1} \sum_{i \in \batch_{t,k+1}} \, \delta_{t,k+1,i}$
  \State \label{algo2:SA} $\hatS_{t,k+1/2} = \hatS_{t,k} +
  \pas_{t,k+1} \ \Smem_{t,k+1}$ \State \label{algo2:prox}
  $\hatS_{t,k+1} = \prox_{t,k}(\hatS_{t,k+1/2})$, \qquad \qquad where
  $\prox_{t,k} \eqdef \prox_{\pas_{t,k+1} \, g}^{B_{t,k+1}}$.
  \EndFor \EndFor
\end{algorithmic}
\end{algorithm*}

\3PS  is
given by \Cref{algo:3PS}. The iteration index is $(t,k)$ where $t$ is
the index of the current {\it outer} loop and ranges from $1$ to
$\kout$, and $k$ is the index of the current {\it inner} loop. At
outer loop $\# t$, there are $\kin{t}$ inner iterations.  The inner
iterations are \Cref{algo:inexactVMFB} (see
\Cref{algo2:batch,algo2:matrix,algo2:SA,algo2:prox} of \Cref{algo:3PS}) combined
with the {\tt SPIDER} variance reduction trick (see
\Cref{algo2:meanfield} of \Cref{algo:3PS}) adapted to the
case when the quantities $\precond_i(\hatS_{t,k},B_{t,k+1}) -
\precond_i(\hatS_{t,k-1},B_{t,k})$ can not be computed exactly (see \Cref{algo2:delta}).

When $G_i$ is a gradient and $B(s) = \Id_q$, different strategies were
proposed for {\tt SPIDER} for the choice of $\lbatch'_t$ and
$\kin{t}$.  In \cite{fang:etal:2018,nguyen:etal:2017,wang:etal:2019},
the number of inner loops is constant ($\kin{t}$ = $\kin{}$ for any $t
\geq 1$) and $\lbatch'_t =n$; \cite{nguyen:etal:2017} also considers
the case when $\kin{t}$ is adapted based on the history of the
algorithm while being upper bounded; in \cite{horvath:etal:2022},
$\lbatch'_t$ is deterministic and depends on $t$, $\lbatch$ depends on
$t$, and $\kin{t}$ is a Geometric random variable with an expectation
depending on $t$; in \cite{li:etal:2021}, $\lbatch'_t$ does not depend
on $t$ and $\kin{t}$ is random. \\ For the EM problem (see
\Cref{sec:EM}), \cite{fort:etal:2020} introduced {\tt SPIDER-EM}, a
variance reduced stochastic EM designed for large scale learning, in a
situation when the computation of $\barprecond_i(s)$ is exact for all
$s,i$. For this algorithm, the benefit of an increasing batch size $t
\mapsto \lbatch'_t$ and a geometric number of inner loops $\kin{t}$
with time-varying expectation, is discussed in
\cite{fort:etal:2021:icassp}. The conclusion is that the best strategy
is a deterministic increasing sequence $\lbatch'_t$ in order to have
an increasing accuracy when refreshing the variable $\Smem_{\cdot}$, and
a constant number of inner loops $\kin{t} = \kin{}$. \\ This paper
allows $\lbatch'_t$ and $\kin{t}$ to vary with $t$: they may be
deterministic functions of $t$ or random ones as well.

The matrices $\{B_{t,k+1}, t \in [\kout]^\star, k \in [\kin{}-1]\}$
can be deterministic or random. They could be chosen prior the run of
the algorithm; more efficient strategies consist in adapting this
matrix along the run of the algorithm, based on its history. In EM
(see \Cref{sec:EM}), $B_{t,k+1}$ is of the form $\B(\hatS_{t,k})$
where $\B$ is defined by the statistic model.

After $\kin{t}$ inner iterations, the outer loop $\# (t+1)$ starts:
the stochastic mean field $\Smem_{t+1,0}$ is refreshed (see
\Cref{algo2:init1} to \Cref{algo2:init2}). Here again, two
approximations of the original {\tt SPIDER} algorithm are allowed: the
first one is when computing $\precond_i(\hatS_{t,0},B_{t,1})$ and the
second one avoids the scan of the full data set (one may choose
$\lbatch'_t < n$).

The input variables of \3PS\ are the number of outer loops $\kout$,
the number of inner loops $\kin{t}$, the stepsize sequence
$\{\pas_{t,k}, t \geq 1, k \geq 1 \}$ for the SA steps, the size of
the mini-batches $\lbatch$ and $\lbatch_t'$, and the initial values of
the iterate $\hatS_{\mathrm{init}}$ and the metric $B_{\init}$ in
$\Sset$ and $\Pmatrix_+^q$ respectively.

\subsection{Monte Carlo approximation of $\precond_i(s,B)$}
\label{sec:montecarlo:casehi}
Set $\vartheta \eqdef (s,i,B) \in \Sset \times [n]^\star \times
\Pmatrix_+^q$.  In some applications, there exist a measurable
function $H_{\vartheta}$ and a probability measure $\pi_{\vartheta}$
defined on the measurable set $(\Zset, \Zsigma)$ such that
\begin{equation}\label{eq:meanfield:expectation} 
\precond_i(s,B) = \int_\Zset H_{\vartheta}(z) \pi_{\vartheta}(\rmd z) \eqsp.
\end{equation}
This is the case of EM in the statistic space (see \Cref{sec:EM})
where $H_{\vartheta}(z) = S(Y_i,z) - s$ and
\[
\pi_{\vartheta}(\rmd z) \eqdef \frac{p(Y_i,z;
  \map(s))}{\int_\Zset p(Y_i,u; \map(s)) \, 
  \mu_{lv}(\rmd u)} \, \mu_{lv}(\rmd z) \eqsp.
\]
When the integral in \eqref{eq:meanfield:expectation} is intractable,
one can resort to Monte Carlo integrations to define the
approximations $\delta_{t,k+1,i}$ and $\delta_{t,0,i}$ (see
e.g. \cite{devroye:1986} for exact sampling methods, and \cite{robert:casella:2004} for an introduction to Markov chain
Monte Carlo methods). If $\{Z_{m}^{\vartheta}, m \geq 0\}$ are
independent samples with distribution $\pi_{\vartheta}(\rmd z)$ or are a path of an ergodic Markov
chain with unique invariant distribution $\pi_{\vartheta}(\rmd z)$, then we can
set
\[
\precond_i(\hatS_{t,k},B_{t,k+1}) \approx \frac{1}{M} \sum_{m=1}^M
H_{\vartheta_{t,k+1,i}}(Z_m^{\vartheta_{t,k+1,i}}) \eqsp,
\]
where $\vartheta_{t,k+1,i} \eqdef (\hatS_{t,k},i,B_{t,k+1})$.  We will
show numerically in \Cref{sec:application} that when approximating the
difference $\precond_i(\hatS_{t,k},B_{t,k+1}) -
\precond_i(\hatS_{t,k-1},B_{t,k})$, there is a gain in correlating the
two sequences $\{Z_m^{\vartheta_{t,k+1,i}}, m \geq 0 \}$ and $\{Z_m^{\vartheta_{t,k,i}} m \geq 0
\}$; this makes stronger the effect of the {\tt SPIDER} control
variate (see \Cref{sec:spider}).

\subsection{The computation of $\prox^B_{\pas g}$}
\label{sec:comput:proximity}
When $g=0$, $\prox_{\pas g}^B(s) = s$. When $g \neq 0$, $\mathsf{p}
\eqdef \prox_{\pas g}^B(s)$ solves $0 \in \mathsf{p} - s + \pas B^{-1}
\, \partial g(\mathsf{p})$ and there does not always exist an explicit
expression of $\mathsf{p}$.

When $B = \Id_q$, \cite[Tables 10.1 and 10.2]{combettes:pesquet:2011}
provide properties of $\prox_{\pas g}$ and expressions of proximity
operators for many functions $g$.

When $B$ is the sum of a diagonal matrix and of a rank one matrix,
\cite[Section 3]{becker:fadili:2012} presents iterative algorithms for
the computation of $\mathsf{p}$.  For a general positive definite
matrix $B$, we have from \cite[Example 3.9]{combettes:vu:2014}
\[
\prox_{\pas g}^B(s) = \sqrt{B}^{-1} \prox_{\pas g(\sqrt{B}^{-1}
  \cdot)}( \sqrt{B} s) \eqsp,
\]
where $\sqrt{B}$ is the square root of the matrix $B$.  \cite[Lemma
  5]{becker:fadili:2012} (see also \cite{combettes:vu:2014})
establishes a Moreau identity \ie\ an expression of $\prox_{\pas g}^B$
as a function of a proximity operator of the Fenchel conjugate of $g$.

In the special case $g$ is the $\{0, +\infty\}$-valued indicator
function of a closed convex set $\Sset$, the {\it projected Landweber}
method is an iterative algorithm for the computation of $\mathsf{p}$
(see \cite{eicke:1992}, see also \cite[Example
  10.10]{combettes:pesquet:2011}).

Finally, for applicatons including a metric selection step, metric
selection strategies for the definition of $B$ can be found in
\cite[Section 3]{park:etal:2019} for diagonal variable metrics; and in
\cite{repetti:wiaux:2021} for specific functions $g$ which circumvent
the often challenging computation of $\prox_{\pas g}^B$.

\section{Non-asymptotic convergence analysis}
\label{sec:theory}
This section is devoted to explicit non-asymptotic bounds for the
convergence in expectation of \3PS. We will restrict to the case there
exist $\B: \Sset \to \Pmatrix_+^q$ and 
\[
B_{t,k+1} \eqdef \B(\hatS_{t,k}) \eqsp.
\]
This framework encompasses the EM problem (see \Cref{sec:EM}) and any
preconditioned gradient-based algorithms (see \Cref{sec:VMFB}) when
the preconditioning matrix depends on the past history of the
algorithm via the current value of the iterate. We will also use the notation
\begin{equation}
  \label{eq:def:barmeanfieldh}
\barprecond_i(s) \eqdef \precond_i(s, \B(s))\eqsp, \quad \barprecond(s) \eqdef \precond(s, \B(s)) \eqsp. 
\end{equation}

\3PS\ is designed to solve \eqref{eq:problem-1} under the constraints
         {\sf c1} to {\sf c3}. Therefore, based on
         \eqref{eq:problem-3}, we are interested in a control of the
         quantities $\prox_{t,k}\left( \hatS_{t,k} + \pas_{t,k+1}
         \ \barprecond(\hatS_{t,k}) \right) - \hatS_{t,k}$ where
\[
\prox_{t,k}(s) \eqdef \prox_{\pas_{t,k+1}\, g}^{\B(\hatS_{t,k})}(s) \eqsp.
\]
Roughly speaking, these quantities evaluate how far the algorithm is
from the limiting set at iteration $\# (t,k)$.  More precisely, we
will control  the cumulative "distances to
stationary" $\sum_{t=1}^{\kout} \sum_{k=0}^{\kin{t}-1}
\Delta_{t,k+1}^\star$ where $\Delta_{t,k+1}^\star$ is equal to
\begin{equation} \label{eq:def:Deltastar}
\!\!\!\!\! \frac{\| \prox_{t,k}(\hatS_{t,k} + \pas_{t,k+1} \barprecond(\hatS_{t,k}))
  - \hatS_{t,k}\|_{\B(\hatS_{t,k})}^2}{\pas_{t,k+1}^2} \eqsp;
\end{equation}
$\precond$ is defined by \eqref{eq:def:meanfieldh}.

The controls in expectation of the cumulated distances are obtained under
the assumptions \Cref{hyp:globalL} to
\Cref{hyp:error}. \Cref{hyp:globalL} is a smoothness assumption on the
functions $\precond_i$, \Cref{hyp:W} assumes that $n^{-1} \sum_{i=1}^n
G_i(s)$ is a gradient operator of some so-called {\em Lyapunov function}, and the spectrum of the matrices $\B(s)$
are bounded uniformly in $s$. \Cref{hyp:error} are assumptions on the
approximations $\delta_{t,k+1,i}$.

\begin{assumption}
\label{hyp:globalL}
For all $i \in [n]^\star$, the function $\barprecond_i$ is globally
Lipschitz on $\Sset$, with constant $L_i$: there exists a positive
constant $L_i$ such that $\forall s,s' \in \Sset$, $\|
\barprecond_i(s) - \barprecond_i(s') \| \leq L_i \, \|s -s'\|$. Set
$L^2 \eqdef n^{-1} \sum_{i=1}^n L_i^2$.
\end{assumption}
\Cref{hyp:globalL} only requires a Lipschitz property on $\Sset$; it
is weaker than assuming the Lipschitz property on the full space
$\rset^q$ as sometimes assumed in the literature (see
e.g. \cite{combettes:wajs:2005}). \Cref{hyp:globalL} holds for example when $\Sset$ is
compact and for all $ i \in [n]^\star$, the gradient $\nabla \precond_i$ exists and is continuous
on $\Sset$.
\begin{assumption}
  \label{hyp:W}
  \begin{enumerate}[a)]
    \item  \label{hyp:W:item1} There exists a function $\lyap: \rset^q \to \rset$,
  continuously differentiable on $\Sset$ and such that 
\[
\forall s \in \Sset, \qquad \nabla \lyap(s) =  \frac{1}{n} \sum_{i=1}^n G_i(s) \eqsp;
\]
in addition, $\barprecond_i(s) = - \B(s)^{-1} \, G_i(s)$, where $\B(s) \in \Pmatrix_+^q$. 
\item  \label{hyp:W:item2}
  $\nabla \lyap$ is globally $L_{\dot \lyap}$-Lipschitz on $\Sset$.
  \item   \label{hyp:W:item3}
There exist $0 < v_{\min} \leq v_{\max} < + \infty$ such that for any
$s \in \Sset$,
$v_{\min} \| \cdot \|^2 \leq \| \cdot\|_{\B(s)}^2 \leq v_{\max} \|\cdot
\|^2$.
\end{enumerate}
\end{assumption}
Here again, both the Lipschitz property and the boundedness condition
on the spectrum of the matrices $\B(s)$ are required on $\Sset$ and not
on the full space $\rset^q$. When $\B(s)$ does not depend on $s$ ($\B(s)
= B$ for any $s \in \Sset$), we have $L_{\dot \lyap} \leq v_{\max}
n^{-1} \sum_{i=1}^n L_i$.

The last assumption is on the fluctuations of the errors when
approximating $\barprecond_i(\hatS_{t,k}) -
\barprecond_i(\hatS_{t,k-1})$: set $\xi_{t,k+1,i} \eqdef
\delta_{t,k+1,i}- \barprecond_i(\hatS_{t,k}) +
\barprecond_i(\hatS_{t,k-1}) $ and define its conditional bias and
variance, conditionally to the $\sigma$-field generated by
$\batch_{t,k+1}$, $\hatS_{t,k}$ and $\hatS_{t,k-1}$. Set
$\mathcal{P}_{t,k+1/2} \eqdef \sigma\left( \batch_{t,k+1},
\hatS_{t,k}, \hatS_{t,k-1} \right)$.
\begin{align*}
  \mu_{t,k+1,i} 
   & \eqdef \PE\left[ \xi_{t,k+1,i}   \vert  \mathcal{P}_{t,k+1/2}\right] \\
 \sigma^2_{t,k+1,i} & \eqdef \PE\left[ \| \xi_{t,k+1,i}  - \mu_{t,k+1,i}  \|^2 \vert  \mathcal{P}_{t,k+1/2} \right] \eqsp.
\end{align*}
We assume 
\begin{assumption}
  \label{hyp:error}
  \begin{enumerate}[a)]
    \item \label{hyp:error:indep} Conditionally to $\batch_{t,k+1}$, $\hatS_{t,k}$ and $\hatS_{t,k-1}$, the approximations $\{ \delta_{t,k+1,i}, i \in \batch_{t,k+1} \}$ are independent.
  \item \label{hyp:error:bias} There exists a non negative constant
    $C_b$ and for any $t \in [\kout]^\star$, there exists a non
    decreasing deterministic sequence $\{\nbrmc_{t,k}, k \geq 1 \}$ such that for
    any $k \in [\kin{t}-1]$, with probability one,
    \[
\| \frac{1}{n} \sum_{i=1}^n \mu_{t,k+1,i} \| \leq
\frac{C_b}{\nbrmc_{t,k+1}} \eqsp.
\]
\item \label{hyp:error:var} There exist non negative constants $C_v$
  and $C_{vb}$ and for any $t \in [\kout]^\star$, there exist non
  decreasing deterministic sequences $\{\nbrmcv_{t,k}, k \geq 1 \}$
  and $\{\bar \nbrmcv_{t,k}, k \geq 1 \}$ such that for any $k \in
  [\kin{t}-1]$, with probability one,
  \begin{align*}
& \frac{1}{n} \sum_{i=1}^n \sigma^2_{t,k+1,i} \leq
  \frac{C_v}{\nbrmcv_{t,k+1}} \eqsp,  \\
  & \frac{1}{n} \sum_{i=1}^n  \| \mu_{t,k+1,i}-  \frac{1}{n} \sum_{j=1}^n \mu_{t,k+1,j}\|^2  \leq \frac{C_{vb}^2}{\bar \nbrmcv_{t,k+1}^2}\eqsp.
 \end{align*}
\end{enumerate}
\end{assumption}
We allow the errors $\xi_{t,k+1,i}$ to be deterministic or random. When there are no errors ($\xi_{t,k+1,i} = 0$) then $C_b = C_v = C_{vb} = 0$. When the errors are deterministic,  we have $\xi_{t,k+1,i} = \mu_{t,k+1,i}$ and $\sigma^2_{t,k+1,i} =0$. When the errors are random and unbiased, then $\mu_{t,k+1,i} = 0$.  Therefore, some of the constants $C_b$, $C_v$ or $C_{vb}$ can be null as summarized in \Cref{table:ctt}.

\begin{table}[ht]
  \begin{centering}
  \begin{tabular}{||l||c|c|c||}
      \hline
      & $C_b$ & $C_v$ & $C_{vb}$ \\ 
      \hline\hline
      exact & $0$ & $0$ & $0$  \\ \hline
      deterministic & $ \geq 0$ & $0$ & $ \geq 0 $ \\
      random, unbiased & $0$ & $ \geq 0$ & $0$ \\
      random, biased & $ > 0$ & $ \geq 0$ & $\geq 0$ \\ [1ex]
      \hline
  \end{tabular}
  \caption{The sign of the constants $C_b$, $C_v$, $C_{vb}$ when there
    are no approximations on the $\barprecond_i(s)'$ (case {\em exact}),
    and when there are approximations. \label{table:ctt}}
  \end{centering}
  \end{table}

In \Cref{sec:hyp:montecarlo}, we
discuss  how  \Cref{hyp:error} is verified in the case $\barprecond_i(s') - \barprecond_i(s)$
is an expectation under a distribution that may depend on $(s,s',i)$ (see
\Cref{sec:montecarlo:casehi}),
and $\delta_{t,k+1,i}$ is a Monte Carlo approximation.

\Cref{theo:main:general} provides an explicit upper bound of the
cumulative distance to stationary $\Delta^\star_{t,k+1}$ (see
\eqref{eq:def:Deltastar}) along the $\sum_{t=1}^{\kout} \kin{t}$
iterations of the algorithm. It also provides an upper bound on the
cumulative errors $\mathcal{D}^\star_{t,k+1}$ defined by
  \[
 \frac{ \| \hatS_{t,k+1} - \prox_{t,k}(\hatS_{t,k}+ \pas_{t,k+1} \barprecond(\hatS_{t,k})) \|_{\B(\hatS_{t,k})}^2}{\pas_{t,k+1}^2} \eqsp,
  \]
where $\barprecond$ is defined by \eqref{eq:def:barmeanfieldh}. Given
the current iterate $\hatS_{t,k}$, $\mathcal{D}_{t,k+1}^\star$
compares two iterations: the ideal one $\prox_{t,k}(\hatS_{t,k}+
\pas_{t,k+1} \barprecond(\hatS_{t,k}))$ and the available one
$\prox_{t,k}(\hatS_{t,k}+ \pas_{t,k+1} \Smem_{t,k+1})$.

\begin{theorem} \label{theo:main:general}
  Assume \Cref{hyp:g}, \Cref{hyp:globalL}, \Cref{hyp:W} and
  \Cref{hyp:error}.  Let $\{\kin{t}, t \in [\kout]^\star\}$ be a
  deterministic positive sequence. For any $t \in [\kout]^\star$ and
  $k \in [\kin{t}-1]$, define $\Lambda_{t,k+1}$ by
          \begin{equation}\label{eq:def:Lambda}
          \frac{\pas_{t,k} L_{\dot \lyap}}{v_{\min}} + \pas^2_{t,k}
          L^2 \frac{2v_\max \kin{t} }{v_{\min} \lbatch} \left(1 +
          \frac{2 \, C_{vb}}{ \sqrt{\lbatch} \, \bar \nbrmcv_{t,k+1}}
          \right) \eqsp.
          \end{equation}
  Let $\{\hatS_{t,k}, t \in [\kout]^\star, k \in [\kin{t}]^\star \}$
  be the sequence given by \Cref{algo:3PS} when the stepsize sequence
  $\{\pas_{t,k+1}, t \in [\kout]^\star, k \in [\kin{}-1]\}$ satisfies
  \begin{equation}\label{eq:conditions:stepsize}
\pas_{t,k+1} \left(1 + \frac{2 C_b}{\nbrmc_{t,k+1}} \right) \leq
\pas_{t,k} \eqsp, \quad \Lambda_{t,k+1} \in \ooint{0,1/2} \eqsp.
  \end{equation}
Then, 
  \begin{align*}
   & \sum_{t=1}^{\kout} \sum_{k=1}^{\kin{t}} \pas_{t,k} \, \left(
    \frac{1}{2} - \Lambda_{t,k+1} \right) \left\{ \PE\left[
      \Delta_{t,k}^\star \right] + \PE\left[ \mathcal{D}_{t,k}^\star
      \right] \right\} \\ & \qquad \leq \PE\left[ \lyap(\hatS_{1,0}) +
      g(\hatS_{1,0})\right] - \min_\Sset \left(\lyap +g \right) \\ &
    \qquad + v_\max \sum_{t=1}^{\kout} \pas_{t,0} \, \kin{t} \,
    \PE\left[ \| \mathcal{E}_t \|^2 \right] \\ & \qquad + v_\max
    \sum_{t=1}^{\kout} \sum_{k=1}^{\kin{t}} \left(\kin{t} -k+1\right)
    \pas_{t,k} \, \mathcal{U}_{t, k} \eqsp,
 \end{align*}
  where $\mathcal{E}_t \eqdef \Smem_{t,0} - \precond(\hatS_{t,0})$ and
 \[  \mathcal{U}_{t,k} \eqdef \frac{2 \, C_b}{\nbrmc_{t,k}} +
  \frac{C_b^2}{\nbrmc_{t,k}^2} + \frac{C_v}{\lbatch \, \nbrmcv_{t,k}}  +
  \frac{2 \, C_{vb}}{\sqrt{\lbatch} \,\bar \nbrmcv_{t,k}}  + \frac{C^2_{vb}}{\lbatch \,\bar \nbrmcv_{t,k}^2}\eqsp.
\]
\end{theorem}
The proof of \Cref{theo:main:general} is given in
\Cref{sec:proof:theogeneral}. Note that $\mathcal{U}_{t,k+1} =0$ when
the algorithm uses exact preconditioned gradients at each iteration:
$\delta_{t,0,i} = \barprecond_i(\hatS_{t,0})$ and $\delta_{t,k+1,i} =
\barprecond_i(\hatS_{t,k}) - \barprecond_i(\hatS_{t,k-1})$ for all
$i,t,k$.

\bigskip

{\bf Random number of inner loops $\kin{t}$.} When the number of inner
loops $\kin{t}$ at the outer loop $\# t$ is a random number, we
consider it is drawn prior the run of the algorithm. Therefore the
expectations in \Cref{theo:main:general} are conditionally to the
random sequence $\{\kin{t}, t \in [\kout]^\star\}$. The expectation
w.r.t. the randomness of $\kin{t}$ can easily be obtained from
\Cref{theo:main:general}; details are omitted.

\bigskip

{\bf The step sizes $\pas_{t,k}$.}  The conditions on the sequence
$\{\pas_{t,k+1}, t \in [\kout]^\star, k \in [\kin{t}-1] \}$ are
satisfied with
\[
\pas_{t,k+1} \eqdef \prod_{j=0}^k \left(1 + \frac{2 \,
  C_b}{\nbrmc_{t,j+1}} \right)^{-1} \pas_{t,0}
\]
where $\pas_{t,0}$ is positive and strictly lower than
\begin{equation}\label{eq:stepsizeinit}
 \frac{1}{4 L v_{\max} \upsilon} \frac{\lbatch}{\kin{t}} 
 \left(\sqrt{\frac{L_{\dot \lyap}^2}{L^2} + 4 v_{\min} v_\max
   \frac{\kin{t}}{ \lbatch} \upsilon} - \frac{L_{\dot
     \lyap}}{L}\right) \eqsp;
\end{equation}
 $\upsilon \eqdef 1+ 2 C_{vb}/(\sqrt{\lbatch} \, \inf_{t,k}\bar
\nbrmcv_{t,k+1})$ (see the proof in \Cref{app:sec:stepsizeinit}).
First, observe that when $C_b =0$, the step size can be a constant
function of the inner loop index $k$ loop ($\pas_{t,k+1} = \pas_{t,0}$
for any $k$). On the contrary, when $C_b >0$ \ie\ for a deterministic
approximation or a biased random approximation (see \Cref{table:ctt}),
the stepsize sequence is a strictly decreasing function of the inner
loop index $k$.

Second, the maximal value of $\pas_{t,0}$ is larger when $C_{vb}=0$ than
when $C_{vb}>0$. Here again, deterministic or unbiased random approximations requires more aggressive step sizes.

\bigskip

{\bf The initialization of the outer loops.} Set $\mathcal{N}_t \eqdef
\|\mathcal{E}_t\|^2$. When $\batch_{t,0} = \{1,
\cdots, n\}$ and $\delta_{t,0,i} = \barprecond_i(\hatS_{t,0})$ for all
$i$, then $ \mathcal{N}_t =0$; otherwise, $\mathcal{N}_t$ is
positive.

Let us discuss the behavior of $\mathcal{N}_t$ when $\delta_{t,0,i}$
is an unbiased random approximation of $\barprecond_i(\hatS_{t,0})$
with variance denoted by $\sigma_{t,0,i}^2$. When $\batch_{t,0} = \{1,
\cdots, n\}$, then
\begin{equation}
  \label{eq:var:init:2}
\PE\left[ \mathcal{N}_t \right] = \frac{1}{n^2} \sum_{i=1}^n
\sigma^2_{t,0,i} \eqsp.
\end{equation}
Nevertheless, the strategy  $\batch_{t,0} = \{1, \cdots, n\}$ has a large computational cost; sampling a
subset of size $\lbatch'_t$ reduces the computational cost but
increases the squared norm of the error: we have
\begin{equation}
  \label{eq:var:init}
 \PE\left[\mathcal{N}_t \right] \leq \frac{1}{ \lbatch'_t n}
\sum_{i=1}^n \left( \sigma^2_{t,0,i} + \|\barprecond_i(s) - \barprecond(s)\|^2
\right) \eqsp,
\end{equation}
with an equality if $\batch_{t,0}$ is sampled with replacement in
$\{1, \cdots, n\}$.  See \Cref{sec:app:Epsilon} for detailed
computations. From a numerical point of view, an efficient strategy
consists in increasing the size $\lbatch'_t$ with the outer loop index
$t$ (see references in \Cref{sec:3PS} for \3PS\ applied to EM).

\bigskip

{\bf Random stopping time of the algorithm.} In non-convex
optimization, the last iterate $\hatS_{\kout,\kin{\kout}}$ is not
necessarily the point which minimizes, over the sequence
$\{\hatS_{t,k}, t \in [\kout]^\star, k \in [\kin{t}]^\star \}$, the
distance to the set of solutions of \eqref{eq:problem-1}. The quantity
$\Delta_{\cdot}^\star$, motivated by \eqref{eq:problem-3}, can not be
computed exactly in our framework so that the "best" iterate can not
be identified thanks to this criterion. It is therefore popular to
analyze the algorithm when stopped at a random time (see
e.g. \cite[Chapter 6]{lan:2020}). For sake of simplicity, we consider
the case when $\kin{t} = \kin{}$ for any $t$ and $C_b =0$. We have the
following corollary:
\begin{corollary}[of \Cref{theo:main:general}]
 \label{coro:maintheo}
Assume that $\kin{t} = \kin{}$, $C_b=0$ and the stepsize sequence is
constant $\pas_{t,k} = \pas_\star$.  Let $(\tau, K)$ be a uniform
random variable on $[\kout]^\star \times [\kin{}]^\star$, independent
of the algorithm. Then
 \begin{align*}
   & \inf_{(t,k) \in [\kout]^\star \times [\kin{}]^\star} \left(
   \frac{1}{2} - \Lambda_{t,k} \right) \, \PE\left[
     \Delta_{\tau,K}^\star + \mathcal{D}_{\tau,K}^\star \right] \\ &
    \leq \frac{ \PE\left[ \lyap(\hatS_{1,0}) +
       g(\hatS_{1,0})\right] - \min_\Sset \left(\lyap +g \right) }{
     \kout \kin{} \pas_\star} \\ &  + v_\max \PE\left[ \|
     \mathcal{E}_\tau \|^2 \right]  + v_\max \ \PE\left[ \left(
     \kin{} - K+1 \right) \mathcal{U}_{\tau, K} \right] \eqsp.
 \end{align*}
  \end{corollary}
An upper bound on $\Lambda_{t,k}$ can easily be obtained from
\eqref{eq:def:Lambda} as a function of $L_{\dot \lyap}, L, v_{\min},
v_{\max}$, $\kin{}, \lbatch, \pas_\star$, $C_{vb}$ and $\inf_{t,k}
\bar \nbrmcv_{t,k+1}$.

 \Cref{coro:maintheo} shows that, even by stopping \3PS\ with this simple rule,
the first term in the RHS is inversely proportional to the maximal
number of iterations $\kout \kin{}$. 

\bigskip

{\bf Complexity analysis when $\mathcal{E}_t = 0$, $\mathcal{U}_{t,k}
  = 0$ and $\kin{t} = \kin{}$.}  For smooth first-order optimization,
algorithms are compared through their complexity in order to satisfy
an $\epsilon$-first order stationary condition. In stochastic
composite optimization, this criterion is naturally extended to the
{\em approximate $\epsilon$-stationary condition} defined by
\[
\PE\left[ \Delta^\star_{\tau,K} \right] \leq \epsilon \eqsp,
\]
where $(\tau,K)$ is a random variable taking values in $[\kout]^\star
\times [\kin{}]^\star$; see e.g. \cite[Section 4]{ghadimi:etal:2016},
\cite[Section 3]{wang:etal:2019} and \cite{fort:moulines:2021}).

\Cref{coro:complexity} studies the proximal complexity
$\mathcal{K}_{\prox}$ defined as the number of calls to the $\prox$
operator in order to satisfy the approximate $\epsilon$-stationary
condition; the stochastic $\barprecond$-complexity
$\mathcal{K}_{\barprecond}$ defined as the number of calls to one of the
$\barprecond_i$'s; and the total number of iterations $\kin{}
\kout$. Again for sake of simplicity, and in order to compare our
results to the literature, we consider a simplified setting.
\begin{corollary}[of \Cref{coro:maintheo}]
  \label{coro:complexity}
Assume in addition that $\mathcal{E}_t = 0$ and $\mathcal{U}_{t,k} =
0$. The approximate $\epsilon$-stationary condition is satisfied with
$\pas_\star = v_{\min} / (4 \, L_{\dot \lyap})$, $\kin{}/\lbatch
=L_{\dot \lyap}^2/(v_{\min} v_{\max} L^2)$, $\lbatch = O(\sqrt{n}
\sqrt{ v_{\min} v_{\max}} L/L_{\dot \lyap})$ and $\kout \kin{}=
O(L_{\dot \lyap}/(\epsilon v_{\min}))$. Moreover, $\mathcal{K}_{\prox}
= O(L_{\dot \lyap}/(v_{\min} \, \epsilon))$ and
$\mathcal{K}_{\barprecond}= O(\sqrt{v_{\max}} L \sqrt{n}/(\epsilon
\sqrt{v_{\min}}))$.
  \end{corollary}
The proof is in \Cref{sec:proof:complexity}. This result shows that
the step size $\pas_\star$ and the number of inner loops $\kin{}$ are
independent of the accuracy $\epsilon$.

When applied to Stochastic Gradient Descent, \3PS\ in the setting of
\Cref{coro:complexity} is the {\tt Prox-SpiderBoost} algorithm studied
in \cite{wang:etal:2019}: \Cref{coro:complexity} and \cite[Theorem
  2]{wang:etal:2019} state the same complexity results. 
  \cite[Table 1]{wang:etal:2019} compares {\tt Prox-SpiderBoost} to
  other stochastic gradient algorithms for composite non-convex finite
  sum optimization. It is shown that the variance reduction based on
  {\tt SPIDER} order-level outperforms other variance reduction
  strategies such as the {\tt SVRG} one and the {\tt SAGA} on, introduced respectively by  \cite{johnson:zhang:2013} and \cite{defazio:etal:2014}. Hence,  \3PS\ reaches the state of the art among the proximal stochastic
gradient algorithms designed to solve finite sum non-convex composite
optimization.

When applied to EM, \3PS\ in the setting of \Cref{coro:complexity} is
the extension of the {\tt SPIDER-EM} algorithm studied in
\cite{fort:etal:2020} to the case there is a proximal step which
manages the constraint $g$. Here again, the comparison of
\Cref{coro:complexity} and \cite[Theorem 2]{fort:etal:2020} shows that
\3PS\ reaches the state of the art among the incremental EM
algorithms with variance reduction, including {\tt sEM-VR} and {\tt
  FIEM} introduced respectively in \cite{chen:etal:2018} and \cite{karimi:etal:2019} (see also \cite{fort:etal:2021}). See the comparison to the
literature in \cite{fort:etal:2020}.

Beyond these two applications, \Cref{coro:complexity} is - to our best
knowledge - the first complexity result for an algorithm designed to
solve \eqref{eq:problem-1} under the constraint \eqref{eq:problem-2}
and for non-convex finite sum composite optimization.

\bigskip

{\bf Approximate $\epsilon$-stationary condition: the cost of inexact
  preconditioned forward operators.} Let us discuss the cost of
inexact $\barprecond_i(s)$'s when the approximation is unbiased and
random (so that $C_b = C_{vb} = 0$, see \Cref{table:ctt}): does it
deteriorate the proximal complexity $\mathcal{K}_{\prox}$ and the
number of calls to an oracle of a preconditioned forward operator
$\barprecond_i$ (still denoted by $\mathcal{K}_{\barprecond}$ below) ?
detailed computations of the assertions below can be found in
\Cref{sec:proof:MCcost}.

If $\PE\left[ \| \mathcal{E}_t \|^2 \right] = O(\epsilon^{1-\pa'}/
(\sqrt{n} t)^{\pa'})$ for some $\pa' \in \coint{0,1}$ and
\[
 \nbrmcv_{t,k+1} = O\left(  \frac{n^{(\pa - \bar \pa)/2}}{\epsilon^{1-\pa}} \, t^{\pa} \, (k+1)^{\bar \pa} \right) 
 \]
 for some $\pa, \bar \pa \in \coint{0,1}$, then the approximate
 $\epsilon$-stationary condition is satisfied with $\kin{t} =
 O(\sqrt{n})$, $\lbatch = O(\sqrt{n})$ and $\kout = O(1/(\sqrt{n}
 \epsilon))$. In addition, $\mathcal{K}_{\prox} = O(1/ \epsilon)$ and
 $\mathcal{K}_{\barprecond}= O( \sqrt{n}/\epsilon)$. Therefore, the
 conclusions of \Cref{coro:complexity} remain valid, and the
 approximations of the $\precond_i$'s do not deteriorate the
 complexity performances of the algorithms, as soon as the
 approximation is small enough.

 Let us now evaluate the computational cost, in the case the unbiased
 random approximation is a Monte Carlo approximation computed from
 independent and identically distributed (i.i.d.) samples. In this case, $\nbrmcv_{t,\cdot}$ is the number of
 terms of the Monte Carlo sum (see \Cref{sec:hyp:montecarlo}). The
 Monte Carlo complexity $\mathcal{K}_{{\sf MC}}$ defined as the total number
 of Monte Carlo draws required to satisfy the approximate
 $\epsilon$-stationary condition is: $\mathcal{K}_{{\sf MC}} =
 O(\sqrt{n}/\epsilon^2)$ for any $\pa, \pa', \bar \pa \in
 \coint{0,1}$.

 To our best knowledge, it is the first complexity analysis with such
 a Monte Carlo approximation of the preconditioned forward operators $\barprecond_i$'s.

 \section{Application: Penalized Logistic Regression with random effects}
 \label{sec:application}
 \subsection{The model}
 Motivated by applications in classification, we consider a logistic
 regression model with random effects.

Let $n$ pairs of examples  $\{(X_i, Y_i), i \in [n]^\star \}$ where $X_i \in \rset^d$ collects the $d$ explanatory variables, and $Y_i$ is the binary response variable taking values in $\{-1, 1\}$. We assume that given $\{X_i, i \in [n]^\star \}$,  the
binary observations $\{Y_i, i \in [n]^\star \}$ are independent with
distribution
\begin{multline*}
 \{-1,1\} \ni  y_i \mapsto  \int_{\rset^d} (1+\exp(-y_i
\pscal{X_i}{z_i}))^{-1} \\ \times \frac{1}{\sqrt{2\pi}^d \sigma^d} \exp\left(-(2 \sigma^2)^{-1} \| z_i
- \theta \|^2 \right) \rmd z_i \eqsp.
\end{multline*}
In words, each
example $\# i$  has an individual regression vector
$Z_i$ in $\rset^d$ and given $Z_i$, the success probability $\PP(Y_i = 1 \mid Z_i)$ is
$(1+\exp(-\pscal{X_i}{Z_i}))^{-1}$. The regression vectors $Z_1,
\cdots, Z_n$ are  independent with  a Gaussian distribution $\mathcal{N}(\param, \sigma^2 \Id_d)$. 
$\theta$ is assumed to be unknown and $\sigma^2$ is known.

The objective is the estimation of $\param$ by maximizing the penalized log-likelihood criterion, with a ridge penalty $\mathrm{pen}(\param) \eqdef n \tau \| \param \|^2$, where $\tau>0$. By a change of variable, we obtain that the criterion to be minimized is (see \Cref{lem:changevar:penll})
\begin{multline*}
  F: \param \mapsto   
  - \frac{1}{n} \sum_{i=1}^n \log  \int_{\rset}   \frac{\exp\left( x \pscal{X_i}{\param}/(\sigma^2 \|X_i\|)  \right)}{1+\exp(-y_i \|X_i\| x
    )} \\
  \times \exp\left(- x^2 / (2\sigma^2)  \right)\rmd x + \|\param\|_{U}^2   \eqsp, 
\end{multline*}
where
\[
 U\eqdef \tau \Id_d + \frac{1}{2 \sigma^2} \frac{1}{n}\sum_{i=1}^n \frac{X_i X_i^\top}{
\|X_i\|^2} \eqsp.
\]
The following lemma shows that the minimizers of $F$ are in a
compact set $\mathcal{K}$ of $\rset^d$ thus implying that the
optimization problem can be constrained to $\mathcal{K}$. The proof is given in \Cref{sec:proof:lem:compactset}.
\begin{lemma}
  \label{lem:compactset}
The minimizers of $F$ are in the set $\mathcal{K} \eqdef \{\param \in \rset^d: \|\param\|^2 \leq (\ln 4) / \tau \}$.
  \end{lemma}
To solve this optimization problem, we propose two approaches: a
gradient one, solved in the original space $\param \in \rset^d$ (see
\Cref{sec:VMFB}); and an EM one, solved in the statistic space (see
\Cref{sec:EM}). The discussions in \Cref{sec:appli:Gdt} and
\Cref{sec:appli:EM} show that EM is a gradient approach for finding
the critical points of $s \mapsto F( U^{-1} s/2)$.

\subsection{A Gradient approach}
\label{sec:appli:Gdt}
We are interested in finding a critical point of $F$ in
$\mathcal{K}$. Equivalently, we want to solve
\[
0 \in \frac{1}{n} \sum_{i=1}^n G_i(\param) + \partial g(\param)
\]
where  $g$ is the $\{0, +\infty\}$-valued indicator function of the set $\mathcal{K}$ and 
\[
G_i(\param) \eqdef 2 U \param - \frac{X_i}{\sigma^2 \, \|X_i\|} \int_{\rset}  z \ \pi_{\param,i}(z) \rmd z \eqsp;
\]
$\pi_{\param,i}(z)$ is the probability density proportional to
\begin{equation} \label{eq:EM:pi}
\frac{\exp\left( z \pscal{X_i}{\param}/(\sigma^2 \|X_i\|)  - z^2 / (2\sigma^2)  \right)}{1+\exp(-y_i \|X_i\| z 
    )}  \eqsp.
\end{equation}
We apply \3PS\ with $B \eqdef \Id_q$ and
$\precond_i(\param,\Id_q) \eqdef -G_i(\param)$; note that $\prox_{\pas
  \, g}(\param) = \mathrm{argmin}_{x \in \mathcal{K}} \|x - \param
\|^2$.  $\precond_i$ is the sum of an explicit term and an integral
with no closed form: it will be approximated by a Monte Carlo method,
based on a Markov chain Monte Carlo (MCMC) sampler (see
\Cref{sec:MC:polyagamma} below). Therefore, $\delta_{t,k,i}$ will be a
biased random approximation.

\subsection{An EM approach}
\label{sec:appli:EM}
The criterion $F$ to be minimized  is of the form \eqref{eq:EM:LogLike} with $\Zset = \rset$, $\mu_{lv}(\rmd z) = \rmd z$ and
$p(Y_i,z; \theta)$ equal to
\[
\frac{\exp\left( z \pscal{X_i}{\param}/(\sigma^2 \|X_i\|)   - z^2 / (2\sigma^2)-  \| \param \|^2_{U} \right)}{1+\exp(-Y_i \|X_i\| z
    )}  \eqsp.
\]
The curved exponential family assumption on the complete data model is satisfied: $p(Y_i,z; \theta) = H(Y_i,z) \, \exp\left(\pscal{S(Y_i,z)}{\phi(\param)} - \psi(\param) \right)$ with $\phi(\param) \eqdef \param$,  $\psi(\param) \eqdef \| \param\|^2_{U}$ 
and
\[
S(Y_i,z) \eqdef z \, \frac{X_i}{ \sigma^2 \,\|X_i\|} \eqsp.
\]
From \Cref{sec:EM}, EM in the statistic space is of the form
\eqref{eq:problem-1}-\eqref{eq:problem-2}: it solves $0 \in n^{-1}
\sum_{i=1}^n \bar G_i(s) + \partial \bar g(s)$ where $\bar G_i(s) = B
\, G_i(B s)$, $B \eqdef U^{-1}/2$ and $\bar g(s)$ is the $\{0, +
\infty\}$-valued indicator function of the set $\{s \in \rset^d:
\map(s) \in \mathcal{K} \}$ where $\map(s) \eqdef B s$; it uses
\begin{equation} \label{eq:appli:precondEM}
\barprecond_i(s) \eqdef   \frac{X_i}{\sigma^2 \, \|X_i\|}  \int_\rset z \, \pi_{B s,i}(z) \rmd z - s \eqsp,
\end{equation}
and the metric induced by $\B(s) \eqdef B$. See
\Cref{sec:proof:appli:hinEM} for detailed computations. As in the
gradient approach, $\precond_i$ requires the expectation of the
distribution $\pi_{\cdot,i}$ (see \eqref{eq:EM:pi}) which has no
closed form. We will run \3PS\ with $B(s) \leftarrow B$ and a biased
random approximation of the $\barprecond_i(s)$'s (see
\Cref{sec:MC:polyagamma}); note that $ \prox_{\pas \, \bar g}^{B}(s)
= B^{-1} \, \mathrm{argmin}_{x \in \mathcal{K}} \left( (x-B s)^\top
B^{-1} (x- B s) \right)$.

\subsection{The MCMC approximation of $\barprecond_i$}
\label{sec:MC:polyagamma}
We discuss how to design an efficient MCMC sampler for the approximation of
\[
\mathcal{I}_i(\param) \eqdef \int_{\rset} z \, \pi_{\param,i}(z) \, \rmd z \eqsp, \quad \param \in \rset^d \eqsp,
\]
where $\pi_{\param,i}$ is defined, up to a normalizing constant, by \eqref{eq:EM:pi}.
By using an integration by parts and by applying \cite[Theorem 1]{polson:scott:windle:2013}, we show that a
data augmentation scheme is possible to approximate integrals
w.r.t. $\pi_{\param,i}(z)$. 
\begin{lemma}
  \label{lem:appli:MCMC:IPPandPolyagamma}
  For any $i \in [n]^\star$ and $\param \in \rset^d$, it holds
  \begin{multline*}
  \mathcal{I}_i(\param)  = \pscal{\frac{X_i}{\|X_i\|}}{\param}  \\
  + y_i \|X_i\| \sigma^2 \int_\rset \int_0^{+\infty} \frac{\bar \pi_{\param,i}(z,\omega)}{1+\exp\left( y_i \|X_i\| z \right)}  \, \rmd z \rmd \omega
  \eqsp,
  \end{multline*}
where $\bar \pi_{\param,i}(z,\omega)$ is a probability density on
  $\rset \times \ooint{0,+\infty}$. The conditional distribution of $z$ given $\omega$  is a Gaussian distribution with parameters
  \[
\frac{\pscal{X_i}{\param}/\|X_i\| + y_i  \|X_i\| \sigma^2/2}{1+\omega \sigma^2 \|X_i\|^2} \eqsp, \qquad \frac{\sigma^2}{1+\omega \sigma^2 \|X_i\|^2} \eqsp;
  \]
 the conditional distribution of $\omega$ given $z$ is a Polya-Gamma distribution with parameters $(1, \|X_i\| z)$.
  \end{lemma}
The proof is given in \Cref{sec:proof:lem:MCMC:IPP}. Therefore, a
Monte Carlo approximation of integrals w.r.t. $\pi_{\param,i}$ are
obtained from a Gibbs sampler targeting the distribution $\bar
\pi_{\param,i}(z,\omega)$: it produces a sequence of pairs $\{(Z_r,
\Omega_r), r \geq 0 \}$ and only the $Z_r$'s are retained for the
Monte Carlo approximation. For example, $\barprecond_i(s)$ given by
\eqref{eq:appli:precondEM} can be approximated by
\begin{multline} \label{eq:appli:precondMC}
  \barprecond_i(s) \approx -s + \frac{X_i}{\sigma^2 \|X_i\|^2} \pscal{X_i}{B s} \\
  + y_i \|X_i\| \frac{1}{m} \sum_{r=1}^m \left(1 + \exp(y_i \|X_i\| Z_r^{s,i}) \right)^{-1}  \eqsp.
\end{multline}

This Gibbs sampler is uniformly ergodic (see \cite[Proposition
  3.1]{choi:hobert:2013}); consequently, upon noting that $z \mapsto
H_i(z) \eqdef (1 + \exp( y_i \|X_i\| z))^{-1}$ is bounded by one
uniformly in $i$ and $z$, the conditions \Cref{hyp:MCMC} in
\Cref{sec:hyp:montecarlo} are verified with $U$ equal to the constant
function $1$ and with a geometric convergence rate $\rho(r) \eqdef
\upsilon^r$ for some $\upsilon \in \ooint{0,1}$ (remember that $\Sset$
is a compact set in our application); details are provided in
\Cref{sec:appli:uniformergoPolyaGamma}.

Therefore, \Cref{hyp:error} is verified and the rates
$\nbrmc_{t,k+1}$, $\nbrmcv_{t,k+1}$ and $\bar \nbrmcv_{t,k+1}$ are
equal, and equal to the number of points in the Monte Carlo sum (see
\Cref{prop:check:hyp:error}).

\subsection{Numerical illustrations}
\label{sec:numerical:appli}
Let us run \3PS for minimizing the criterion $F$; based on previous results comparing variance reduced Expectation Maximization algorithms and variance reduced Gradient algorithms (see e.g. \cite[section 4]{chen:etal:2018}), we restrict our attention to the EM
approach. In this numerical application, $n = 24 \, 989$ and $d = 21$; we choose $\tau = 1$ and $\sigma^2 = 0.05$.

{\bf The data set.} The $n$ pairs $(y_i, X_i)$ are built from the
MNIST data set. The $13 \, 007$ examples labeled $y_i=-1$ are the
examples labeled $1$ or $7$ in the MNIST training data set; the $11
\, 982$ examples labeled $y_i = 1$ are the examples labeled $3$ or
$8$ in the MNIST training data set. The covariates $X_i$ are obtained
as follows. Let $\mathsf{X}^{\mathrm{im}}$ be the $784 \times n$
matrix collecting the $784$ pixels for each image. The pixels take
values in $\ccint{0,1}$. Then the rows of $\mathsf{X}^{\mathrm{im}}$
are centered; by a PCA, each image is reduced to a vector in
$\rset^{20}$. This yields $\mathsf{X}^{\mathrm{red}} \in \rset^{20
  \times n}$. Finally, $\mathsf{X}^{\mathrm{red}}$ is augmented with a
row of ones, yielding $\mathsf{X} \in \rset^{21 \times n}$. The
columns of $\mathsf{X}$ are the $X_i$'s.

{\bf The algorithms.} We compare four algorithms. {\tt EM} denotes the
SAEM algorithm (\citet{lavielle:delyon:moulines:1999}) combined with a proximal step: each iteration
processes the full data set so that there is one iteration of {\tt EM}
per epoch:
\[
\hatS_{r+1}^{\tt EM} \eqdef \prox_{\pas \,g}^B(\hatS_r^{\tt EM} + \frac{\pas}{n} \sum_{i=1}^n \widehat{\barprecond_i(\hatS_r^{\tt EM}))} \eqsp.
\]
 {\tt Online EM} is the algorithm given by \citet{cappe:moulines:2009}
 combined with a proximal step; each iteration processes $\lbatch$
 examples and below, we will run $\kin{} \eqdef \lceil{n/\lbatch \rceil}$ iterations per epoch:
\[
\hatS_{r+1}^{\tt OEM} \eqdef \prox_{\pas \,g}^B(\hatS_r^{\tt OEM} + \frac{\pas}{ \lbatch} \sum_{i \in \batch_{r+1}} \widehat{\barprecond_i(\hatS_r^{\tt 0EM}))} \eqsp.
\]
For {\tt EM} and {\tt Online EM},
$\widehat{\barprecond_i(\hatS_t^\bullet)}$ is a Monte Carlo approximation
of $\barprecond_i(\hatS_t^\bullet)$ computed with $m^t$ points.  {\tt
  \3PS} is \Cref{algo:3PS}; we choose $\kin{t} = \kin{}$ and $\kin{}
= \lceil n/\lbatch \rceil$ so that one epoch corresponds to the
$\kin{}$ inner loops; we choose $\lbatch_t' = n$ so that the
initialization of each outer loop is one epoch; the $\delta_{t,k,i}$
are computed by Monte Carlo sums (see \eqref{eq:appli:precondMC}) with
$m^0$ points for $\delta_{t,0,i}$ and $m^t$ points for
$\delta_{t,k+1,i}$; since $\hatS_{t,0} = \hatS_{t,-1}$, we set
$\delta_{t,1,i} =0$ for all $i$, so that $\Smem_{t,1} = \Smem_{t,0} = n^{-1}
\sum_{i=1}^n \delta_{t,0,i}$.  Finally, \3PS and \3PS-{\tt corr}
differ as follows: the Monte Carlo approximation $\delta_{t,k+1,i}$
necessitates a Monte Carlo approximation of $\barprecond_i(\hatS_{t,k})$
and one of $\barprecond_i(\hatS_{t,k-1})$. In \3PS, the Monte Carlo
approximations are based on two independent chains (see
\eqref{eq:appli:precondMC}) while in \3PS-{\tt corr} the chains are
correlated. \\ All the algorithms are initialized at the null vector
$\hatS_\init = 0 \in \rset^d$. The step size is equal to $\pas = 0.4$ during
the first six epochs and then equal to $\pas = 0.1$. The length of all
the paths is $20$ epochs. On all the figures except
Figure~\ref{fig:iterate}, we report a mean value computed over $25$
independent runs of each algorithm; the shadowed area is delimited by
the minimal and maximal value of the displayed criterion over these runs.

{\bf Analyses.} Most of the comparisons are based on the evolution of
\[
\Delta_{t,k+1} \eqdef \frac{\|\prox_{\pas \,g}^B(\hatS_{t,k} + \pas \, \Smem_{t,k+1}) -
  \hatS_{t,k} \|_B^2 }{\pas^2}
\]
as a function of the number of epochs; this criterion is an
approximation of $\Delta_{t,k+1}^\star$ (see \eqref{eq:def:Deltastar})
which can not be computed here since $\barprecond$ has no closed form in
this application.  The criterion $\Delta_{t,k+1}$ for \3PS and
\3PS-{\tt corr}, is compared to $\Delta_r^{\tt EM}$ defined by
\[
\frac{\|\prox_{\pas \,g}^B(\hatS_r^{\tt EM} + \pas n^{-1} \sum_{i=1}^n \widehat{\barprecond_i(\hatS_r^{\tt EM}))} -
  \hatS_{r}^{\tt EM} \|_B^2}{\pas^2} \eqsp;
\]
and to $\Delta_r^{\tt OEM}$ defined by
\[
\frac{\|\prox_{\pas \,g}^B(\hatS_r^{\tt OEM} + \pas \lbatch^{-1}
  \sum_{i \in \batch_{r+1}} \widehat{\barprecond_i(\hatS_r^{\tt OEM}))} -
  \hatS_{r}^{\tt OEM} \|_B^2}{\pas^2} \eqsp.
\]
The best algorithm will have the smallest value of
$\Delta_{t,k+1}$.

We first study the role of some design parameters of
\3PS, such as the number of Monte Carlo points when computing
$\delta_{t,0,i}$ (denoted by $m^0$) and $\delta_{t,k+1,i}$ (denoted by
$m^t$) and the balance between $\kin{}$ and $\lbatch$ which satisfy $\kin{} \lbatch \approx n$.  \\ On
Figure~\ref{fig:nbrmc}, two strategies are chosen: first, $m^0 = m^t =
2 \lceil \sqrt{n} \rceil$; then $m^0 = m^t = 5 \lceil \sqrt{n}
\rceil$; in all cases, $\kin{} = \lceil \sqrt{n}/10 \rceil$ and
$\lbatch = \lceil n/\kin{} \rceil$.  For comparison, {\tt EM} and {\tt
  Online EM} are also run, with a number of Monte Carlo point equal to
$m^t$ at each iteration. 

\definecolor{olivegreen}{rgb}{0.47, 0.67,
      0.19}
\begin{figure} [h]
  \includegraphics[width=\columnwidth]{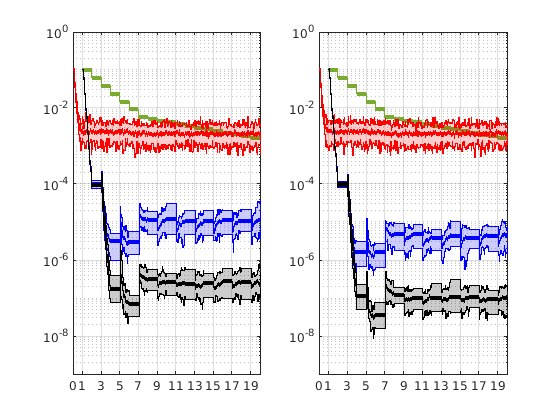}
\caption{Different strategies for the number of Monte Carlo points
  when approximating $\barprecond_i(s)$ - see
  \eqref{eq:appli:precondMC}. Evolution of \textcolor{olivegreen}{$\Delta^{{\tt EM}}_r$ in green},
  \textcolor{red}{$\Delta_r^{\tt OEM}$ in red}, 
  \textcolor{blue}{$\Delta_{t,k+1}$ for {\tt 3P-SPIDER} in blue}  and $\Delta_{t,k+1}$ for
            {\tt 3P-SPIDER corr} in black, as a function of the number
            of epochs.  [left] $m^0 = m^t = 2 \lceil \sqrt{n}
\rceil$, [right] $m^0 = m^t = 5 \lceil \sqrt{n}
\rceil$.}
\label{fig:nbrmc}
  \end{figure}

On Figure~\ref{fig:nbrinner}, the case when
$\kin{} = \lceil \sqrt{n}/10 \rceil$ is compared to the case $\kin{} =
\lceil \sqrt{n}/2 \rceil$; in both cases, $\lbatch = \lceil n/\kin{}
\rceil$ and $m^0 = m^t = 2 \lceil \sqrt{n} \rceil$.
\begin{figure} [h]
  \includegraphics[width=\columnwidth]{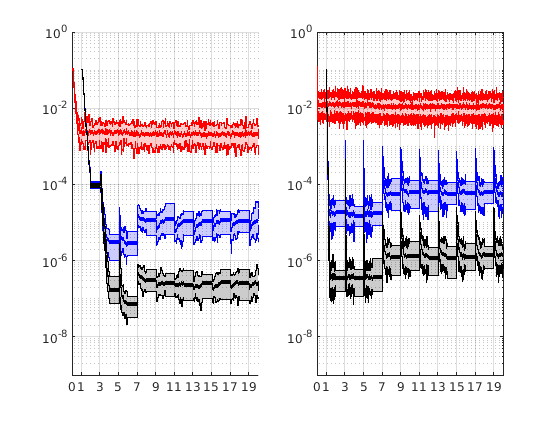}
\caption{Number of inner loops per epoch. Evolution of
  \textcolor{red}{$\Delta_r^{\tt OEM}$ in red},
  \textcolor{blue}{$\Delta_{t,k+1}$ for {\tt 3P-SPIDER} in blue} and
  $\Delta_{t,k+1}$ for {\tt 3P-SPIDER corr} in black, as a function of
  the number of epochs.  [left] $\kin{} = \lceil \sqrt{n}/10 \rceil$
  and $\lbatch = \lceil n/\kin{} \rceil$. [right] $\kin{} = \lceil
  \sqrt{n}/2 \rceil$ and $\lbatch = \lceil n/\kin{} \rceil$.}
\label{fig:nbrinner}
  \end{figure}

 Each algorithm returns a sequence of
points in the $s$-space, from which a sequence of points in the
$\param$-space is deduced through the formula $\param = \map(s) = B s
\in \rset^d$.  On Figure~\ref{fig:iterate}, three components of this
$\theta$-sequence are displayed, versus the number of epochs.  \\ 

\begin{figure} [h]
  \includegraphics[width=\columnwidth]{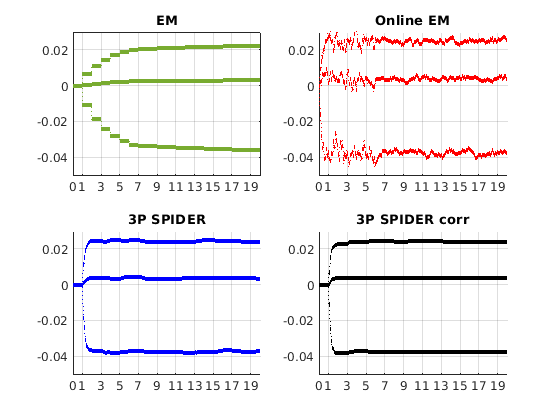}
\caption{Estimation of three parameters. Evolution of the three
  components of $\param$ by \textcolor{olivegreen}{{\tt
        EM} in green} (top, left), \textcolor{red}{{\tt
      OEM} in red} (top, right), \textcolor{blue}{
   {\tt 3P-SPIDER} in blue} (bottom, left) and 
 {\tt 3P-SPIDER corr} in black (bottom, right), as a function of
  the number of epochs.}
\label{fig:iterate}
\end{figure}
Finally, we also display on Figure~\ref{fig:normiterate} the evolution
of the squared norm of the iterates $\| \hatS_{t,k}\|^2$ obtained by
\3PS and \3PS-{\tt corr}, and $\|\hatS_r^{\tt OM}\|^2$ and
$\|\hatS_r^{\tt OEM}\|^2$ obtained resp. by {\tt EM} and {\tt Online
  EM}. They are plotted as a function of the epochs.

\begin{figure} [h]
  \includegraphics[width=\columnwidth]{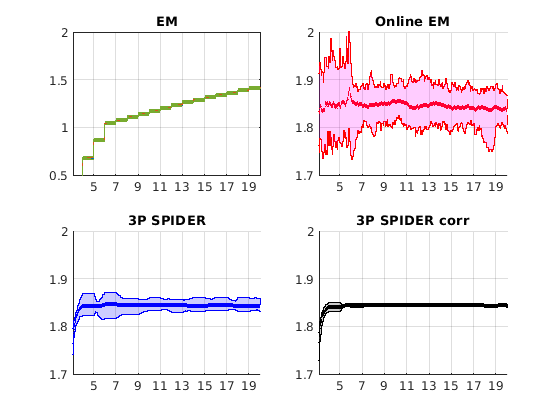}
\caption{Squared norm of the iterates. Evolution of \textcolor{olivegreen}{$\| \hatS^{{\tt EM}}_r\|^2$ in green} (top, left),
  \textcolor{red}{$\| \hatS^{{\tt OEM}}_r\|^2$ in red} (top, right), 
  \textcolor{blue}{$\| \hatS_{t,k}\|^2$ for {\tt 3P-SPIDER} in blue}  (bottom, left) and $\| \hatS_{t,k}\|^2$ for
            {\tt 3P-SPIDER corr} in black (bottom, right), as a function of the number
            of epochs. 
  }
\label{fig:normiterate}
  \end{figure}

{\bf Conclusions.} {\tt EM} has a slow convergence rate and even fails
to converge before $20$ epochs contrary to the other algorithms (see
e.g. Figure~\ref{fig:normiterate}): one update of the iterate per
epoch is not enough especially during the first iterations when more updates even based on part of the data set is a better strategy (see
e.g. the behavior of {\tt Online EM}, which contains $\kin{}$ updates
per epoch). \\ {\tt Online EM}, \3PS and \3PS-{\tt corr} process part
of the data set at each iteration; compared to {\tt Online EM}, the
\3PS's contain a variance reduction. All the plots illustrate the
benefit of this variance reduction, which reduces the variability at
convergence. \\ The choice of $\pas$ impacts this variability: see
e.g. Figures~\ref{fig:nbrmc}, \ref{fig:nbrinner} and
Figure~\ref{fig:normiterate} where a change occurs at epoch $\# 7$
(remember that from epoch $2\ell$ to $2\ell+1$, {\tt Online EM} runs
$\kin{}$ updates of the iterates while the \3PS's do not update the
iterate since they compute $\Smem_{t,0}$). \\ \3PS-{\tt corr} improves
on \3PS. The control variate has a larger impact when the correlation
is increased, as illustrated by all plots. It decreases the
variability introduced by the mini-batches ($\lbatch < n$) and the
variability introduced by the Monte Carlo approximation
$\delta_{t,k+1,i}$.  \\ Given the budget of $n$ examples processed per
outer loops, Figure~\ref{fig:nbrinner} shows that at convergence, the
accuracy is improved by larger mini batch sizes and therefore a
smaller number of inner loops. Not surprisingly, a larger number of
Monte Carlo points decreases the variability at convergence (see
Figure~\ref{fig:nbrmc}).

\section{Proof of \Cref{sec:algo:3P}}
\label{sec:proof1}
\subsection{Proof of \Cref{lem:ExistenceUniqueness:prox}} \label{sec:proof:lem:Prox}
\Cref{lem:ExistenceUniqueness:prox2} collects the two statements of
\Cref{lem:ExistenceUniqueness:prox} and a third property.
\begin{lemma} 
\label{lem:ExistenceUniqueness:prox2} Assume \Cref{hyp:g}. 
\begin{enumerate}
\item For any $\pas >0$, $B \in \Pmatrix_+^q$ and $s \in \rset^q$, the
  optimization problem  \eqref{eq:varmatric:prox} has a unique
  solution, characterized as the unique point ${\sf p} \in \Sset$
  satisfying $- \pas^{-1} \, B ({\sf p} -s) \in \partial g({\sf p})$.
\item  \label{lem:ExistenceUniqueness:prox:item2} For any $\pas >0$, $B
  \in \Pmatrix_+^q$, $s \in \Sset$ and $h \in \rset^q$,
\begin{equation} \label{eq:fixedpoint}
s = \prox_{\pas g}^{B}(s + \pas h) \quad \text{iff} \quad  B h \in  \partial g(s). 
\end{equation}
\item Let $\pas >0$ and $B \in \Pmatrix_+^q$. The operator
  $\prox_{\pas g}^B$ is firmly nonexpansive; this implies that for any
  $s,s' \in \rset^q$,
  \begin{multline*}
  \| \prox_{\pas g}^{B}(s') - \prox_{\pas g}^{B}(s) \|_B^2 \\
  \leq
 \pscal{\prox_{\pas g}^{B}(s') - \prox_{\pas g}^{B}(s)}{s'-s}_B \eqsp.
    \end{multline*}
\end{enumerate}
\end{lemma}
\begin{proof}
  Existence, uniqueness and characterization are established in
  \citet[Chapter XV, Lemma 4.1.1]{hiriarturruty:lemarechal:1996}.  The
  statement \eqref{eq:fixedpoint} follows from the characterization;
  note that $\prox_{\pas g}^B(s) \in \Sset$ for any $s \in \rset^q$.
  The firmly nonexpansive property is a consequence of \citet[Chapter
    XV, Theorem 4.1.4]{hiriarturruty:lemarechal:1996}.
\end{proof}

\section{Proof of \Cref{sec:theory}}
\label{sec:proof2}
\subsection{Notations}
\label{sec:def:filtrations}
Define for any $s \in \Sset$,
 \[
 \barprecond(s) \eqdef \frac{1}{n} \sum_{i=1}^n \barprecond_i(s) \eqsp, \qquad
 \barprecond_\batch \eqdef \frac{1}{\lbatch} \sum_{i \in \batch} \barprecond_i
 \eqsp,
 \]
 where $\batch$ is an $n$-tuple of elements of $[n]^\star$ (with or
 without multiplicity) of cardinal $\lbatch$.

All the random variables are defined on a probability space
$(\Omega, \mathcal{A}, \PP)$. It is endowed with the following
filtrations for $t \geq 0$ and $k \geq 0$,
 \begin{align*}
   \F_{0,\kin{0}} & \eqdef \sigma(\hatS_{\mathrm{init}}), \\
   \F_{t,0}
                  & \eqdef \F_{t-1, \kin{t-1}} \bigvee \sigma\left(\batch_{t,0}, \delta_{t,0,i} \ \text{for all} \ i \right),  \\
   \F_{t,k+\frac{1}{2}} & \eqdef \F_{t,k} \bigvee \sigma( \batch_{t,k+1}), \\
   \F_{t,k+1} & \eqdef \F_{t,k+\frac{1}{2}} \bigvee \sigma\left(\delta_{t,k+1,i} \ \text{for all} \ i \in \batch_{t,k+1} \right) \eqsp.
\end{align*}

 For any
 $t \in [\kout]^\star$, set
\begin{align*}
  \mathcal{E}_t & \eqdef   \Smem_{t,0} - 
                  \barprecond(\hatS_{t,0})   = 
                  \frac{1}{\lbatch'_t} \sum_{i \in
                    \batch_{t,0}} \delta_{t,0,i} -  \barprecond(\hatS_{t,0}) \eqsp.
\end{align*}
$\mathcal{E}_t$ is the error when replacing the full sum using exact
terms $\barprecond_i(\hatS_{t,0})$, with a possibly subsum of size
$\lbatch_t' <n$ using approximations of
$\barprecond_i(\hatS_{t,0})$.
Remember that
\[
\xi_{t,k+1,i} \eqdef \delta_{t,k+1,i} - \barprecond_i(\hatS_{t,k}) +
                 \barprecond_i(\hatS_{t,k-1})  \eqsp,
                 \]
                 and
                 \begin{align*}
                 \mu_{t,k+1,i} & \eqdef \PE\left[ \xi_{t,k+1,i}  \vert \F_{t,k+1/2} \right] \eqsp, \\
                 \sigma^2_{t,k+1,i} &\eqdef  \PE\left[ \| \xi_{t,k+1,i}  -  \mu_{t,k+1,i}  \|^2 \vert \F_{t,k+1/2} \right] \eqsp.
                 \end{align*}
                 Finally, set
                 \[ 
  \eta_{t,k+1}  \eqdef  \frac{1}{\lbatch} \sum_{i \in
                 \batch_{t,k+1}} \xi_{t,k+1,i}.
\]
Throughout the proof, we will use the shorthand notation
\[
B_{t,k} \eqdef \B(\hatS_{t,k}) \eqsp.
\]

\subsection{Preliminary lemmas}
\begin{lemma}
  \label{lem:batch} Let $\batch$ be a batch of $[n]^\star$ of size
  $\lbatch$, sampled at random (with or without replacement).
\begin{enumerate}
\item For any family $\{f_1, \cdots, f_n \}$, $ \PE\left[ \lbatch^{-1}
  \sum_{i \in \batch} f_i \right] = n^{-1} \sum_{i=1}^n f_i$.
\item For any family $\{f_1, \cdots,  f_n \}$,
\begin{multline*}
  \PE\left[ \Big \| \frac{1}{\lbatch} \sum_{i \in \batch} f_i -
    \frac{1}{n}\sum_{i=1}^n f_i \Big \|^2 \right] \\ \leq
  \frac{1}{\lbatch \, n} \sum_{i=1}^n \| f_i - \frac{1}{n}\sum_{j=1}^n
  f_j \|^2 \eqsp.
  \end{multline*}
\item Assume \Cref{hyp:globalL}.  For any $s,s' \in \Sset$, it holds
  \begin{multline*}
    \PE\left[ \Big \| \left\{ \barprecond_\batch(s) - \barprecond_\batch(s')
      \right\} - \left\{ \barprecond(s) - \barprecond(s')
      \right\} \Big \|^2 \right]
    \\
    \leq \frac{1}{\lbatch}\left(L^2 \|s-s'\|^2 - \|  \barprecond(s) -
       \barprecond(s')\|^2 \right) \eqsp.
\end{multline*} 
\end{enumerate}
\end{lemma}
\begin{proof}
  The proof is along the same lines as the proof of \cite[Lemma 4]{fort:etal:2020}. A detailed proof
  is provided in \Cref{sec:proof:Bronde}.
\end{proof}

\begin{lemma}
  \label{lem:vareta}
Assume  A\autoref{hyp:error}-\autoref{hyp:error:indep}  and A\autoref{hyp:error}-\autoref{hyp:error:var}.  For any $t \in
[\kout]^\star$ and $k \in [\kin{t}-1]$, it holds
\begin{align*}
   &  \PE\left[ \eta_{t,k+1} \vert \F_{t,k+1/2} \right]  =   \frac{1}{\lbatch} \sum_{i\in \batch_{t,k+1}} \mu_{t,k+1,i}  \eqsp, \\
  & \PE\left[ \eta_{t,k+1} \vert \F_{t,k} \right] = \frac{1}{n} \sum_{i=1}^n \mu_{t,k+1,i}  \eqsp, \\
&  \PE\left[ \| \eta_{t,k+1} - \PE\left[ \eta_{t,k+1} \vert \F_{t,k} \right] \|^2 \vert \F_{t,k} \right]  \\
  & \qquad \leq \frac{1}{\lbatch} \left( \frac{C_v}{\nbrmcv_{t,k+1}} + \frac{C_{vb}^2}{\bar \nbrmcv_{t,k+1}^2}  \right) \eqsp.
\end{align*}
\end{lemma}
\begin{proof}  Let $t \in [\kout]^\star$ and $k \in [\kin{t}-1]$.
  We have
  \[
  \PE\left[ \eta_{t,k+1} \vert \F_{t,k+1/2} \right]  = \frac{1}{\lbatch} \sum_{i\in \batch_{t,k+1}} \mu_{t,k+1,i} \eqsp,
  \]
  since $\batch_{t,k+1} \in \F_{t,k+1/2}$; and by \Cref{lem:batch},
  \[
\PE\left[ \eta_{t,k+1} \vert \F_{t,k} \right] = \frac{1}{n} \sum_{i=1}^n \mu_{t,k+1,i} \eqsp.
  \]
 We write
  \begin{align*}
    \eta_{t,k+1} & - \PE\left[ \eta_{t,k+1} \vert \F_{t,k} \right]  \\
    & = \frac{1}{\lbatch} \sum_{i \in \batch_{t,k+1}} \xi_{t,k+1,i} - \frac{1}{n} \sum_{i=1}^n \mu_{t,k+1,i} \\
  & =  \frac{1}{\lbatch} \sum_{i \in \batch_{t,k+1}} \left\{ \xi_{t,k+1,i} - \mu_{t,k+1,i}  \right\}  \\
  & + \frac{1}{\lbatch} \sum_{i \in \batch_{t,k+1}} \mu_{t,k+1,i} -  \frac{1}{n} \sum_{i=1}^n \mu_{t,k+1,i}  \eqsp.
  \end{align*}
  The RHS is of the form $U+V$ and we write $\| U+V \|^2= \|U\|^2 +
  \|V \|^2 + 2 \pscal{U}{V}$ with $U \leftarrow \lbatch^{-1} \sum_{i
    \in \batch_{t,k+1}} \left\{ \xi_{t,k+1,i} - \mu_{t,k+1,i}
  \right\}$. By conditioning and by definition of
  $\sigma^2_{t,k+1,i}$, we have 
  \begin{align*}
    \PE\left[ \|U \|^2 \vert \F_{t,k} \right]    & =  \frac{1}{\lbatch^2}  \PE\left[  \sum_{i \in \batch_{t,k+1}}  \sigma^2_{t,k+1,i} \vert \F_{t,k} \right] \eqsp.
 \end{align*}
    Under  A\autoref{hyp:error}-\autoref{hyp:error:indep}, we have by \Cref{lem:batch}
 \begin{align*}
    \PE\left[ \|U \|^2 \vert \F_{t,k} \right]   
    & = \frac{1}{\lbatch \, n} \sum_{i=1}^n \sigma^2_{t,k+1,i}  \leq \frac{C_v}{\lbatch \, \nbrmcv_{t,k+1}}  \eqsp.
  \end{align*}
  By \Cref{lem:batch} again, it holds
  \[
 \PE\left[ \|V \|^2 \vert \F_{t,k} \right] \leq \frac{1}{\lbatch \, n} \sum_{i=1}^n \| \mu_{t,k+1,i} - \frac{1}{n} \sum_{j=1}^n \mu_{t,k+1,j} \|^2 \eqsp,
 \]
 which yields
 \[
 \PE\left[ \|V \|^2 \vert \F_{t,k} \right] \leq \frac{C_{vb}^2}{\lbatch \, \bar \nbrmcv_{t,k+1}^2} \eqsp.
 \]
 Finally, upon noting that $\PE\left[ U \vert \F_{t,k+1/2} \right] =0$ and $V \in  \F_{t,k+1/2}$, we have
 \[
\PE\left[ \pscal{U}{V} \vert \F_{t,k} \right] = \PE\left[ \pscal{\PE\left[U \vert \F_{t,k+1/2} \right]}{V} \vert \F_{t,k} \right] = 0 \eqsp.
\]
This concludes the proof.
  \end{proof}
\subsection{Results on the variables $\Smem_{t,k}$}
\Cref{prop:bias} studies the bias of the variables $\Smem_{t,k+1}$. It
shows that $\Smem_{t,k+1}$ is a {\em biased} approximation of
$\barprecond(\hatS_{t,k})$:
\[
\PE\left[ \Smem_{t,k+1} \vert \F_{t,k} \right] \neq
\barprecond(\hatS_{t,k}).
\]
When $k=0$, we may have
$\PE\left[ \Smem_{t,1} \vert \F_{t,0} \right] = \barprecond(\hatS_{t,0})$
if $\delta_{t,0,i} =\precond_i(\hatS_{t,0})$ and $\batch_{t,0} = [n]^\star$. The choice
$\batch_{t,0} = [n]^\star$ is the strategy proposed in \cite{wang:etal:2019} for {\tt SpiderBoost}; it has an important
computational cost but has the advantage to cancel the bias of the
variable $\Smem_{\cdot}$ at the beginning of each outer loop. Along the inner loops, a (signed) bias appears.

\begin{proposition}
  \label{prop:bias}
  For any $t \in [\kout]^\star$ and $k \in [\kin{t}-1]$, it holds
  \begin{multline*}
\PE\left[ \Smem_{t,k+1} \vert \F_{t,k} \right] - \barprecond(\hatS_{t,k})
\\
= \Smem_{t,k} - \barprecond(\hatS_{t,k-1}) + \PE\left[\error_{t,k+1} \vert
  \F_{t,k} \right] \eqsp,
\end{multline*}
and
\[
\PE\left[ \Smem_{t,k+1} -  \barprecond(\hatS_{t,k}) \vert \F_{t,0} \right]
= \mathcal{E}_t + \sum_{j=1}^{k+1} \PE\left[\error_{t,j} \vert \F_{t,0}
  \right]\eqsp.
\]
  \end{proposition}
\begin{proof}
  Let $t \in [\kout]^\star$ and $k \in [\kin{t}-1]$. We write $\Smem_{t,k+1} = \Smem_{t,k} +  h_{\batch_{t,k+1}}(\hatS_{t,k}) - h_{\batch_{t,k+1}}(\hatS_{t,k-1}) + \eta_{t,k+1}$. By
  \Cref{lem:batch},
  \begin{multline*}
    \PE\left[ \Smem_{t,k+1} \vert \F_{t,k} \right] = \Smem_{t,k} +
    \barprecond(\hatS_{t,k}) -  \barprecond(\hatS_{t,k-1}) \\ +
    \PE\left[\error_{t,k+1} \vert \F_{t,k} \right]\eqsp.
  \end{multline*}
  Since $\F_{t,0} \subseteq \F_{t,k}$, we have
   \begin{multline*}
    \PE\left[ \Smem_{t,k+1} - \barprecond(\hatS_{t,k}) \vert \F_{t,0}
      \right] = \PE\left[\Smem_{t,k} - \barprecond(\hatS_{t,k-1}) \vert
      \F_{t,0} \right] \\ + \PE\left[\error_{t,k+1} \vert \F_{t,0}
      \right]\eqsp.
  \end{multline*}
 Summing from $j=0$ to $j=k$ yields
  \begin{multline*}
    \PE\left[ \Smem_{t,k+1} - \barprecond(\hatS_{t,k}) \vert \F_{t,0} \right] =
    \Smem_{t,0} - \barprecond(\hatS_{t,-1}) \\ + \sum_{j=0}^{k}\PE\left[\error_{t,j+1} \vert
      \F_{t,0} \right]\eqsp.
  \end{multline*}
  The proof is concluded by using $\hatS_{t,0} = \hatS_{t,-1}$ and the
  definition of $\mathcal{E}_t$; note that $\mathcal{E}_t \in
  \F_{t,0}$.
\end{proof}

\Cref{prop:variance} provides a control of the conditional variance  of $\Smem_{t,k}$.
\begin{proposition}
  \label{prop:variance}
  Assume \Cref{hyp:globalL}, \Cref{hyp:error}-\autoref{hyp:error:indep} and
  \Cref{hyp:error}-\autoref{hyp:error:var}.  For any $t \in
       [\kout]^\star$ and $k \in [\kin{t}-1]$, it holds
  \begin{align*}
 & \PE\left[ \Big \|\Smem_{t,k+1} - \PE\left[ \Smem_{t,k+1} \vert
        \F_{t,k} \right]\Big \|^2 \vert \F_{t,k} \right] \\ & \leq
  \frac{L^2}{\lbatch}  \left( 1 +  \frac{2 
      C_{vb}}{\sqrt{\lbatch} \,  \bar \nbrmcv_{t,k+1}} \right) \|\hatS_{t,k} -
    \hatS_{t,k-1}\|^2\\ & + \frac{C_v}{\lbatch \, \nbrmcv_{t,k+1}} + \frac{C_{vb}^2}{\lbatch \, \bar \nbrmcv_{t,k+1}^2} + \frac{2 C_{vb}}{\sqrt{\lbatch} \,  \bar \nbrmcv_{t,k+1}} \eqsp.
  \end{align*}
  \end{proposition}
\begin{proof}
  Let $t \in [\kout]^\star$, $k \in [\kin{t}-1]$.  By
  \Cref{lem:batch}, \Cref{prop:bias}, the definitions of
  $\Smem_{t,k+1}$ and of the filtration $\F_{t,k}$,
  \begin{align*}
   &  \Smem_{t,k+1} - \PE\left[ \Smem_{t,k+1} \vert \F_{t,k} \right] \\
& = \Smem_{t,k+1} - \barprecond(\hatS_{t,k}) - \Smem_{t,k} + 
     \barprecond(\hatS_{t,k-1}) - \PE\left[ \error_{t,k+1} \vert
       \F_{t,k}\right] \\ & = \error_{t,k+1} - \PE\left[
       \error_{t,k+1} \vert \F_{t,k}\right] \\ & +  \barprecond_{\batch_{t,k+1}}(\hatS_{t,k}) -
     \barprecond_{\batch_{t,k+1}}(\hatS_{t,k-1})  -  \barprecond(\hatS_{t,k}) +
     \barprecond(\hatS_{t,k-1}) \eqsp.
  \end{align*}
  The RHS is of the form $U+V$ with $U \leftarrow \error_{t,k+1} -
  \PE\left[ \error_{t,k+1} \vert \F_{t,k}\right]$ and $V \in
  \F_{t,k+1/2}$.  Then, we write $\PE\left[\|U+V\|^2 \vert \F_{t,k}
    \right] = \PE\left[\|U\|^2 \vert \F_{t,k} \right] +
  \PE\left[\|V\|^2 \vert \F_{t,k} \right] + 2 \PE\left[\pscal{
      \PE\left[U \vert \F_{t,k+1/2} \right]}{V} \vert \F_{t,k}
    \right]$.

  The term $\PE\left[\|V\|^2 \vert \F_{t,k} \right]$ is
  controlled by Lemma~\ref{lem:batch}: an upper bound is $L^2
  \lbatch^{-1} \| \hatS_{t,k} - \hatS_{t,k-1}\|^2$. The term
  $\PE\left[\|U\|^2 \vert \F_{t,k} \right]$ is controlled by \Cref{lem:vareta}: an upper bound is $C_v / (\lbatch \, \nbrmcv_{t,k+1}) + C_{vb}^2 / (\lbatch \, \bar \nbrmcv_{t,k+1}^2)$. 
Upon noting that $V \in \F_{t,k+1/2}$,  and using \Cref{lem:vareta} and \Cref{lem:batch}, the scalar product is
  upper bounded by
  \begin{multline*}
    2 \, \PE\left[\|V\| \, \Big \| \PE\left[U \vert \F_{t,k+1/2}
        \right] \Big \| \Big \vert \F_{t,k} \right] \\ \leq 2
    \frac{C_{vb} }{\sqrt{\lbatch} \, \bar \nbrmcv_{t,k+1}} \, \left\{\PE\left[ \|V\|^2 \Big \vert
      \F_{t,k} \right] \right\}^{1/2} \\ \leq 2 \frac{C_{vb} }{\sqrt{\lbatch} \, \bar \nbrmcv_{t,k+1}} \,
    \left(1+ \PE\left[ \|V\|^2 \vert \F_{t,k} \right]\right) 
    \eqsp,
  \end{multline*}
  where we used that $|a| \leq 1 + |a|^2$.
\end{proof}

\Cref{prop:squarederror} establishes an upper bound on the conditional
expectation of the quadratic error $\|\Smem_{t,k+1} -
\barprecond(\hatS_{t,k})\|^2$.
\begin{proposition}
  \label{prop:squarederror}
 Assume \Cref{hyp:globalL} and \Cref{hyp:error}. For any $t \in
 [\kout]^\star$ and $k \in [\kin{t}-1]$, it holds
  \begin{align*}
 & \PE\left[ \|\Smem_{t,k+1} - \barprecond(\hatS_{t,k}) \|^2 \vert
      \F_{t,k} \right] \\ & \leq \left(1 +  \frac{2
      C_b}{\nbrmc_{t,k+1}}\right) \|\Smem_{t,k} -
    \barprecond(\hatS_{t,k-1})\|^2  \\ & +  \frac{L^2}{\lbatch}  \left(1 + \frac{2 C_{vb}}{\sqrt{\lbatch} \,  \bar \nbrmcv_{t,k+1}} \right)  \|\hatS_{t,k} -
    \hatS_{t,k-1}\|^2 \\ & +
    \mathcal{U}_{t, k+1} \eqsp,
  \end{align*}
where
\[
\mathcal{U}_{t, k} \eqdef \frac{2 C_b}{\nbrmc_{t,k}} +
\frac{C_b^2}{\nbrmc_{t,k}^2} + \frac{C_v}{\lbatch \,\nbrmcv_{t,k}}+ \frac{2 C_{vb}}{\sqrt{\lbatch} \, \bar \nbrmcv_{t,k}} +\frac{C_{vb}^2}{\lbatch \, \bar \nbrmcv_{t,k}^2}  \eqsp.
\]
\end{proposition}
\begin{proof}
  Let $t \in [\kout]^\star$ and $k \in [\kin{t}-1]$.  By definition of
  the conditional expectation, we have for any r.v. $U,V$ and any
  $\sigma$-field $\F$ such that $V \in \F$:
  \begin{multline*}
  \PE\left[ \| U - V \|^2 \vert \F \right] \\ = \PE\left[ \| U -
    \PE\left[U \vert \F \right] \|^2 \vert \F \right] + \| \PE\left[U
    \vert \F \right] - V \|^2 \eqsp.
  \end{multline*}
  We apply this equality with $U \leftarrow \hatS_{t,k+1}$, $V
  \leftarrow \barprecond(\hatS_{t,k})$ and $\F \leftarrow \F_{t,k}$.
  Proposition~\ref{prop:variance} controls the first term. For the
  second one, by Proposition~\ref{prop:bias}, \Cref{lem:vareta} and
  A\autoref{hyp:error}-\autoref{hyp:error:bias} we have
  \begin{align*}
    & \| \PE\left[\Smem_{t,k+1} \vert \F_{t,k} \right] - \barprecond(\hatS_{t,k}) \|^2 \\
& =  \| \Smem_{t,k} - \barprecond(\hatS_{t,k-1}) + \PE\left[\error_{t,k+1}
    \vert \F_{t,k} \right] \|^2 \\ & \leq \| \Smem_{t,k} -
  \barprecond(\hatS_{t,k-1}) \|^2 + \frac{C_b^2}{\nbrmc^2_{t,k+1}} \\ &+2
  \frac{C_b}{\nbrmc_{t,k+1}} \| \Smem_{t,k} - \barprecond(\hatS_{t,k-1})
  \| \eqsp.
\end{align*}
We conclude by using $\|a\| \leq 1 + \|a\|^2$ with $a \leftarrow \|
\Smem_{t,k} - \barprecond(\hatS_{t,k-1}) \|$.
\end{proof}

\begin{corollary}[of \Cref{prop:squarederror}]
  \label{coro:squarederror:second} Assume also \Cref{hyp:W}-\autoref{hyp:W:item3}.
  For $t \in [\kout]^\star$ and $k \in [\kin{t}-1]$, define
 $\mathcal{D}_{t,k+1}$ by
  \[
   \| \hatS_{t,k+1} - \prox_{\pas_{t,k+1} g}^{B_{t,k}}(\hatS_{t,k}+
   \pas_{t,k+1} \barprecond(\hatS_{t,k}))\|_{B_{t,k}}^2 \eqsp,
   \]
   and $\mathcal{D}_{t,0} \eqdef 0$. For $t \in [\kout]^\star$ and $k
   \in [\kin{t}-]$, it holds
 \begin{align*}
 & \PE\left[ \|\Smem_{t,k+1} - \barprecond(\hatS_{t,k}) \|^2 \vert
     \F_{t,k} \right] \\ & \quad \leq \left( 1 +  \frac{2
     C_b}{\nbrmc_{t,k+1}} \right) \ \|\Smem_{t,k} -
   \barprecond(\hatS_{t,k-1})\|^2 \\ & \quad + \pas_{t,k}^2
   \frac{2}{v_{\min}} \frac{L^2}{\lbatch} \left(1+  \frac{2 C_{vb}}{\sqrt{\lbatch} \,\bar \nbrmcv_{t,k+1}} \right)
   \Delta_{t,k}^\star \\ & \quad + \frac{2}{v_{\min}}
   \frac{L^2}{\lbatch}  \left(1+ \frac{2
     C_{vb}}{\sqrt{\lbatch} \, \bar \nbrmcv_{t,k+1}} \right) \mathcal{D}_{t,k} \\ & \quad +
   \mathcal{U}_{t, k+1} \eqsp.
 \end{align*}
 By convention, $\Delta_{t,0}^\star \eqdef 0$.
  \end{corollary}
\begin{proof}
  The proof consists in an upper bound for $\| \hatS_{t,k} -
  \hatS_{t,k-1}\|^2$.  Let $s \in \Sset$, $H,h \in \rset^q$, $\pas >0$
  and $B$ be a $q \times q$ positive definite matrix. For any $\beta
  >0$, it holds
\begin{multline*}
\! \! \| \prox_{\pas g}^B(s +\pas H) - s\|^2_B \leq (1+\frac{1}{\beta}) \, \|
\prox_{\pas g}^B(s +\pas h) - s\|^2_B \\ + (1+\beta) \, \| \prox_{\pas
  g}^B(s +\pas H) - \prox_{\pas g}^B(s+ \pas h) \|^2_B \eqsp.
\end{multline*}
We apply these inequalities with $\pas \leftarrow \pas_{t,k}$, $B
\leftarrow B_{t,k-1}$, $s \leftarrow \hatS_{t,k-1}$, $H \leftarrow
\Smem_{t,k}$ and $h \leftarrow \barprecond(\hatS_{t,k-1})$.   Then, for any $k >0$,
\begin{multline} \label{eq:auxcoro}
  \| \hatS_{t,k} - \hatS_{t,k-1}\|^2_{B_{t,k-1}} \\
  \leq (1+\beta^{-1}) \,
\pas_{t,k}^2 \Delta_{t,k}^\star  + (1+\beta) \mathcal{D}_{t,k}
\eqsp.
\end{multline}
We choose $\beta=1$ and conclude by \Cref{hyp:W}-\autoref{hyp:W:item3}: $\|\cdot\|^2 \leq v_{\min}^{-1}
\|\cdot\|^2_{B_{t,k-1}}$.

When $k=0$, $\| \hatS_{t,k} - \hatS_{t,k-1}\|^2_{B_{t,k-1}} = 0$ since
by definition, $\hatS_{t,0} = \hatS_{t,-1}$.  Therefore,
\eqref{eq:auxcoro} remains valid since $\mathcal{D}_{t,0} =0$ and
$\Delta_{t,0}^\star = 0$ by convention. This concludes the proof.
 \end{proof}

\begin{corollary}[of \Cref{coro:squarederror:second}]
  \label{coro:squarederror}
 Let $\{\rho_{t,k}, t \geq 1, k \geq 0 \}$ be a positive sequence
 satisfying
  \begin{equation}
    \label{eq:choice:beta}
\rho_{t,k+1} \ \left( 1 + \frac{2 C_b}{\nbrmc_{t,k+1}} \right) \leq
\rho_{t,k} \eqsp.
  \end{equation}
 For any $t \in [\kout]^\star$, $k \in [\kin{t}-1]$, it holds
  \begin{align*}
& \rho_{t,k+1} \PE\left[\| \Smem_{t,k+1} - \barprecond(\hatS_{t,k})\|^2
      \vert \F_{t,0} \right] \\ & \leq \rho_{t,0} \| \mathcal{E}_t
    \|^2 + \sum_{\ell=1}^{k+1} \rho_{t,\ell} \, \mathcal{U}_{t, \ell}
    + \frac{2}{v_{\min}} \frac{L^2}{\lbatch} \cdots \\ & \times
    \left\{ \sum_{\ell=1}^k \pas^2_{t,\ell} \rho_{t,\ell+1} \left(1 +
    \frac{2 C_{vb}}{\sqrt{\lbatch} \, \bar \nbrmcv_{t,\ell+1}} \right) \PE\left[
      \Delta_{t,\ell}^\star \vert \F_{t,0} \right] \right. \\ &
    \left. + \sum_{\ell=1}^k \rho_{t,\ell+1} \left(1 + \frac{2
      C_{vb}}{\sqrt{\lbatch} \, \bar \nbrmcv_{t,\ell+1}} \right) \PE\left[
      \mathcal{D}_{t,\ell} \vert \F_{t,0} \right] \right\} \eqsp.
\end{align*}
  \end{corollary}
\begin{proof}
  In \Cref{coro:squarederror:second}, the claim is of the form
  \begin{multline*}
\PE\left[\| \Smem_{t,k+1} - \barprecond(\hatS_{t,k})\|^2\vert \F_{t,k}
  \right] \\ \leq \left( 1 + \frac{2 C_b}{\nbrmc_{t,k+1}} \right)  \, \| \Smem_{t,k} -
\barprecond(\hatS_{t,k-1})\|^2 + A_k \eqsp.
\end{multline*}
This yields, by using the condition \eqref{eq:choice:beta},
 \begin{multline*}
  \rho_{t,k+1} \PE\left[\| \Smem_{t,k+1} - \barprecond(\hatS_{t,k})\|^2
    \vert \F_{t,k} \right] \\ \leq \rho_{t,k} \, \| \Smem_{t,k} -
  \barprecond(\hatS_{t,k-1})\|^2 + \rho_{t,k+1} \, A_k \eqsp.
\end{multline*}
Using $\PE[U \vert \F_{t,0}] = \PE\left[\PE[U \vert \F_{t,k}] \vert
  \F_{t,0}\right]$ and summing from $\ell = 0$ to $\ell = k$ yields
\begin{multline*}
 \rho_{t,k+1} \, \PE\left[\| \Smem_{t,k+1} - \barprecond(\hatS_{t,k})\|^2
   \vert \F_{t,0} \right] \\ \leq \sum_{\ell=0}^k \rho_{t,\ell+1}
 \PE\left[ A_\ell \vert \F_{t,0} \right] + \rho_{t,0} \| \Smem_{t,0} -
 \barprecond(\hatS_{t,-1})\|^2 \eqsp;
\end{multline*}
we then conclude by using the equality $\hatS_{t,-1} = \hatS_{t,0}$
and the definition of $\mathcal{E}_t$. Note also that
$\Delta_{t,0}^\star = 0$ and $\mathcal{D}_{t,0} = 0$.
  \end{proof}

\subsection{Lyapunov inequalities for $\lyap$, $g$ and $\lyap+g$}
\label{sec:proof:Lyapunov}
\Cref{lem:contraction:W}, while being classical in smooth
optimization, is provided for a self-content purpose.
\begin{lemma}
  \label{lem:contraction:W}
  Assume \Cref{hyp:W}.  For any $s,s' \in \Sset$ and $\pas >0$,
\[
  \lyap(s') \leq \lyap(s) - \pscal{\barprecond(s)}{s'-s}_{B(s)} +
  \frac{L_{\dot \lyap}}{2} \|s'-s\|^2 \eqsp.
  \]
  \end{lemma}
\begin{proof}
  $\Sset$ is convex since it is the domain of a convex function.  By
  \Cref{hyp:W}, $\lyap$ is continuously differentiable on $\Sset$ with
  $L_{\dot \lyap}$-Lipschitz gradient. Then for any $s,s' \in \Sset$,
\[
\lyap(s') - \lyap(s) \leq \pscal{\nabla \lyap(s)}{s'-s} +
\frac{L_{\dot \lyap}}{2} \|s'-s\|^2 \eqsp.
\]
We use that $\nabla \lyap(s) = -B(s) \barprecond(s)$, so that 
\[
\pscal{\nabla \lyap(s)}{s'-s}  =  - \pscal{\barprecond(s)}{s'-s}_{B(s)} \eqsp.
\]
\end{proof}
\begin{lemma}
  \label{lem:contraction:prox}
  Assume \Cref{hyp:g}.  Let $B$ be a $q \times q$ positive definite
  matrix. For any $s \in \Sset$, $\pas >0$, $H,h \in \rset^q$ and
  $\beta >0$,
\begin{align*}
  & g\left( \prox_{\pas g}^B(s+ \pas H) \right) \leq g(s) \\ & \quad -
  \frac{1}{\pas} \left(1 - \frac{\beta}{4}\right) \, \| \prox_{\pas
    g}^B(s+ \pas h) -s \|_B^2\\ & \quad - \frac{1}{\pas}(1-
  \frac{1}{\beta}) \| \prox_{\pas g}^B(s+ \pas h) - \prox_{\pas
    g}^B(s+ \pas H)\|^2_B \\ &\quad - \pscal{h}{s - \prox_{\pas
      g}^B(s+ \pas h)}_B \\ & \quad + \pscal{ H}{ \prox_{\pas g}^B(s+
    \pas H) - \prox_{\pas g}^B(s+ \pas h)}_B \eqsp.
\end{align*}
\end{lemma}
\begin{proof}
  In this proof, we use the shorthand notation $\mathsf{p}_H \eqdef
  \prox_{\pas g}^B(s +\pas H)$ and $\mathsf{p}_h \eqdef \prox_{\pas
    g}^B(s +\pas h)$. By \Cref{lem:ExistenceUniqueness:prox} and the
  definition of the subdifferential at a point, it holds
\begin{align*}
& g(\mathsf{p}_h) \geq g(\mathsf{p}_H) - \pas^{-1} \pscal{\mathsf{p}_H
    - s - \pas H}{\mathsf{p}_h-\mathsf{p}_H}_B \\ & g(s) \geq
  g(\mathsf{p}_h) - \pas^{-1} \pscal{\mathsf{p}_h - s - \pas
    h}{s-\mathsf{p}_h}_B \eqsp.
\end{align*}
This yields
\begin{multline*}
 g(\mathsf{p}_H) \leq g(s) - \pas^{-1} \|\mathsf{p}_h - s \|_B^2 -
 \pscal{h}{s - \mathsf{p}_h}_B \\ - \pscal{H}{\mathsf{p}_h -
   \mathsf{p}_H}_B + \pas^{-1} \pscal{\mathsf{p}_H - s}{\mathsf{p}_h -
   \mathsf{p}_H}_B \eqsp.
\end{multline*}
For the last term, we write for any $\beta >0$,
\begin{align*}
  & \pas^{-1} \pscal{\mathsf{p}_H - s }{\mathsf{p}_h - \mathsf{p}_H}_B
  + \pas^{-1} \| \mathsf{p}_h - \mathsf{p}_H\|_B^2 \\ & = \pas^{-1}
  \pscal{\mathsf{p}_h - s }{\mathsf{p}_h - \mathsf{p}_H}_B \\ & \leq 2
  \pscal{(\mathsf{p}_h - s) \frac{\sqrt{\beta}}{2\sqrt{\pas}} }{(\mathsf{p}_h -
    \mathsf{p}_H)\frac{1}{\sqrt{\beta \pas}}}_B \\ & \leq \frac{\beta}{4 \pas}
  \| \mathsf{p}_h - s \|^2_B + \frac{1}{\beta \pas} \|\mathsf{p}_h
  - \mathsf{p}_H\|_B^2 \eqsp.
\end{align*}
This concludes the proof.
\end{proof}

\begin{proposition}
  \label{prop:lyap:DL}
  Assume \Cref{hyp:g}, \Cref{hyp:globalL} and \Cref{hyp:W}. For any $t \in [\kout]^\star$,
  $k \in [\kin{t}-1]$ and $\beta >0$,
  \begin{align*}
    & \PE\left[ \lyap(\hatS_{t,k+1}) + g(\hatS_{t,k+1}) \vert \F_{t,0}
      \right] \\ & \leq \PE\left[ \lyap(\hatS_{t,k}) + g(\hatS_{t,k})
      \vert \F_{t,0} \right] \\ & - \pas_{t,k+1} \left( 1 -
    \frac{\beta}{4}  - \frac{L_{\dot \lyap} \pas_{t,k+1}
    }{ v_{\min}} \right) \PE\left[ \Delta_{t,k+1 }^\star \vert
      \F_{t,0} \right]  \\&
    - \frac{1}{\pas_{t,k+1}} \left( 1 - \frac{1}{\beta}
    - \frac{L_{\dot \lyap}}{v_{\min}} \pas_{t,k+1}\right)
   \PE\left[ \mathcal{D}_{t,k+1}  \vert
      \F_{t,0} \right]  \\ & + \pas_{t,k+1} \PE\left[ \| \Smem_{t,k+1}
      - \barprecond(\hatS_{t,k}) \|^2_{B_{t,k}}\vert \F_{t,0} \right] \eqsp,
  \end{align*}
  where $ \mathcal{D}_{t,k+1} $ is defined in
  \Cref{coro:squarederror:second}.
\end{proposition}
\begin{proof}
  Let $\pas>0$, $s \in \Sset$ and $H \in \rset^q$.  Apply
  \Cref{lem:contraction:W} with $s' \leftarrow \prox_{\pas g}^{B(s)}(s
  + \pas H) \in \Sset$; and \Cref{lem:contraction:prox} with $h \leftarrow \barprecond(s)$. This yields for any
  $\beta >0$,
\begin{multline*}
  \lyap(\prox_{\pas g}^{B}(s + \pas H)) + g(\prox_{\pas g}^{B}(s +
  \pas H)) \\ \leq \lyap(s) + g(s) - \frac{1}{\pas} (1-
  \frac{\beta}{4}) \| \prox_{\pas g}^B(s+ \pas \barprecond(s)) -s
  \|_{B}^2 \\ - \frac{1}{\pas} \left( 1 - \frac{1}{\beta
   }\right) \| \prox_{\pas g}^B(s+ \pas H) -\prox_{\pas g}^B(s+
  \pas \barprecond(s)) \|_B^2 \\ - \pscal{ \barprecond(s) -H }{\prox_{\pas g}^{B}(s + \pas
    H)- \prox_{\pas g}^{B}(s + \pas \barprecond(s))}_{B} \\ +
  \frac{L_{\dot \lyap}}{2} \|\prox_{\pas g}^{B}(s + \pas H)-s\|^2 \eqsp.
\end{multline*}
Since $\prox_{\pas g}^B$ is firmly nonexpansive (see \Cref{lem:ExistenceUniqueness:prox2}), the
scalar product is upper bounded by $\pas \|H - \barprecond(s)\|_{B}^2$.
By \Cref{hyp:W}-\autoref{hyp:W:item3}, we write
\[ \|\prox_{\pas g}^{B}(s + \pas H)-s\|^2 \leq
\frac{1}{v_{\min}} \|\prox_{\pas g}^{B}(s + \pas H)-s\|^2_{B} \eqsp;
\]
then we use $\|a + b\|^2_B \leq 2 \|a\|^2_B + 2 \|b\|^2_B$ with $a
\leftarrow \prox_{\pas g}^{B}(s + \pas H) - \prox_{\pas g}^{B}(s
+ \pas \barprecond(s))$. This yields
\begin{align*}
 & \frac{L_{\dot \lyap}}{2} \|\prox_{\pas g}^{B}(s + \pas H)-s\|^2 \\
  & \quad \leq \frac{L_{\dot \lyap}}{v_{\min}} \|
\prox_{\pas g}^{B}(s + \pas H) - \prox_{\pas g}^{B}(s + \pas
\barprecond(s))\|^2_B \\ &  \quad + \frac{L_{\dot \lyap}}{v_{\min}} \|\prox_{\pas g}^{B}(s +
\pas \barprecond(s))-s\|^2_{B} \eqsp.
\end{align*}
We apply these inequalities with $s \leftarrow \hatS_{t,k}$, $\pas
\leftarrow \pas_{t,k+1}$, $H \leftarrow \Smem_{t,k+1}$, $s' \leftarrow
\hatS_{t,k+1}$ and $B \leftarrow B_{t,k}$. Note that $\prox_{\pas
  g}^B(s + \pas H) = \hatS_{t,k+1}$. The proof is concluded.
\end{proof}

\subsection{Proof of \Cref{theo:main:general}}
\label{sec:proof:theogeneral}
  Let $t \in [\kout]^\star$. Let $\mu \in \ooint{0,1}$. Throughout the proof, set
  \[
A_{t,k+1} \eqdef \left(1 + \frac{2 C_{vb}}{\sqrt{\lbatch} \, \bar \nbrmcv_{t,k+1}} \right) \eqsp.
  \]
  From \Cref{coro:squarederror} applied with $C_{t,k+1} \leftarrow
  \pas_{t,k+1}$ and \Cref{prop:lyap:DL} applied with $ \beta
  \leftarrow 4 \mu$, it holds for any $k \in [\kin{t}-1]$,
 \begin{align*}
& \PE\left[ \lyap(\hatS_{t,k+1}) + g(\hatS_{t,k+1}) \vert \F_{t,0}
     \right] \\ & \leq \PE\left[ \lyap(\hatS_{t,k}) + g(\hatS_{t,k})
     \vert \F_{t,0} \right] \\ & - \pas_{t,k+1} \left(1-\mu -
   \frac{L_{\dot \lyap}}{v_{\min}} \pas_{t,k+1} \right) \PE\left[
     \Delta_{t,k+1}^\star \vert \F_{t,0} \right] \\ & -
   \frac{1}{\pas_{t,k+1}} \left(1 - \frac{1}{4 \mu} - \frac{L_{\dot
       \lyap}}{v_{\min}} \pas_{t,k+1} \right) \PE\left[
     \mathcal{D}_{t,k+1} \vert \F_{t,0} \right] \\ & + \pas_{t,0}
   v_\max \| \mathcal{E}_t \|^2 + v_\max \sum_{\ell=1}^{k+1}
   \pas_{t,\ell} \, \mathcal{U}_{t, \ell}\\ & +
   \frac{2v_\max}{v_{\min}} \frac{L^2}{\lbatch}\sum_{\ell=1}^k \pas^3_{t,\ell} A_{t,\ell+1}
   \PE\left[ \Delta_{t,\ell}^\star \vert \F_{t,0} \right] \\ & +
   \frac{2 v_\max}{v_{\min}} \frac{L^2}{\lbatch} \sum_{\ell=1}^k \pas_{t,\ell+1}
   A_{t,\ell+1} \PE\left[ \mathcal{D}_{t,\ell} \vert \F_{t,0} \right]
   \eqsp.
 \end{align*}
 Above, we used that $\pas_{t,k+1} \leq \pas_{t,\ell}$ for any $\ell
 \in [k+1]$.  We now sum from $k=0$ to $k = \kin{t}-1$. This yields,
 \begin{align*}
   & \PE\left[ \lyap(\hatS_{t, \kin{t}}) + g(\hatS_{t,\kin{t}}) \vert
     \F_{t,0} \right] \\ & \qquad \leq \PE\left[ \lyap(\hatS_{t,0}) +
     g(\hatS_{t,0}) \vert \F_{t,0} \right] \\ & - \sum_{k=1}^{\kin{t}}
   \pas_{t,k} \left(1-\mu - \frac{L_{\dot \lyap}}{v_{\min}} \pas_{t,k}
   \right) \PE\left[ \Delta_{t,k}^\star \vert \F_{t,0} \right] \\ & -
   \sum_{k=1}^{\kin{t}} \frac{1}{\pas_{t,k}} \left(1 - \frac{1}{4 \mu}
   - \frac{L_{\dot \lyap}}{v_{\min}} \pas_{t,k} \right) \PE\left[
     \mathcal{D}_{t,k} \vert \F_{t,0} \right] \\ & + \pas_{t,0} v_\max
   \kin{t} \| \mathcal{E}_t \|^2 \\
   & + v_\max \sum_{\ell=1}^{\kin{t}}
   (\kin{t}-\ell+1) \pas_{t,\ell} \, \mathcal{U}_{t, \ell}\\ & +
   \frac{2v_\max}{v_{\min}}  \kin{t} \sum_{k=1}^{\kin{t}-1} \pas^3_{t,k}
  A_{t,k+1} \PE\left[
     \Delta_{t,k}^\star \vert \F_{t,0} \right] \\ & + \frac{2
     v_\max}{v_{\min}}  \kin{t}  \sum_{k=1}^{\kin{t}-1}
   \pas_{t,k+1} A_{t,k+1}
   \PE\left[ \mathcal{D}_{t,k} \vert \F_{t,0} \right] \eqsp.
 \end{align*}
Observe that the coefficient in front of $\PE\left[ \Delta_{t,k}^\star
  \vert \F_{t,0} \right] $ is $ \pas_{t,k} (1- \mu -
\Lambda_{t,k+1})$; and the term in front of $\PE\left[
  \mathcal{D}_{t,k} \vert \F_{t,0} \right]$ is $ \pas_{t,k}^{-1}
(1-1/(4 \mu) - \Lambda_{t,k+1})$. By symmetry, we choose $\mu = 1/2$
so that $\mu = 1/(4 \mu)$. This yields
\begin{align*}
   & \PE\left[ \lyap(\hatS_{t, \kin{t}}) + g(\hatS_{t,\kin{t}}) \vert
    \F_{t,0} \right] \\ & \qquad \leq \PE\left[ \lyap(\hatS_{t,0}) +
    g(\hatS_{t,0}) \vert \F_{t,0} \right] \\ & \qquad - \sum_{k=1}^{\kin{t}}
  \pas_{t,k} \left( \frac{1}{2} - \Lambda_{t,k+1} \right) \PE\left[
    \Delta_{t,k}^\star \vert \F_{t,0} \right] \\ & \qquad -
  \sum_{k=1}^{\kin{t}} \frac{1}{\pas_{t,k}} \left( \frac{1}{2} -
  \Lambda_{t,k+1} \right) \PE\left[ \mathcal{D}_{t,k} \vert \F_{t,0}
    \right] \\ & \qquad + \pas_{t,0} v_\max \kin{t} \| \mathcal{E}_t \|^2
  \\ & \qquad + v_\max \sum_{\ell=1}^{\kin{t}} (\kin{t}-\ell+1) \pas_{t,\ell}
  \, \mathcal{U}_{t, \ell} \eqsp.
 \end{align*}
We now sum for $t=1$ to $t = \kout$ and compute the expectation. This
yields, by using that $\hatS_{t+1,0} = \hatS_{t,\kin{}}$,
\begin{align*}
   & \sum_{t=1}^{\kout} \sum_{k=1}^{\kin{t}} \pas_{t,k} \left(
  \frac{1}{2} - \Lambda_{t,k+1} \right) \PE\left[ \Delta_{t,k}^\star
    \right] \\ & + \sum_{t=1}^{\kout} \sum_{k=1}^{\kin{t}}
  \frac{1}{\pas_{t,k}} \left( \frac{1}{2} - \Lambda_{t,k+1} \right)
  \PE\left[ \mathcal{D}_{t,k} \right] \\ & \leq \qquad \PE\left[
    \lyap(\hatS_{t,0}) + g(\hatS_{t,0})\right] \\ & \qquad - \PE\left[
    \lyap(\hatS_{\kout, \kin{\kout}}) + g(\hatS_{\kout,\kin{\kout}})
    \right] \\ & \qquad + v_\max \sum_{t=1}^{\kout} \pas_{t,0} \kin{t}
  \PE\left[ \| \mathcal{E}_t \|^2 \right] \\ & \qquad + v_\max
  \sum_{t=1}^{\kout} \sum_{\ell=1}^{\kin{t}} \left(\kin{t} -
  \ell+1\right) \pas_{t,\ell} \, \mathcal{U}_{t, \ell} \eqsp.
 \end{align*}
 The proof is concluded upon noting that $\PE\left[
   \lyap(\hatS_{\kout, \kin{\kout}}) + g(\hatS_{\kout,\kin{\kout}})
   \right] \geq \min_\Sset (W+g)$.

 \subsection{Proof of \Cref{coro:complexity}}
 \label{sec:proof:complexity}
Since $\mathcal{U}_{t,k} = 0$, we have $C_b = C_{vb}=0$. In addition,
$\kin{t} = \kin{}$ for any $t$. Therefore, we can consider a constant
stepsize sequence $\pas_{t,k} = \pas_\star$ where $\pas_{\star}$ satisfies (see \eqref{eq:def:Lambda} and \eqref{eq:conditions:stepsize})
\[
 \pas_{\star} \frac{L_{\dot \lyap}}{v_{\min}} + \pas^2_{\star}
          \frac{2v_\max}{v_{\min}} L^2 \frac{\kin{}}{\lbatch}  \in \ooint{0,1/2} \eqsp.
\]
This condition is satisfied by choosing
\begin{align*}
\frac{\kin{}}{\lbatch} & \eqdef \frac{1}{v_{\min} v_\max}
\frac{L_{\dot \lyap}^2}{L^2} \eqsp, \\
 \pas_{\star} & \eqdef \frac{1}{4 v_\max} \frac{L_{\dot \lyap}}{L^2} \frac{\lbatch}{\kin{}} = \frac{v_{\min}}{4 L_{\dot \lyap}} \eqsp.
\end{align*}
Such a choice implies that $\inf_{t,k} (1/2 - \Lambda_{t,k}) = 1/2-3/8 = 1/8$. Since $\mathcal{E}_t = \mathcal{U}_{t,k} =0$, we obtain from \Cref{coro:maintheo} that
\begin{multline*}
\PE\left[
     \Delta_{\tau,K}^\star + \mathcal{D}_{\tau,K}^\star \right]
     \\ \leq \frac{32 \, L_{\dot \lyap}}{v_{\min}} \frac{ \left(  \PE\left[ \lyap(\hatS_{1,0}) +
       g(\hatS_{1,0})\right] - \min_\Sset \left(\lyap +g \right) \right)}{\kout \kin{}} \eqsp.
\end{multline*}
The approximate $\epsilon$-stationary condition is satisfied by choosing $\kout \kin{} = O(L_{\dot \lyap}/(v_{\min} \, \epsilon))$.
The number of calls to the proximal operator is $\kout \kin{}$ so that $\mathcal{K}_{\prox} = O(L_{\dot \lyap}/(v_{\min} \, \epsilon))$. Finally, we have $\lbatch_{t}' = n$ so that the number of calls to  one of the $\barprecond_i$'s is $\kout \, n + 2 \kout \kin{} \lbatch$. We can choose
$\lbatch = O\left(\sqrt{n} \sqrt{ v_{\min} v_{\max}} L/L_{\dot \lyap}\right)$. This yields $ \kout = O( L \sqrt{v_\max} / ( \sqrt{v_{\min}} \epsilon \sqrt{n}))$, and $\mathcal{K}_{\barprecond}= O(\sqrt{v_{\max}} L \sqrt{n}/(\epsilon \sqrt{v_{\min}}))$.

\subsection{Cost of the approximation on the $\barprecond_i$'s}
\label{sec:proof:MCcost}
Following the rates obtained in \Cref{coro:complexity}, let us set
$\kin{t} = O(\sqrt{n})$, $\lbatch = O(\sqrt{n})$ and $\kout =
O(1/(\sqrt{n} \epsilon))$ and let us show that we can define random
approximations $\delta_{t,0,i}$ and $\delta_{t,k+1,i}$ such that the
approximate $\epsilon$-stationarity condition is satisfied.

{\bf On the term $\PE\left[ \| \mathcal{E}_\tau \|^2 \right]$.}
We write $\PE\left[ \| \mathcal{E}_\tau \|^2 \right] = (\kout)^{-1} \,
\sum_{t=1}^{\kout} \PE\left[ \| \mathcal{E}_t \|^2 \right]$ and
           \[ 
\frac{1}{\kout} \sum_{t=1}^{\kout}  \frac{\epsilon^{1-\pa'}}{\sqrt{n}^{\pa'} t^{\pa'}} = \epsilon \, 
O(1) \eqsp.
\]

Let us compute the associated Monte Carlo complexity in the case $
\mathcal{E}_t = n^{-1} \sum_{i=1}^n \{ \delta_{t,0,i} -
\barprecond_i(\hatS_{t,0}) \}$ and $\delta_{t,0,i}$ is equal to a
Monte Carlo sum with $m_{t,0}$ i.i.d.  samples. Then $ \PE\left[ \|
  \mathcal{E}_t \|^2 \right] = n^{-1} m_{t,0}^{-1} O(1)$. It is equal
to $O(\epsilon^{1-\pa'}/ (\sqrt{n} t)^{\pa'})$ when $m_{t,0} = O(
n^{\pa'/2-1} t^{\pa'}/\epsilon^{1-\pa'})$. Therefore, the Monte Carlo
cost is
\[
\sum_{t=1}^{\kout} n m_{t,0} = O\left(  \frac{1}{\sqrt{n} \epsilon^2}\right) \eqsp.
\]

  {\bf On the term $\PE\left[ \left(
      \kin{} - K+1 \right) \mathcal{U}_{\tau, K} \right]$.}  This term is upper bounded by  $\kin{} \, \PE\left[  \mathcal{U}_{\tau, K} \right]$ and we write
  \[
\kin{} \, \PE\left[ \mathcal{U}_{\tau, K} \right] \leq
\frac{1}{\kout} \sum_{t=1}^{\kout} \sum_{k=1}^{\kin{}} O\left( \frac{1}{\lbatch \, \nbrmcv_{t,k+1}} \right) \eqsp.
\]
The RHS is $O(\epsilon)$. The associated Monte Carlo complexity is
\[
2 \lbatch \, \sum_{t=1}^{\kout}  \sum_{k =1}^{\kin{}}  \nbrmcv_{t,k}= O\left(\frac{\sqrt{n}}{\epsilon^2} \right) \eqsp,
\]
whatever $\pa, \bar \pa  \in \coint{0,1}$.

\bmhead{Acknowledgments} This work was partly supported by the Fondation Simone
  et Cino del Duca, under the program OpSiMorE; and by the french
  Agence Nationale de la Recherche (ANR) under the program ANR-19-CE23
  MASDOL.

\begin{appendices}
  \section{The condition \Cref{hyp:error} in the Monte Carlo case}
  \label{sec:hyp:montecarlo}
Following the framework detailed in \Cref{sec:montecarlo:casehi}, let
us assume that \textit{(i)} the intractable quantities $\precond_i(\hatS_{t,k}, B_{t,k+1})$ and 
$\precond_i(\hatS_{t,k-1},B_{t,k})$ are of the form
\begin{equation} \label{eq:integralforwardoperator}
\precond_i(s,B) = \int_\Zset H_{\vartheta}(z) \pi_{\vartheta}(\rmd z) \eqsp, 
\end{equation}
where $\vartheta \eqdef (s,i,B)$; and \textit{(ii)} these integrals
are approximated by a Monte Carlo sum:  set $\vartheta_{t,\ell+1,i}\eqdef(\hatS_{t,\ell}, i, B_{t,\ell+1})$ and 
\begin{multline}
\delta_{t,k+1,i} \eqdef \frac{1}{\nbrmc_{t,k+1}}
\sum_{r=1}^{\nbrmc_{t,k+1}} \left\{ H_{\vartheta_{t,k+1,i}}(Z_r^{\vartheta_{t,k+1,i}})  \right.   \\
\left. -
H_{\vartheta_{t,k,i}}(Z_r^{\vartheta_{t,k,i}}) \right\} \eqsp,  \label{eq:delta:MCMC}
\end{multline}
where, conditionally to $\hatS_{t,k-1}$, $\hatS_{t,k}$, $B_{t,k}$ and
$B_{t,k+1}$, the samples $\{Z_r^{\vartheta_{t,\ell,i}}, r \geq 1 \}$ are
a Markov chain with unique stationary distribution
$\pi_{\vartheta_{t,\ell,i}}( \rmd z)$; $\ell \in \{k, k+1\}$. Below, we
show that \Cref{hyp:error} is verified when the Markov chain is
ergodic enough. Let us start with introducing few notations from the
Markov chain theory (see e.g. \cite{meyn:tweedie}).

Let $\kernel$ be a transition kernel onto the measurable set
$(\Zset,\Zsigma)$ and $\linit, \pi$ be probability measures on
$(\Zset, \Zsigma)$.  For a measurable function $\xi: \Zset \to
\coint{0,+\infty}$, define
\[
\pi(\xi) \eqdef \int_\Zset \xi(z) \, \pi(\rmd z) \eqsp.
\]
For any $r \in \nset$, the $r$-iterated transition kernel $\kernel^r$
is defined by induction:
\begin{align*}
  \kernel^{r+1}(z, A) & \eqdef \int_{\Zset} \kernel^r(z,\rmd y) \, \kernel(y,A) \\
  & = \int_{\Zset} \kernel(z,\rmd y) \, \kernel^r(y,A) \eqsp,
\end{align*}
for all $z \in \Zset, A \in \Zsigma$; by convention, $\kernel^0(z,A)
\eqdef \indic_A(z)$ the $\{0,1\}$-valued indicator function and
$\kernel^0(z,A) = \delta_z(A)$, the Dirac mass at zero.  Given a
probability measure $\linit$ on $(\Zset, \Zsigma)$, $\linit \kernel$ stands for the
probability measure on $(\Zset, \Zsigma)$ given by
\[
\linit \kernel(A) \eqdef \int_\Zset \linit(\rmd y) \kernel(y,A) \eqsp,
\qquad \forall A \in \Zsigma \eqsp. 
\]
For a function $U: \Zset \to \coint{1,+\infty}$ such that
$\linit \kernel^r(U) + \pi(U) < +\infty$, define the $U$-norm of a
measurable function $\xi:\Zset \to \rset^q$
\[
\|\xi\|_U \eqdef \sup_\Zset \frac{\|\xi\|}{U} \eqsp;
\]
and  the $U$-norm of the signed measure $\linit \kernel^r - \pi$ by
\[
\| \linit \kernel^r - \pi \|_U \eqdef \sup_{\xi: \|\xi \|_U \leq
  1} \left\| \linit \kernel^r(\xi) - \pi(\xi) \right\| \eqsp.
\]
Let us go back to sufficient conditions for verifying
\Cref{hyp:error}. Denote by $\kernel_{\vartheta}$ a Markov transition
kernel with invariant distribution $\pi_{\vartheta}(\rmd z)$: at
iteration $(t,k+1)$, conditionally to ($\hatS_{t,k-1}, \hatS_{t,k}$,
$B_{t,k}$, $B_{t,k+1}$), the chains $\{Z_r^{\vartheta_{t,k}}, r \geq
0\}$ and $\{Z_r^{\vartheta_{t,k+1}}, r \geq 0\}$ are Markov chains
with transition kernels $\kernel_{\vartheta_{t,k}}$ and
$\kernel_{\vartheta_{t,k+1}}$ respectively. They have the same initial
value $\lambda$. Assume
\begin{assumption} \label{hyp:MCMC}
\begin{enumerate}[a)]
\item \label{hyp:MCMC:Unorm} There exists a measurable function
  $U: \Zset \to \coint{1,+\infty}$ such that 
\[
H_\star \eqdef \sup_{(s,i,B) \in \Sset \times [n]^\star \times \Pmatrix_+^q} \|H_{\vartheta} \|_U < + \infty \eqsp,
\]
where $H_\vartheta$ is defined by \eqref{eq:integralforwardoperator}.
\item \label{hyp:MCMC:ergo} There exist a function
  $\rho: \nset \to \ccint{0,1}$ and a positive constant $C_\MC$ such
  that for any $r \in \nset$,
\[
\sup_{\vartheta \in \Sset \times [n]^\star \times \Pmatrix_+^q}   \| \linit \kernel_{\vartheta}^r -
\pi_{\vartheta} \|_U \leq C_\MC \, \rho(r) \eqsp.
\]
In addition, $\sum_{r \geq 1} \rho(r) <  +\infty$.
\item \label{hyp:MCMC:Burk} Let $\vartheta \in \Sset \times [n]^\star
  \times \Pmatrix_+^q$. Let $\{Z_r^{\vartheta}, r \geq 1 \}$ be a
  Markov chain with transition kernel $\kernel_{\vartheta}$ and
  initial distribution $\lambda$. There exists a positive constant
  $C_\MC'$ such that for any $\vartheta \in \Sset \times [n]^\star
  \times \Pmatrix_+^q$ and $m' \in \nset_\star$,
\begin{multline*}
 \PE\left[ \| \sum_{r=1}^{m'}
  \{ H_{\vartheta}(Z_r^{\vartheta}) - \pi_{\vartheta}(H_{\vartheta}) \} \|^2 
  \right] \\
 \leq H_\star^2 \, C_\MC' \, m'\eqsp.
\end{multline*}
\end{enumerate}
\end{assumption}
A\autoref{hyp:MCMC}-\autoref{hyp:MCMC:ergo} is a uniform-in-$s$
ergodicity condition. Sufficient conditions for it are provided in
\citep[Lemma 2.3.]{fort:moulines:priouret:2012} in the case of a
geometric rate $\rho(r)= \kappa^r$ for some $\kappa \in
\ooint{0,1}$. By adapting \citep[Theorem 1]{andrieu:fort:vihola:2015},
similar conditions can be obtained in the case of a subgeometric rate
$\rho(r)$.  Sufficient conditions for
A\autoref{hyp:MCMC}-\autoref{hyp:MCMC:Burk} can be obtained from a
trivial adaptation of \citep[Proposition 12]{fort:moulines:2003}.

We prove the following result.
\begin{proposition}\label{prop:check:hyp:error}
Assume \Cref{hyp:MCMC}. Let $\delta_{t,k+1,i}$ be given by
\eqref{eq:delta:MCMC}, where conditionally to $(\hatS_{t,k-1},
\hatS_{t,k}, B_{t,k}, B_{t,k+1})$, $\{Z_r^{\vartheta_{t,\ell,i}}, r \geq
0\}$ is a Markov chain with transition kernel
$\kernel_{\vartheta_{t,\ell,i}}$ and initial distribution $\lambda$, for
$\ell \in \{k, k+1\}$. Then \Cref{hyp:error} is verified with
$\nbrmc_{t,k+1} = \nbrmcv_{t,k+1} = \bar \nbrmcv_{t,k+1} \leftarrow
\nbrmc_{t,k+1}$, $C_b = C_{vb} \eqdef 2 \, H_\star \, C_\MC \, \sum_{r
  \geq 1} \rho(r)$ and $C_v \eqdef 2 H_\star^2 \, C_\MC'$.
  \end{proposition}
\begin{proof}
  We will use the notations
  \[
  \kernel_{\ell,i} \eqdef
  \kernel_{\vartheta_{t,\ell+1,i}}, \ \   \pi_{\ell,i} \eqdef
  \pi_{\vartheta_{t,\ell+1,i}}, \ \   H_{\ell,i} \eqdef H_{\vartheta_{t,\ell+1,i}} \eqsp.
  \]
  
  {\bf $\bullet$ Expression of $\mu_{t,k+1}$.} We have
  \begin{multline*}
    \mu_{t,k+1,i} \eqdef \frac{1}{\nbrmc_{t,k+1}}
    \sum_{r=1}^{\nbrmc_{t,k+1}} \left( \lambda \kernel_{k,i}^r H_{k,i}
    - \pi_{k,i}(H_{k,i})\right) \\ - \frac{1}{\nbrmc_{t,k+1}}
    \sum_{r=1}^{\nbrmc_{t,k+1}} \left( \lambda \kernel_{k-1,i}^r
    H_{k-1,i} - \pi_{k-1,i}(H_{k-1,i}) \right) \eqsp.
  \end{multline*}

  {\bf $\bullet$  The condition \Cref{hyp:error}-\Cref{hyp:error:bias}.}
  By \Cref{hyp:MCMC} and since $\hatS_{\bullet} \in \Sset$, we write
  \[
\sup_{k,i} \, \| \lambda \kernel_{k,i}^r H_{k,i}
    - \pi_{k,i}(H_{k,i}) \| \leq  H_\star \, C_\MC \, \rho(r) \eqsp.
    \]
    This implies that
    \[
\| \mu_{t,k+1,i} \| \leq 2 \, H_\star \, C_\MC \, \frac{1}{\nbrmc_{t,k+1}}
    \sum_{r=1}^{\nbrmc_{t,k+1}} \rho(r) \eqsp.
    \]
   Since $\sum_r \rho(r) < \infty$, the RHS is of the form $C_b /
   \nbrmc_{t,k+1}$ with $C_b \eqdef 2 \, H_\star \, C_\MC \, \sum_r \rho(r)$.

   {\bf $\bullet$  The condition \Cref{hyp:error}-\Cref{hyp:error:var}.} We write
   \[
\sigma^2_{t,k+1,i} \leq \PE\left[ \| \xi_{t,k+1,i}\|^2 \vert \mathcal{P}_{t,k+1/2} \right] \eqsp.
\]
Then, we have
\begin{multline*}
\PE\left[ \| \xi_{t,k+1,i}\|^2 \vert \mathcal{P}_{t,k+1/2} \right]
\\
\leq 2 \sup_{s,i} \PE\left[ \| \frac{1}{\nbrmc_{t,k+1}}
  \sum_{r=1}^{\nbrmc_{t,k+1}} H_{s,i}(Z_r^{s,i}) -  \pi_{s,i}(H_{s,i})\|^2\right]
\end{multline*}
and the RHS is upper bounded by $2 H_\star^2 \, C_\MC' /
\nbrmc_{t,k+1}$ by \Cref{hyp:MCMC}-\Cref{hyp:MCMC:Burk}.

We also have
\[
\frac{1}{n} \sum_{i=1}^n  \| \mu_{t,k+1,i}-  \frac{1}{n} \sum_{j=1}^n \mu_{t,k+1,j}\|^2 \leq \frac{1}{n} \sum_{i=1}^n  \| \mu_{t,k+1,i} \|^2.
\]
From the upper bound on $\| \mu_{t,k+1,i}\|$ above, we have
\[
\| \mu_{t,k+1,i}\|^2 \leq \frac{C_b^2}{\nbrmc_{t,k+1}^2} \eqsp.
\]
This concludes the proof.
  \end{proof}

\section{Supplementary materials for \Cref{sec:application}}
\label{sec:proof:appli}
\subsection{The penalized log-likelihood criterion}
The observations are assumed independent, so the log-likelihood is given by
\begin{multline*}
\param \mapsto \sum_{i=1}^n \log \int_{\rset^d} (1+\exp(-y_i
\pscal{X_i}{z_i}))^{-1} \\ \times \frac{1}{\sqrt{2\pi}^d \sigma^d} \exp\left(-(2 \sigma^2)^{-1} \| z_i
- \theta \|^2 \right) \rmd z_i \eqsp.
  \end{multline*}
The penalty term is $- n \tau \| \param \|^2$.

\begin{lemma}
  \label{lem:changevar:penll}
  The sum of the log-likelihood and the penalty term is equal to
  \begin{multline*}
 - \frac{n}{2} \log( 2 \pi \sigma^2) - \frac{1}{2 \sigma^2} \param^\top \, \sum_{i=1}^n \frac{ X_i X_i^\top}{\|X_i\|^2} \param  - n \tau \|\param\|^2 \\
    + \sum_{i=1}^n \log \int_{\rset}   \frac{\exp\left( x \pscal{X_i}{\param}/(\sigma^2 \|X_i\|)  \right)}{1+\exp(-y_i \|X_i\| x
    )}  \exp(-\frac{x^2}{2\sigma^2})\rmd x \eqsp.
  \end{multline*}
\end{lemma}
\begin{proof}
  Let $i \in [n]^\star$.  Define an orthogonal $d \times d$ matrix $Q$
  with columns denoted by $(Q_1, \cdots, Q_d)$, such that $Q_1 \eqdef
  X_i/ \|X_i\|$. We have $\pscal{z_i}{X_i} = \|X_i\|
  \pscal{Q_1}{z_i}$, $\pscal{z_i}{\param} = \pscal{Q^\top z_i}{Q^\top
    \param}$ and $\|z_i\|^2 = \|Q^\top z_i \|^2$. This implies that
  \begin{align*}
& -y_i \pscal{z_i}{X_i} = -y_i \|X_i\| \pscal{Q_1}{z_i} \\
 &  \|z_i - \param\|^2 = \| \param\|^2 + \|Q^T z_i\|^2 - 2
  \pscal{Q^\top z_i}{Q^\top \param} 
\end{align*}
  so that the log-likelihood of the observation $Y_i$ is (up to the additive constant $C_1 \eqdef - d \ln \sigma -(d/2) \ln(2\pi)$)
 \begin{multline*}
   y_i \mapsto - \frac{\|\param\|^2}{2 \sigma^2} + \log  \int_{\rset^d} (1+\exp(-y_i \|X_i\| \pscal{Q_1}{z_i}
   ))^{-1} \\ \times\exp\left(-(2 \sigma^2)^{-1}   \left( \|Q^T z_i\|^2 - 2
  \pscal{Q^\top z_i}{Q^\top \param}  \right)\right) \rmd z_i \eqsp.
 \end{multline*}
 By a change of
  variable $v=(v_1, \cdots, v_q) \leftarrow Q^\top z_i$, the logarithm of the integral is equal to 
  \begin{align*}
&\log \int_{\rset^d}     \frac{\exp\left(-(2 \sigma^2)^{-1}   \left( \|v\|^2 - 2
    \pscal{v}{Q^\top \param}  \right)\right)}{1+\exp(-y_i \|X_i\| v_1
    )}  \rmd v  \\
& = \log \int_{\rset}   \frac{\exp\left(-(2 \sigma^2)^{-1}   \left(v_1^2 - 2
    v_1 \pscal{Q_1}{\param}  \right)\right)}{1+\exp(-y_i \|X_i\| v_1
    )} \rmd v_1 \\
 & \quad  +   \sum_{u=2}^d \log \int_{\rset} \exp\left(
  -\frac{1}{2\sigma^2} \{v_u^2 - 2 v_u \pscal{Q_u}{\param} \} \right)
  \, \rmd v_u \eqsp.
  \end{align*}
  The last $(d-1)$  integrals have a closed form. Observe indeed that $v_u^2 - 2 v_u \pscal{Q_u}{\param} = (v_u - \pscal{Q_u}{\param})^2 -(\pscal{Q_u}{\param})^2$ so that up to the additive constant $C_2 \eqdef  (d-1) \{ \log(2 \pi)/2 + \log \sigma \}$
  \begin{align*}
& \sum_{u=2}^d \log \int_{\rset} \exp\left( -\frac{1}{2\sigma^2} \{v_u^2 - 2 v_u
\pscal{Q_u}{\param} \} \right) \, \rmd v_u \\ & = 
\sum_{u=2}^d \frac{(\pscal{Q_u}{\param})^2}{2\sigma^2} = \frac{\|\param\|^2 - (\pscal{Q_1}{\param})^2}{2\sigma^2}  \\
& =  \frac{\|\param\|^2 - (\pscal{X_i}{\param})^2/ \|X_i\|^2}{2\sigma^2}  \eqsp.
    \end{align*}
This concludes the proof; the constant (w.r.t. to $\param$) is equal to $C_1 + C_2$.
\end{proof}

\subsection{Proof of \Cref{lem:compactset}}
\label{sec:proof:lem:compactset}
The criterion $F$ is equal to $-\mathcal{L}(\param) - \log(2 \pi \sigma^2)/2$, where $\mathcal{L}(\param)$ is the normalized penalized log-likelihood.

The likelihood is the product of probabilities, taking values in
$\ooint{0,1}$; therefore, its logarithm is negative. The penalized
log-likelihood is upper bounded $-\mathrm{pen}(\param) = - n \tau
\|\param\|^2$. The normalized penalized
log-likelihood is upper bounded $- \tau
\|\param\|^2$.  Therefore the criterion is lower bounded by $\tau \| \param \|^2 - \ln(2 \pi \sigma^2)/2$.

On the other hand, the minimum of the criterion is smaller than the
value of the criterion at $\param =0$. Let us show that this value is
$(\ln 4) /n - \ln(2 \pi \sigma^2)/2$. This will imply that the
minimizers of the criterion are in the set $\{\param \in \rset^d: \tau
\| \param \|^2 \leq \ln 4\}$ and conclude the proof.

We have $\mathrm{pen}(0) = 0$. Let us lower bound the likelihood of an observation $Y_i=+1$ at $\param = 0$. The likelihood is equal to
\[
\frac{1}{\sqrt{2\pi}^d \sigma^d} \int_{\rset^d}  \frac{  \exp\left(-(2 \sigma^2)^{-1} \| z_i \|^2 \right)}{1+\exp(-
  \pscal{X_i}{z_i})} \rmd z_i \eqsp.
\]
By using the same change of variable than in the proof of \Cref{lem:changevar:penll}, it is equal to 
\begin{multline*}
\frac{1}{\sqrt{2\pi} \sigma} \int_{\rset}  \frac{  \exp\left(-x^2/(2 \sigma^2) \right)}{1+\exp(-
  \|X_i\| x)} \rmd x   \\
\left( \frac{1}{\sqrt{2\pi} \sigma} \int_{\rset}    \exp\left(- x^2/(2\sigma^2) \right) \rmd x \right)^{d-1}\eqsp,
\end{multline*}
and is lower bounded by (note that the $(d-1)$  identical integrals are equal to one)
\[
\frac{1}{\sqrt{2\pi} \sigma} \int_{\rset^+}  \frac{  \exp\left(-x^2/(2 \sigma^2) \right)}{1+\exp(-
  \|X_i\| x)} \rmd x  \eqsp,
\]
which is in turn lower bounded by $1/4$ since $1+\exp(-
\|X_i\| x) \leq 2$ for all $x \geq 0$.

The proof for the case $Y_i = -1$ is on the same lines and is omitted.

This implies that the likelihood of the $n$ variables is lower bounded
by $1/4^n$; the normalized log-likelihood is lower bounded by $-\ln
4$; the criterion is upper bounded by $\ln 4 - \ln(2 \pi \sigma^2)/2$.

\subsection{The optimization problem seen as an EM}
\label{sec:proof:appli:hinEM}
We established that for any $\param \in \rset^d$, $\nabla F(\param) = n^{-1} \sum_{i=1}^n G_i(\param)$ where
\[
G_i(\param) \eqdef 2 U \param - \frac{X_i}{\sigma^2 \, \|X_i\|} \int_{\rset}  z \ \pi_{\param,i}(z) \rmd z \eqsp,
\]
and $\pi_{\param,i}(z)$ is the probability density proportional to \eqref{eq:EM:pi}.

From the expressions of $\phi, \psi$ and $S(Y_i,z)$, we obtain that $\map$, defined by  \Cref{prop:EM:statisticspace}, is given by $\map(s) \eqdef U^{-1} s/2$ for any $s \in \rset^d$. This implies that for any $s \in \rset^d$,
\begin{align*}
  \precond_i(s,B) & \eqdef \int_{\rset} S(Y_i,z) \, \pi_{\map(s),i}(z) \, \rmd z - s  \\
  & = \frac{X_i}{\sigma^2 \, \| X_i \| } \int_{\rset} z \, \pi_{Bs,i}(z) \, \rmd z - s \eqsp.
\end{align*}

For any $s \in \rset^d$, let us find the matrix $\B(s)$ satisfying
$\nabla (F \circ \map)(s) = - n^{-1} \sum_{i=1}^n \B(s) \precond_i(s)$ (see \eqref{eq:EM:preconfGdt}). We have $\nabla (F \circ \map) = \nabla F(B\cdot) = B^\top \, (\nabla F)(B \cdot) = B \, (\nabla F)(B\cdot)$. This yields
\begin{align*}
 &  \nabla (F \circ \map)(s)  = B \frac{1}{n} \sum_{i=1}^n  G_i(B s) \\
  & = B \frac{1}{n} \sum_{i=1}^n \left( 2 U \param - \frac{X_i}{\sigma^2 \, \|X_i\|} \int_{\rset}  z \ \pi_{Bs,i}(z) \rmd z \right) \\
  & = B s - B  \frac{1}{n} \sum_{i=1}^n \frac{X_i}{\sigma^2 \, \|X_i\|}  \int_{\rset}  z \ \pi_{B s,i}(z) \rmd z  \\
  & = - B \frac{1}{n} \sum_{i=1}^n \precond_i(s) \eqsp.
\end{align*}
This yields $\B(s) \eqdef B$ for any $s \in \rset^d$.

\subsection{Proof of \Cref{lem:appli:MCMC:IPPandPolyagamma}}
\label{sec:proof:lem:MCMC:IPP}
Let $i \in [n]^\star$ and $\param \in \rset^d$.
{\bf Step 1.} By using
\begin{multline*}
  -z^2/(2 \sigma^2) + z \pscal{X_i}{\param}/(\sigma^2 \|X_i\|)
  \\ = - (z - \pscal{X_i}{\param}/\|X_i\|)^2/(2 \sigma^2) \\
  +
(\pscal{X_i}{\param})^2/(2 \sigma^2 \|X_i\|^2) \eqsp,
\end{multline*}
we write
\[
\pi_{\param,i}(z) =  \frac{\exp\left( - \left( z - \pscal{X_i}{\param}/\|X_i\|\right)^2 / (2\sigma^2)  \right)}{Z_{\param,i} \ \ 1+\exp(-y_i \|X_i\| z 
    )} 
\]
where $Z_{\param,i}$ is the normalizing constant.  Second, we use $z
\, \pi_{\param,i}(z) = (z - a_i) \, \pi_{\param,i}(z) + a_i \,
\pi_{\param,i}(z)$ with $a_i \leftarrow \pscal{X_i}{\param}/\|X_i\|$ and
since $\int_\rset \pi_{\param,i}(z) \rmd z = 1$, we obtain
\[
\mathcal{I}_{i}(\param) = \frac{\pscal{X_i}{\param}}{\|X_i\|} +\int_\rset \left(z - \frac{\pscal{X_i}{\param}}{\|X_i\|} \right)  \, \pi_{\param,i}(z) \, \rmd z \eqsp.
\]
Finally, the integral in the RHS being of the form
\[
\sigma^{2} \int_\rset  \frac{f'(z)}{Z_{\param,i} \ \ 1+\exp(-y_i \|X_i\| z)}  \rmd z
\]
with
\[
f(z) \eqdef - \exp\left( - \left( z - \pscal{X_i}{\param}/\|X_i\|\right)^2 / (2\sigma^2)  \right) \eqsp,
\]
we use an integration by parts. Upon noting that the derivative of $z \mapsto 1/(1+\exp(-y_i \|X_i\| z))$ is
\[
y_i \|X_i\| \frac{\exp(-y_i \|X_i\| z)}{\left( 1+\exp(-y_i \|X_i\| z)\right)^2} \eqsp,
\]
we write
\begin{multline*}
  \int_\rset  \frac{f'(z)}{ 1+\exp(-y_i \|X_i\| z)}  \rmd z \\
  =  - y_i \|X_i\|  \int f(z) \frac{\exp(-y_i \|X_i\| z)}{\left( 1+\exp(-y_i \|X_i\| z)\right)^2} \rmd z \eqsp.
\end{multline*}
Therefore, the conclusion of this first step is 
\begin{multline*}
 \mathcal{I}_i(\param)  = \pscal{\frac{X_i}{\|X_i\|}}{\param}  \\
  + y_i \|X_i\| \sigma^2 \int_\rset \frac{\pi_{\param,i}(z)}{1+\exp\left( y_i \|X_i\| z \right)}  \, \rmd z  \eqsp.
\end{multline*}

{\bf Step 2.} This step is classical in the MCMC
literature (see e.g. \cite{choi:hobert:2013} and references therein). We prove  that for any $z \in \rset$,
\[
\pi_{\param,i}(z) = \int_0^{+\infty} \bar \pi_{\param,i}(z,\omega) \rmd \omega \eqsp.
\]
By \cite[Theorem 1]{polson:scott:windle:2013}, it holds
\begin{multline*}
\frac{1}{1+\exp\left( - y_i \|X_i\| z \right)} = \frac{1}{2} \exp\left( y_i \|X_i\| z/2 \right)  \\
\times \int_0^{+\infty} \exp(-\omega \|X_i\|^2 z^2/2) \mathsf{p}(\omega;1) \rmd \omega \eqsp,
\end{multline*}
where $\mathsf{p}(\omega;b) \rmd \omega$ is a Polya-Gamma distribution  with parameter $b$.
This implies that $\pi_{\param,i}(z)$ is equal to 
\begin{multline*} 
\exp\left( \frac{y_i \|X_i\| z}{2} -\frac{(z - \pscal{X_i}{\param}/\|X_i\|)^2}{2\sigma^2}\right) \\
\times \frac{1}{2 Z_{\param,i}}  \int_0^{+\infty} \exp(-\omega \|X_i\|^2 z^2/2) \mathsf{p}(\omega;1) \rmd \omega \eqsp. 
  \end{multline*}
This concludes the proof.

\subsection{The assumption \Cref{hyp:error} is verified.}
\label{sec:appli:uniformergoPolyaGamma}
Define the Markov kernel with density
\[
\kernel_{s,i}(z; z') \eqdef \left(\int_0^\infty \pi_2(\omega \vert
z;i) \, \pi_1(z' \vert \omega; s,i) \ \rmd \omega \right)
\]
w.r.t. the Lebesgue measure on $\rset$; here, $\pi_1(z' \vert \omega;
s,i)$ is the density of a Gaussian distribution with expectation
$\mathsf{m}_{s,i}(\omega)$ and variance $v_{i}(\omega)$ given by
\begin{align*}
  v_{i}(\omega) & \eqdef \frac{\sigma^2}{1+\omega \sigma^2 \|X_i\|^2} \eqsp, \\
\mathsf{m}_{s,i}(\omega)  & \eqdef v_{i}(\omega)  \, \left( \frac{1}{\sigma^2}\pscal{\frac{X_i}{\|X_i\|}}{B s} +  \frac{1}{2} y_i
  \|X_i\|\right) \eqsp;
\end{align*}
and $\pi_2(\omega \vert z; i)$ is a Polya-Gamma distribution with
parameter $(1, \|X_i\| z)$. The Gibbs kernel described by
\Cref{lem:appli:MCMC:IPPandPolyagamma} and targeting the density
distribution $\bar \pi_{Bs,i}(z,\omega) \rmd z \rmd \omega$, produces
a Markov chain $\{(Z_r^{s,i}, \Omega_r^{s,i}), r \geq 0\}$ such that
the marginal $\{Z_r^{s,i}, r \geq 0\}$ is a Markov chain with
transition kernel $\kernel_{s,i}(z; z') \, \rmd z'$.
We apply the results of \cite[Proposition 3.1]{choi:hobert:2013} with
\begin{table}[ht]
  \begin{centering}
  \begin{tabular}{|l|l|}
      \hline
      $y \in \{0,1 \}$ & $y_i \in \{-1,1\}$ \\
      $n$ & $1$ \\
      $\Omega(\omega)$ & $\omega$ \\
      $X$  & $\|X_i\|$ \\
      $B$ & $\sigma^2$ \\
      $b$ & $\pscal{X_i}{Bs}/\|X_i\|$ \\
      \hline
  \end{tabular}
  \caption{[left] The notations of \cite{choi:hobert:2013}. [right] the  notations in this paper.}
  \end{centering}
\end{table}

This yields
\begin{align} \label{eq:choihobert}
  \int_A & \kernel_{s,i}(z; z') \rmd z'  \nonumber \\
  & \geq  \varepsilon \ \int_A  \exp(-0.5 (x- \mathsf{m}_\star)^2/ v_\star^2) \, \rmd x
\end{align}
where
\[
\varepsilon \eqdef \inf_{s \in \Sset, i \in [n]^\star} \frac{ \exp\left( -\frac{1}{4} - \frac{\{\mathsf{m}_{s,i}(1/2)\}^2 \, \sigma^2 \|X_i\|^2}{4 \, v_i(1/2)}\right)}{2 \sqrt{1+ \sigma^2 \|X_i\|^2/2}} 
\]
and $(\mathsf{m}_\star, v_\star)$ satisfy for any $x \in \rset$,
\begin{multline*}
\inf_{s \in \Sset, i \in [n]^\star} \exp(-0.5 \, (x-
\mathsf{m}_{s,i}(1/2))^2 / v_i(1/2) )\\ \geq \exp(-0.5 (x-
\mathsf{m}_\star)^2/ v_\star^2) \eqsp.
\end{multline*}
\begin{lemma}
Since $\Sset$ is bounded, then $\varepsilon >0$ and $\mathsf{m}_\star,
v_\star$ exist in $\rset \times \ooint{0,+\infty}$.
  \end{lemma}
The minorization condition \eqref{eq:choihobert} implies that the
kernel $\kernel_{s,i}(z, z') \rmd z'$ is uniformly ergodic, uniformly
in $s,i$ and $z$. By \cite[Theorem 16.0.2.]{meyn:tweedie} and \cite[Proposition 1]{fort:moulines:2003}, \Cref{hyp:MCMC}-\Cref{hyp:MCMC:ergo} and \Cref{hyp:MCMC}-\Cref{hyp:MCMC:Burk} are satisfied.

\section{Detailed proofs}\label{sec:appendix}
\subsection{Proof of \eqref{eq:stepsizeinit}}
\label{app:sec:stepsizeinit}
Let $t \in [\kout]^\star$.  The sequence given by $\pas_{t,k+1} \eqdef
\prod_{j=0}^k \left(1 + \frac{2 \, C_b}{\nbrmc_{t,j+1}} \right)^{-1}
\pas_{t,0}$ for any $k \geq 0$, satisfies
\[
\pas_{t,k+1} \left(1 + \frac{2 C_b}{\nbrmc_{t,k+1}} \right) \leq
\pas_{t,k} \eqsp.
\]
A sufficient condition for the property $\Lambda_{t,k+1} \in
\ooint{0,1/2}$ to hold is $\pa \pas_{t,0}^2 + \bar \pa \pas_{t,0} - 1/2 < 0$
where 
\begin{align*}
 & \bar \pa \eqdef \frac{ L_{\dot
      \lyap}}{v_{\min}} \eqsp, \\
&  \pa \eqdef L^2 \frac{2v_\max \kin{t} }{v_{\min} \lbatch}
  \left(1 + \frac{2 \, C_{vb}}{ \sqrt{\lbatch} \, \bar
    \nbrmcv_{t,k+1}} \right)  \eqsp.
\end{align*}
 The function $x \mapsto \pa x^2 + \bar \pa x - 1/2$ possesses two roots:
 one is positive and one is negative. The positive one is given by
 $(-\bar \pa + \sqrt{\bar \pa^2 + 2 \pa})/(2\pa)$; it is equal to
 \eqref{eq:stepsizeinit}.

  \subsection{Proof of \eqref{eq:var:init:2} and \eqref{eq:var:init}}
  \label{sec:app:Epsilon}
  We write $\Smem_{t,0} - \precond(\hatS_{t,0}) = U + V$ where
\begin{align*}
  U & \eqdef \frac{1}{\lbatch_t'} \sum_{i \in \batch_{t,0}} \left(
  \delta_{t,0,i} - \PE\left[ \delta_{t,0,i} \vert \hatS_{t,0},
    \batch_{t,0} \right] \right) \eqsp, \\ V & \eqdef
  \frac{1}{\lbatch_t'} \sum_{i \in \batch_{t,0}} \PE\left[
    \delta_{t,0,i} \vert \hatS_{t,0}, \batch_{t,0} \right] -
  \precond(\hatS_{t,0}) \eqsp.
\end{align*}
We have $\PE\left[ \|U+V \|^2 \right] = \PE\left[ \| U \|^2 \right] +
\PE\left[ \| V \|^2 \right]$ by definition of the conditional
expectation.  Since $\delta_{t,0,i}$ is an unbiased random
approximation of $\precond_i(\hatS_{t,0})$, we have $ \PE\left[
  \delta_{t,0,i} \vert \hatS_{t,0}, \batch_{t,0} \right]=
\precond_i(\hatS_{t,0})$.

In the case $\batch_{t,0} = \{1, \cdots, n\}$, then $V=0$. Therefore,
\begin{align*}
 & \PE\left[ \| \Smem_{t,0} - \precond(\hatS_{t,0}) \|^2 \right]  \\
  & =
   \frac{1}{n^2} \PE\left[ \|\sum_{i=1}^n \{ \delta_{t,0,i} -
     \PE\left[ \delta_{t,0,i} \vert \hatS_{t,0}\right] \} \|^2 \right] \\
  & = \frac{1}{n^2}\sum_{i=1}^n \sigma_{t,0,i}^2 \eqsp,
\end{align*}
where we used that the variables $\{\delta_{t,0,i}, i \in [n]^\star
\}$ are independent conditionally to $\hatS_{t,0}$, and with variance
$\sigma_{t,0,i}^2$.

In the case $\batch_{t,0}$ is a subset of $[n]^\star$ of cardinality
$\lbatch'_t$, then we write
\begin{align*}
 & \PE\left[ \| U \|^2 \right] = \frac{1}{(\lbatch_t')^2} \PE\left[
    \PE \left[ \|U\|^2 \vert \batch_{t,0}, \hatS_{t,0} \right]\right]
  \\ & = \frac{1}{\lbatch_t'} \PE\left[ \frac{1}{\lbatch_t'} \sum_{i
      \in \batch_{t,0} } \sigma_{t,0,i}^2\right]
\end{align*}
and we conclude by \Cref{lem:batch}.  Again from \Cref{lem:batch}, we
have
\[
\PE\left[ \|V \|^2 \right] \leq \frac{1}{\lbatch_t' \, n} \sum_{i=1}^n
\| \precond_i(\hatS_{t,0}) - \precond(\hatS_{t,0}) \|^2 \eqsp,
\]
with an equality when $\batch_{t,0}$ is sampled with replacement.
This concludes the proof.

\subsection{Proof of \Cref{lem:batch}}
\label{sec:proof:Bronde} 
Set, for ease of notations,
 \[
\batch \eqdef \batch_{t,k}, \qquad \precond_{\batch} \eqdef \frac{1}{\lbatch} \sum_{i \in \batch} \precond_i \eqsp.
  \]
\subsubsection{Case with replacement}
 We write
  $\batch = \{I_1, \cdots, I_{\lbatch} \}$ where the r.v. $I_i$'s are
 independent, and uniformly distributed on $[n]^\star$.

$\bullet$ Then
  \begin{align*}
\PE\left[ f_{\batch}\right] & =
\frac{1}{\lbatch}\sum_{\ell=1}^\lbatch \PE\left[ f_{I_\ell}
  \right] = \PE\left[ f_{I_1}\right] =  n^{-1} \sum_{i=1}^n f_i(u) \eqsp.
\end{align*}

$\bullet$ Set $\bar f \eqdef  n^{-1} \sum_{i=1}^n f_i$. We have, by using that the r.v. $\{I_1,\cdots, I_\lbatch \}$ are independent,
\begin{align*}
& \PE\left[ \|f_{\batch} -  \bar f \|^2 \right]  =
\PE\left[ \|\frac{1}{\lbatch} \sum_{\ell = 1}^{\lbatch}
  \left(f_{I_\ell} -   \bar f\right) \|^2 \right] \\
&  =
\frac{1}{\lbatch^2} \sum_{\ell = 1}^{\lbatch} \PE\left[
  \|f_{I_\ell} -  \bar f \|^2 \right] \\ & =
\frac{1}{\lbatch} \PE\left[ \|f_{I_1} -  \bar f \|^2 \right]
=  \frac{1}{\lbatch \, n} \sum_{i=1}^n  \|f_{i} - \bar f \|^2 \eqsp.
  \end{align*}

$\bullet$ Since the variance of the sum is the sum of the variance
for independent r.v.
\begin{align*}
& \PE\left[ \| \precond_{\batch}(u) - \precond_{\batch}(u') - \{
   \precond(u) -  \precond(u')\} \|^2 \right] \\
&   =
\frac{1}{\lbatch^2}\sum_{\ell=1}^\lbatch \PE\left[ \| \precond_{I_\ell}(u)
  - \precond_{I_\ell}(u') - \precond(u) +  \precond(u') \|^2 \right]
\eqsp.
\end{align*}
Then, since $I_\ell$ is uniformly distributed on $[n]^\star$,
\begin{align}
& \PE\left[ \| \precond_{I_\ell} (u) - \precond_{I_\ell} (u') -  \precond (u) +
   \precond(u') \|^2 \right] \nonumber \\
 & = \frac{1}{n} \sum_{i=1}^n
 \PE\left[ \| \precond_{i}(u) - \precond_{i}(u') \|^2 \right] - \| \precond (u)
 -  \precond(u') \|^2 \nonumber \\ & \qquad \leq \|u - u'\|^2
 \frac{1}{n} \sum_{i=1}^n L_i^2 - \|  \precond (u) -  \precond(u')
 \|^2  \eqsp. \label{eq:replacement}
\end{align}

\subsubsection{Case without replacement}
Set $\batch = \{I_1,\cdots, I_\lbatch\}$ and $\bar f \eqdef n^{-1} \sum_{i=1}^n f_i$.  $I_1$ is a uniform random
variable on $[n]^\star$ so that
$\PE\left[f_{I_1} \right] =  \bar f$.

$\bullet$ Conditionally
to $I_1$, $I_2$ is a uniform random variable on
$[n]^\star \setminus \{I_1 \}$. Therefore
\begin{align*}
\PE\left[f_{I_2} \right] & = \ \frac{1}{n-1} \left( \sum_{j=1}^n
f_{j} - \PE\left[f_{I_1} \right] \right) \\
& = \frac{n}{n-1}
\bar f - \frac{1}{n-1} \bar f =  \bar f \eqsp.
\end{align*}
By induction, for any $\ell \geq 2$,
\begin{align*}
&  \PE\left[f_{I_\ell} \right] \\
& = \ \frac{1}{n-\ell+1} \left(
 \sum_{j=1}^n f_j - \sum_{r=1}^{\ell-1} \PE\left[f_{I_r}
   \right] \right)  \\
& = \frac{n}{n-\ell+1} \bar f -
 \frac{\ell-1}{n-\ell+1} \bar f=  \bar f \eqsp.
\end{align*}
As a conclusion, $\lbatch^{-1} \, \sum_{\ell=1}^\lbatch
\PE\left[f_{I_\ell} \right] = \bar f$.

$\bullet$  Let $u,u' \in \Sset$; set $\phi({I_\ell}) \eqdef \precond_{I_\ell} (u) - \precond(u) - \precond_{I_\ell} (u') +  \precond(u')$.  Then
$\PE\left[\phi(I_\ell) \right] =0$. First, we prove by induction that
$\PE\left[ \| \phi(I_\ell) \|^2 \right] = \PE\left[ \| \phi(I_1) \|^2
  \right]$.  Upon noting that $I_1$ is a uniform random variable on
$[n]^\star$ and by using the induction assumption,
\begin{align*}
&  (n- \ell+1) \PE\left[ \| \phi(I_\ell) \|^2 \right] \\
&  =\left(
  \sum_{i=1}^n \|\phi(i) \|^2 - \PE\left[ \sum_{r=1}^{\ell-1} \|\phi(I_r)\|^2  \right] \right)  \\
& = n
  \PE\left[\|\phi(I_1)\|^2 \right] - (\ell-1)
  \PE\left[\|\phi(I_1)\|^2 \right],
\end{align*}
which concludes the induction. Second, let us prove that for any $\ell \geq 0$,
\begin{equation}\label{eq:replacement:subadditivity}
\PE\left[ \|\sum_{r=1}^{\ell+1} \phi(I_r) \|^2 \right] \leq (\ell+1)
\PE\left[ \|\phi(I_1) \|^2 \right] \eqsp.
\end{equation}
Since $\sum_{i=1}^n \phi(i) = n \PE\left[\phi(I_1) \right]=0$,
\begin{align*}
&  \PE\left[ \pscal{\sum_{p=1}^\ell \phi(I_p)}{\phi(I_{\ell+1})} \right] \\
 &= \PE\left[ \pscal{\sum_{p=1}^\ell \phi(I_p)}{\PE\left[
       \phi(I_{\ell+1}) \Big \vert I_1, \cdots, I_\ell \right]}
   \right]  \\
 &= \frac{1}{n-\ell} \PE\left[ \pscal{\sum_{p=1}^\ell
     \phi(I_p)}{\sum_{i=1}^n \phi(i) - \sum_{p=1}^\ell \phi(I_p)}
   \right] \\ &= - \frac{1}{n-\ell}\PE\left[\|\sum_{p=1}^\ell
   \phi(I_p)\|^2 \right] \eqsp,
\end{align*}
so that
\begin{align*}
 & \PE\left[ \|\sum_{p=1}^{\ell+1} \phi(I_p) \|^2 \right] \\
& = \left(1 -
  \frac{2}{n-\ell} \right) \PE\left[\|\sum_{p=1}^\ell \phi(I_p)\|^2
    \right] + \PE\left[ \| \phi(I_{\ell+1}) \|^2 \right]  \\
& \leq
  (\ell+1) \PE\left[ \|\phi(I_1) \|^2 \right] \eqsp.
\end{align*}
The proof of the first bound follows from
\eqref{eq:replacement:subadditivity} and \eqref{eq:replacement} since
here again, $I_1$ is uniformly distributed on $[n]^\star$.

$\bullet$  The proof
of the second bound is similar (change the definition of
$\phi(I_\ell)$); it is omitted.

\end{appendices}

\end{document}